\pdfoutput=1
\documentclass{article}
\usepackage{iclr2027_conference,times}

\usepackage{amsmath,amsfonts,bm}

\def\eqref#1{equation~\ref{#1}}

\def\1{\bm{1}}

\DeclareMathAlphabet{\mathsfit}{\encodingdefault}{\sfdefault}{m}{sl}
\SetMathAlphabet{\mathsfit}{bold}{\encodingdefault}{\sfdefault}{bx}{n}

\newcommand{\E}{\mathbb{E}}

\newcommand{\R}{\mathbb{R}}

\usepackage{hyperref}
\hypersetup{pdftitle={ShellElastic: Scalable Elastic Pretraining with Simultaneous Maximal Updates and Hyperparameter Transfer},
  pdfauthor={Can Jin, Yan Kang, Yifan Gong, Chongjian Ge, Mingze Xu, Yuanjun Xiong, Dimitris N. Metaxas, Hongwu Peng}}
\usepackage{url}
\usepackage{amsmath,amssymb,amsthm}
\usepackage{booktabs}
\usepackage{array}
\usepackage{multirow}
\usepackage{graphicx}
\usepackage{xcolor}
\usepackage{xspace}
\usepackage{tikz}
\usepackage{pifont}
\usepackage{enumitem}
\usepackage{etoc}
\usepackage{placeins}
\usetikzlibrary{arrows.meta,positioning,calc,shapes.misc}
\graphicspath{{paper_figures/}{figures/}{./}}

\newtheorem{theorem}{Theorem}
\newtheorem{lemma}{Lemma}
\newtheorem{proposition}{Proposition}
\newtheorem{corollary}{Corollary}
\newtheorem{definition}{Definition}
\newtheorem{assumption}{Assumption}

\definecolor{ourcolor}{HTML}{8B4513}

\newcommand{\mup}{$\mu$P\xspace}
\newcommand{\muT}{$\mu$-Transfer\xspace}
\newcommand{\muI}{$\mu$-Internal\xspace}
\newcommand{\muE}{\textcolor{ourcolor}{\textsc{$\mu$-Elastic}}\xspace} 
\newcommand{\SP}{the standard parameterization}
\newcommand{\Ord}[1]{\Theta\!\left(#1\right)}
\newcommand{\silu}{\operatorname{SiLU}}
\newcommand{\CE}{\operatorname{CE}}

\newcommand{\dbase}{g}

\newcommand{\Ours}{\textcolor{ourcolor}{\textsc{ShellElastic}}\xspace}

\definecolor{shella}{RGB}{42,120,214}
\definecolor{shellb}{RGB}{235,104,52}
\definecolor{shellc}{RGB}{27,175,122}
\definecolor{shelld}{RGB}{74,58,167}
\definecolor{shelle}{RGB}{160,160,160}
\definecolor{zeroblk}{RGB}{238,238,238}
\tikzset{linkbadge/.style={rounded rectangle, minimum height=14.6pt, inner ysep=0pt,
  text height=7.8pt, text depth=2.4pt}}
\newlength{\linkbadgeurlwidth}
\AtBeginDocument{\settowidth{\linkbadgeurlwidth}{\small\nolinkurl{https://huggingface.co/spaces/Can111/ShellElastic}}}
\newcommand{\linkbadge}[4]{{\hypersetup{pdfborder={0 0 0}}\href{#3}{%
  \begin{tikzpicture}[baseline=(label.base)]
    \node[linkbadge, rounded rectangle right arc=none, inner xsep=7pt, fill=ourcolor,
      text=white, font=\small\bfseries] (label) {\makebox[92pt][l]{#1\hspace{4pt}#2}};
    \node[linkbadge, rounded rectangle left arc=none, inner xsep=8pt, fill=ourcolor!9,
      font=\small, right=0pt of label]
      {\makebox[\linkbadgeurlwidth][l]{\textcolor{ourcolor!85!black}{\nolinkurl{#4}}}};
  \end{tikzpicture}}}\kern-\tabcolsep}
\newcommand{\badgeicon}[1]{\tikz[baseline=-3.6pt, x=4.2pt, y=4.2pt, line width=0.8pt,
  line cap=round, line join=round]{\useasboundingbox (-4.6pt,-4.6pt) rectangle (4.6pt,4.6pt); #1}}
\newcommand{\globeicon}{\badgeicon{\draw (0,0) circle (1); \draw (-1,0) -- (1,0);
  \draw (0,-1) arc[start angle=-40.8, end angle=40.8, radius=1.53]
    arc[start angle=139.2, end angle=220.8, radius=1.53];}}
\newcommand{\codeicon}{\badgeicon{\draw (0.4,0.6) -- (1,0) -- (0.4,-0.6);
  \draw (-0.4,0.6) -- (-1,0) -- (-0.4,-0.6);}}
\newcommand{\dataicon}{\badgeicon{\draw (0,0.7) ellipse [x radius=0.85, y radius=0.3];
  \draw (-0.85,0.7) -- (-0.85,-0.7) arc[start angle=180, end angle=360, x radius=0.85, y radius=0.3]
    -- (0.85,0.7);
  \draw (-0.85,0) arc[start angle=180, end angle=360, x radius=0.85, y radius=0.3];}}

\title{ShellElastic: Scalable Elastic Pretraining\\ with Simultaneous Maximal Updates\\ and Hyperparameter Transfer}

\author{%
Can Jin$^{1}$\thanks{Corresponding author: \texttt{can.jin@rutgers.edu}.
  Work done during an internship at Adobe Research.}\phantom{$^{*}$},
Yan Kang$^{2}$,
Yifan Gong$^{2}$,
Chongjian Ge$^{2}$,
Mingze Xu$^{2}$,
Yuanjun Xiong$^{2}$,\\
\bfseries Dimitris N. Metaxas$^{1}$,
Hongwu Peng$^{2}$\\
$^{1}$Rutgers University \qquad $^{2}$Adobe Research\\[5pt]
\linkbadge{\globeicon}{Project Page}{https://jincan333.github.io/ShellElastic_page}{https://jincan333.github.io/ShellElastic_page}\\[4pt]
\linkbadge{\codeicon}{Code}{https://github.com/jincan333/ShellElastic}{https://github.com/jincan333/ShellElastic}\\[4pt]
\linkbadge{\dataicon}{Checkpoints \& Data}{https://huggingface.co/spaces/Can111/ShellElastic}{https://huggingface.co/spaces/Can111/ShellElastic}
}

\iclrfinalcopy
\begin{document}

\maketitle
\lhead{Preprint.}

\begin{abstract}
Foundation models need to support diverse inference budgets, yet training a separate model for each budget duplicates pretraining effort. Elastic models train several nested sub-networks inside one weight-shared supernet, but existing recipes face three obstacles: variable per-step workloads that hinder distributed scaling, limited hyperparameter (HP) transfer across model sizes, and no guarantee that all sub-networks are maximally updated at once. We present \Ours, an elastic transformer that organizes hidden dimensions into a dense core and successive shells, so that every hidden matrix is block lower triangular. One full-width forward pass contains every sub-network's exact hidden states as prefixes, so joint training keeps a fixed computation graph and costs about as much per step as training the widest sub-network alone. We further introduce \muE, an extension of maximal update parameterization (\mup) to \Ours and other elastic architectures. Its scaling rules provably keep all sub-networks of a supernet maximally updated simultaneously (\muI), and keep doing so when the supernet is widened or given more shells; sweeps confirm that HPs tuned on a small non-elastic proxy stay near-optimal for elastic families and larger supernets (\mbox{\muT}). A single \Ours run produces sub-networks that outperform MatFormer and fully elastic baselines at matched parameter counts, and approach separately trained models that together take several times the steps.
\end{abstract}

\etocdepthtag.toc{mtchapter}
\section{Introduction}
\label{sec:intro}
Foundation models need to support different compute budgets for different tasks and deployment platforms. 
Elastic training answers this need by optimizing several nested sub-networks inside one shared parameter set: slimmable networks and Once-for-All established the idea for vision models~\citep{yu2019slimmable,cai2020onceforall}, MatFormer nests the FFN width of a transformer and underlies released on-device models~\citep{devvrit2024matformer,gemmateam2025gemma3n}, and Flextron and Nemotron Elastic make pretrained LLMs elastic by post-training~\citep{cai2024flextron,taghibakhshi2025nemotronelastic}.

Three problems remain when the elastic training recipe is carried into foundation-model pretraining.

\emph{Scalability.} Standard recipes sample one or a few sub-networks per step~\citep{devvrit2024matformer,yu2019universally}, so the computation graph, the parameter access pattern and the workload change from step to step; under parameter sharding this variability complicates activation memory, recomputation and distributed scheduling, the very machinery that large-scale pretraining depends on. Training all sub-networks at every step fixes the workload but needs a forward and backward pass per sub-network, so the cost grows with the family size and erodes the saving that motivates elastic training.

\emph{Transferability.}
Hyperparameters (HPs) tuned at one model size need not remain effective at another, and repeating the search over learning rate, initialization scale and weight decay at every target scale cancels the compute that elastic training saves. Maximal update parameterization transfers HPs across widths for independently trained models~\citep{yang2022tensorv}; elastic models add two questions: can HPs tuned on a non-elastic proxy serve an entire elastic family, and can those tuned on a smaller elastic family serve a larger one? We refer to these two forms of transfer as \muT for elastic models and seek one scaling rule that delivers both under parameter sharing.

\emph{Simultaneous maximal updates.}
Existing elastic training recipes do not generally ensure maximal updates for all sub-networks simultaneously.\footnote{A parameter group is \emph{maximally updated} when the pre-activations it produces are $\Ord{1}$ at initialization and move by $\Ord{\eta_0}$ after one optimizer step, with $\eta_0$ the base learning rate before width scaling, so that its contribution neither vanishes nor blows up as the width grows~\citep{yang2021featurelearning,yang2022tensorv}.} Without width-dependent scaling, HPs that yield well-scaled feature updates in one sub-network make updates vanish or blow up in another, and re-tuning one global value only trades one width against another, so some widths stay mis-scaled. The requirement is therefore a shared parameterization that keeps every width maximally updated simultaneously despite parameter sharing; we call this property \muI for elastic models.

We address the first problem with a new architecture, \Ours, and the other two with \muE. \Ours{} (Figure~\ref{fig:architecture}) partitions the residual stream into a dense core followed by ordered shells, each a multiple of a base width $g$, and each supported sub-network is the core plus a prefix of the shells. Every hidden matrix is block lower triangular: a shell reads the core, the earlier shells and itself, but never a later shell. Trunk normalization acts within each group, attention heads align with the group boundaries, and every supported width has its own readout. One full-width forward pass therefore contains the exact hidden states of every sub-network as prefixes (Proposition~\ref{prop:prefix}), and the per-step compute stays close to that of a single full-width pass however many sub-networks are trained, so every step has a fixed workload, as in non-elastic training.

\muE (Figure~\ref{fig:mue-schematic}) extends maximal update parameterization (\mup)~\citep{yang2022tensorv} to elastic models with two factors: a \muT factor $m=g/d_0$ that carries HPs from a non-elastic proxy of width $d_0$ to base width $g$, and a \muI factor for the different widths inside one shared model. Every group reads the same input coordinates in every sub-network that contains it, so we set its initialization, learning rate and weight decay from this fan-in, while each readout uses its own input width (Table~\ref{tab:widthparam}). This keeps every group maximally updated in all sub-networks at once (Theorem~\ref{thm:shell}), which uniform scaling across groups cannot do in general (Proposition~\ref{prop:noniid}). The same two factors apply to MatFormer and to fully elastic models (Appendix~\ref{app:other-param}).

Our main contributions are summarized as follows:
\begin{itemize}
\vspace{-1mm}
\item[\ding{72}] \textbf{Scalable elastic architecture.} We propose \Ours, an elastic transformer whose single full-width forward pass yields the exact hidden states of every sub-network as prefixes, so every step has a fixed workload rather than a varying one.
\item[\ding{72}] \textbf{\muE parameterization.} We introduce \muE, which transfers HPs from a non-elastic proxy to an elastic model and from smaller to larger elastic models (\muT), and keeps all sub-networks of a supernet maximally updated at once (\muI).
\item[\ding{72}] \textbf{Theoretical and empirical validation.} We prove that \muE keeps all sub-networks maximally updated at both levels of transfer; sweeps confirm that transferred HPs stay near-optimal and that \muE beats the standard parameterization on every elastic family.
\item[\ding{72}] \textbf{Comprehensive ablation studies.} We search architectures at equal parameters and ablate sub-network selection, loss weights, normalization scope, readout design, in-place distillation, sub-network ensembles and stage-residual gradient scaling for elastic pretraining.
\vspace{-2mm}
\end{itemize}

\begin{figure}[t]
\vspace{-4mm}
\centering
\includegraphics[width=\linewidth]{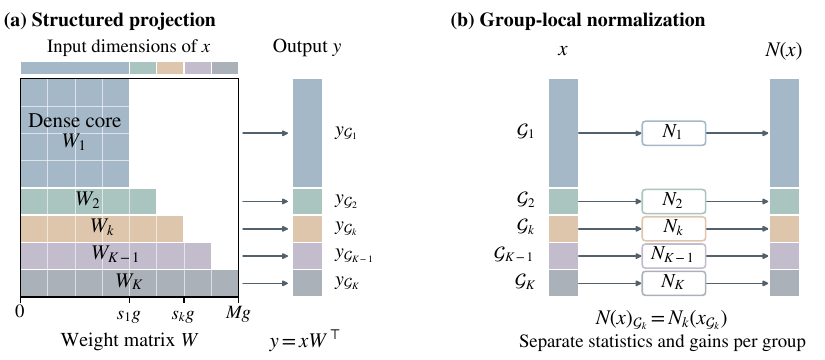}
\vspace{-6mm}
\caption{\textbf{Core and shells of \Ours{} preserve supported prefixes.} (a) Each colored block row produces output group $\mathcal G_k$ from the input prefix ending at $s_kg$. Group $\mathcal G_1$ is the dense core; subsequent groups are shells. (b) $N_k$ normalizes group $\mathcal G_k$ using its own statistics.}
\label{fig:architecture}
\vspace{-3mm}
\end{figure}

\vspace{-2mm}
\section{Related Work}
\label{sec:related}
\vspace{-2mm}

\textbf{Elastic models.}
Elastic models train several nested sub-networks in one shared parameter set, so one run serves many inference budgets~\citep{yu2019slimmable,cai2020onceforall,devvrit2024matformer}.
Slimmable networks train one convolutional network at several widths~\citep{yu2019slimmable,yu2019universally}, Once-for-All also nests depth, kernel size and resolution~\citep{cai2020onceforall}, and the same idea has been applied to the width and depth of BERT~\citep{devlin2019bert,hou2020dynabert} and to output representations~\citep{kusupati2022mrl}; these recipes train the sub-networks jointly or by sampling, usually with distillation from the largest one~\citep{hinton2015distilling}. For decoder-only transformers~\citep{vaswani2017attention}, MatFormer~\citep{devvrit2024matformer} nests only the FFN hidden width and underlies released on-device models~\citep{gemmateam2025gemma3n}, while Flextron and Nemotron Elastic extract nested sub-networks from a pretrained model~\citep{cai2024flextron,taghibakhshi2025nemotronelastic}.
The concurrent StairFormer~\citep{trost2026fullynested} nests every smaller model as a prefix of one forward pass and instantiates this nesting with equal-width blocks; in its experiments, it trails dense models of the same size at every budget but supports a wider range of usable budgets than MatFormer.
\Ours{} uses the same prefix structure but chooses the supported widths, which set the core and shell sizes, and normalizes within groups instead of over cumulative prefixes.
No existing work derives how HPs should scale across nested widths or keeps every sub-network maximally updated.

\textbf{HP transfer and $\mu$-parameterization.}
Tuning HPs such as the learning rate and initialization is expensive at scale, and under \SP{}~\citep{sohldickstein2020standard} their optima drift with width, so a small model cannot stand in for a large one~\citep{yang2022tensorv}.
Scaling HPs with model size removes this drift, so values tuned on a small proxy transfer zero-shot~\citep{yang2022tensorv,dey2025completep}.
Building on infinite-width theory~\citep{jacot2018ntk,mei2018meanfield}, \citet{yang2021featurelearning} identified \mup{}, a parameterization that gives every layer maximal $\Ord{1}$ feature updates as width grows, and \citet{yang2022tensorv} turned it into a transfer rule for initialization, learning rate and multipliers.
CompleteP extends the rule across depth for pre-norm transformers~\citep{yang2024depthmup,bordelon2024depthwise,dey2025completep}, and \citet{mlodozeniec2025completed} add batch size and token budget.
\citet{everett2024scaling} derive per-layer learning-rate exponents across parameterizations and optimizers, including Adam, and find that HP transfer is not limited to \mup{}; Table~\ref{tab:widthparam} follows their scaling rules.
These rules assume separately trained models, where each tensor has one fan-in; \muE{} extends them to the shared tensors of an elastic supernet.

\vspace{-2mm}
\section{Method}
\label{sec:method}
\vspace{-2mm}
In the \Ours{} architecture, all sub-networks reuse hidden-state prefixes from one full-width trunk pass; \muE{}, a parameterization based on each weight's active input width, enables reuse of base optimization settings across model sizes.
\vspace{-1mm}
\subsection{The \Ours{} architecture}
\label{sec:arch}
\vspace{-1mm}
\paragraph{Core, shells, and supported widths.}
We organize hidden dimensions into a dense \textbf{\emph{core}} followed by successive \textbf{\emph{shells}} (Figure~\ref{fig:architecture}). For base width $g$, choose supported sub-network width indices
\begin{equation}
 \mathcal S=\{s_1<\cdots<s_K=M\}\subseteq\{1,\ldots,M\},\qquad
 D_{\mathrm{model}}^{(s)}=sg,\quad D_{\mathrm{model}}=Mg.
\label{eq:widthset}
\end{equation}
With $s_0=0$, these boundaries define contiguous groups $\mathcal G_k=\{s_{k-1}g+1,\ldots,s_kg\}$ of size $(s_k-s_{k-1})g$. The core is $\mathcal G_1$, and the shells are $\mathcal G_2,\ldots,\mathcal G_K$. Sub-network $s_k$ contains the first $k$ groups.

\paragraph{Projections and normalization.}
Write a token's hidden state as a row vector $x$, with $x_{:d}$ denoting its first $d$ dimensions. A square projection is $y=xW^\top$, where $W\in\R^{D_{\mathrm{model}}\times D_{\mathrm{model}}}$. For output group $k$, the block row $W_k\in\R^{|\mathcal G_k|\times s_kg}$ reads only the first $s_kg$ input dimensions:
\begin{equation}
 y_{\mathcal G_k}=x_{:s_kg}W_k^\top.
\label{eq:grouplinear}
\end{equation}
This makes $W$ block lower triangular at the group boundaries (Figure~\ref{fig:architecture}a).

Group-local normalization uses a separate operator $N_k$ and statistics within each group (Figure~\ref{fig:architecture}b). For RMSNorm~\citep{zhang2019rmsnorm}:
\begin{equation}
 \big[N_k(x_{\mathcal G_k})\big]_j=a_j\frac{x_j}{\sqrt{|\mathcal G_k|^{-1}\sum_{u\in\mathcal G_k}x_u^2+\epsilon_N}},\qquad j\in\mathcal G_k.
\label{eq:norm}
\end{equation}
Here $a_j$ is a learned gain and $\epsilon_N>0$ is a numerical stabilizer. The combined normalization satisfies $N(x)_{\mathcal G_k}=N_k(x_{\mathcal G_k})$. Both operations preserve prefixes ending at the boundaries in $\mathcal S$. Cuts within a group generally change the retained outputs, so sub-networks can be extracted only at these boundaries. Thus, $\mathcal S$ determines the block-triangular pattern of every projection, the normalization groups, and the available sub-network widths.

\paragraph{Attention and FFN.}
Both branches (Figure~\ref{fig:block} in Appendix~\ref{app:architecture}) use the group-local normalization of Equation~\ref{eq:norm} as their pre-norm. Attention is elastic in its number of heads: sub-network $s$ uses the first $D_{\mathrm{model}}^{(s)}/d_{\mathrm{head}}=sg/d_{\mathrm{head}}$ heads, keeping the head dimension $d_{\mathrm{head}}$ fixed. We require $g$ to be a multiple of $d_{\mathrm{head}}$ so that supported widths retain complete heads. The query, key, value, gate, and output projections follow Equation~\ref{eq:grouplinear}; they constrain dependencies among hidden dimensions before and after attention. Query/key normalization and rotary position transformations~\citep{su2024rope} act within each head. Each active head $a$ performs causal self-attention:
\begin{equation}
 \operatorname{Attn}(x)=\left[\operatorname{concat}_a\!\left(\operatorname{softmax}\!\left(Q_aK_a^\top/\sqrt{d_{\mathrm{head}}}+C\right)V_a\right)\odot\operatorname{sigmoid}(xW_G^\top)\right]W_O^\top,
\label{eq:attention}
\end{equation}
where $C$ is the causal mask, $Q_a,K_a,V_a$ come from the corresponding projections of $x$ and concatenation runs over the active heads. Adding heads leaves earlier heads' attention outputs unchanged; the structured output projection then preserves the supported hidden-state prefixes.

We use a SwiGLU FFN~\citep{shazeer2020glu} with expansion $e=D_{\mathrm{ffn}}/D_{\mathrm{model}}$ and integer base FFN width $eg$, so sub-network $s$ has $D_{\mathrm{ffn}}^{(s)}=eD_{\mathrm{model}}^{(s)}$. All three maps use the same group dependency pattern. The up and gate weights have shape $D_{\mathrm{ffn}}\times D_{\mathrm{model}}$, with group boundaries $s_kg$ on the input axis and $es_kg$ on the output axis. The down weight has shape $D_{\mathrm{model}}\times D_{\mathrm{ffn}}$. The FFN computes
\begin{equation}
 \operatorname{FFN}(x)=\left[\silu(xW_1^\top)\odot(xW_3^\top)\right]W_2^\top.
\label{eq:ffn}
\end{equation}
Elementwise operations and residual additions preserve the same prefixes.

\begin{table}[t]
\vspace{-4mm}
\caption{\textbf{\muE for \Ours.} Scaling rules set by the \muT factor $m$, the \muI factor $s$ (Equation~\ref{eq:widthratio}), and the FFN expansion $e=D_{\mathrm{ffn}}/D_{\mathrm{model}}$; $s$ is the output group's supported boundary for hidden weights and the sub-network width index for readouts. Initialization gives nominal weight standard deviations, or initial values for gains and scalars.}
\label{tab:widthparam}
\centering
\begingroup
\small
\setlength{\tabcolsep}{2pt}
\renewcommand{\arraystretch}{1.12}
\begin{tabular*}{\linewidth}{@{\extracolsep{\fill}}lccccc@{}}
\toprule
\textbf{Parameter} & \textbf{Active input width} & \textbf{Initialization} & \textbf{Learning rate} & \textbf{Weight decay} & \textbf{Multiplier} \\
\midrule
Token embedding $E$ & $V$ (fixed) & $\sigma_0$ & $\eta_0$ & $\lambda_0$ & $1$ \\
Attention; FFN up/gate & $sg$ & $\sigma_0/\sqrt{ms}$ & $\eta_0/(ms)$ & $\lambda_0ms$ & $1$ \\
FFN down & $esg$ & $\sigma_0\sqrt e/\sqrt{ms}$ & $\eta_0/(ms)$ & $\lambda_0ms$ & $1/e$ \\
Readout $U_s$ & $sg$ & $\sigma_0/\sqrt{ms}$ & $\eta_0/(ms)$ & $\lambda_0ms$ & $\alpha_s$ \\
Norm gains; readout scalar $\alpha_s$ & $1$ (elementwise) & $1$ & $\eta_0$ & $0$ & $1$ \\
\bottomrule
\end{tabular*}
\endgroup
\vspace{-3mm}
\end{table}

\paragraph{Embedding and readout.}
Sub-network $s$ uses $sg$-dimensional token embeddings and computes its logits from the first $sg$ dimensions of the $L$-layer trunk output $x^{(L)}$:
\begin{equation}
 z_s=\alpha_s\,N_s^{\rm out}(x^{(L)}_{:sg})U_s^\top,\qquad s\in\mathcal S.
\label{eq:head}
\end{equation}
Here $\alpha_s$ is a learned logit scale, $N_s^{\rm out}$ normalizes over all $sg$ selected dimensions, and $U_s\in\R^{V\times sg}$ is the readout matrix for a vocabulary of size $V$. Figure~\ref{fig:full-shellelastic} in Appendix~\ref{app:architecture} shows the complete model.

\begin{proposition}[Exact reuse at supported widths]
\label{prop:prefix}
For fixed model parameters, let $E_s$ be the first $sg$ columns of the full-width embedding table, $T_s$ the $L$-layer trunk restricted to width $sg$, and $P_s$ the operator selecting the first $sg$ hidden dimensions. For any token sequence $u$, evaluating sub-network $s$ directly gives the same hidden states as extracting its prefix from sub-network $t$:
\[
 T_s(E_s[u])=P_s\!\left(T_t(E_t[u])\right),\qquad s,t\in\mathcal S,\quad s\leq t.
\]
\end{proposition}
Applying sub-network $s$'s normalization and readout (Equation~\ref{eq:head}) to either result gives identical logits. Thus, one full-width trunk evaluation yields the same joint loss and parameter gradients as separate evaluations of all supported sub-networks. Appendix~\ref{app:prefixproof} proves the result by composing the prefix-preserving trunk operations across the $L$ layers.

\vspace{-2mm}
\subsection{\texorpdfstring{\muE: parameterizing elastic models}{mu-Elastic: parameterizing elastic models}}
\label{sec:param}
\vspace{-2mm}
We introduce \muE to address the two optimization requirements (Figure~\ref{fig:mue-schematic}): transferring HPs from a non-elastic proxy to an elastic family and between elastic model sizes (\muT), and maintaining maximal updates across sub-networks within each shared model (\muI).

\begin{figure}[t]
\vspace{-4mm}
\centering
\includegraphics[width=\linewidth]{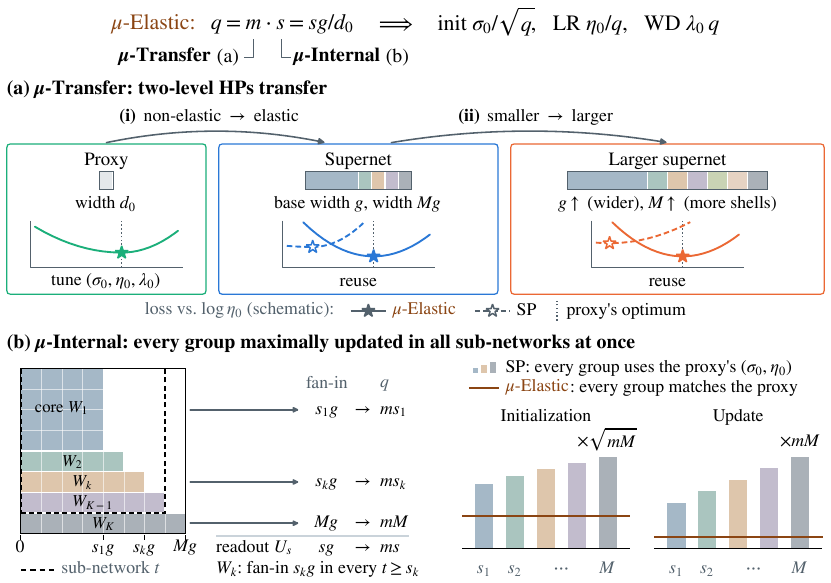}
\caption{\textbf{Overview of \muE{}.} The ratio $q=m\cdot s$ of Equation~\ref{eq:widthratio} scales the HPs of every group and readout. (a)~\muT{}: the $(\sigma_0,\eta_0,\lambda_0)$ tuned on the proxy are reused (i)~for an elastic supernet and (ii)~for a wider one with more shells (Corollary~\ref{cor:transfer}). (b)~\muI{}: block row $W_k$ has fan-in $s_kg$ in every sub-network that contains it, so $q=ms_k$ keeps it maximally updated in all of them (Theorem~\ref{thm:shell}). The bars show each group's pre-activation size at initialization and update size: with the proxy's $(\sigma_0,\eta_0)$, as under SP, all groups, even the core, exceed the proxy by $\sqrt{q}$ and $q$, so no shared value fixes them all (Proposition~\ref{prop:noniid}); under \muE{}, all match it.}
\label{fig:mue-schematic}
\vspace{-3mm}
\end{figure}

\paragraph{Composing transfer and internal scaling.}
Let a non-elastic proxy of width $d_0$ supply the base initialization scale $\sigma_0$, base learning rate $\eta_0$, and weight decay $\lambda_0$. For base width $g$, define
\begin{equation}
 m=\frac{g}{d_0},\qquad
 q(s)=\underbrace{m}_{\mu\text{-Transfer}}\cdot
 \underbrace{s}_{\mu\text{-Internal}}=\frac{sg}{d_0},\quad s\in\mathcal S.
\label{eq:widthratio}
\end{equation}
For a hidden weight, $s$ is the supported boundary of its output group, whose input width stays $sg$ ($esg$ for an FFN down projection) in every sub-network that contains it, so one rule based on $q(s)$ serves all of them. The readout of sub-network $s$ uses the same ratio.

\paragraph{Parameterization and optimization rules.}
Table~\ref{tab:widthparam} summarizes initialization, learning rate, weight decay, and forward multipliers by active input width for AdamW~\citep{kingma2015adam,loshchilov2019adamw}, following the per-layer scaling of~\citet{everett2024scaling}. Width is the only elastic axis, since all sub-networks share the depth $L$, minibatches and number of training steps; these shared quantities follow existing scaling rules for depth~\citep{dey2025completep}, batch size and token budget~\citep{mlodozeniec2025completed}. The learning rate and weight decay scalings cancel, so AdamW applies the same relative decay to every decayed parameter group. The embedding is a lookup over a fixed vocabulary of size $V$, so it uses the base $(\sigma_0,\eta_0,\lambda_0)$. For FFN down projections, the $\sqrt e$ initialization factor and $1/e$ forward multiplier account for the expanded input width $esg$. The learned scalars $\alpha_s$ in Equation~\ref{eq:head} adjust logit coefficients beyond width scaling. Norm gains and $\alpha_s$ are initialized to one and use learning rate $\eta_0$ with zero weight decay. The \emph{standard parameterization} (SP) sets $m=s=e=1$ only in the scaling rules and fixes $\alpha_s=1$, preserving the architecture.

\paragraph{Theoretical guarantees.}
The statements below are proved in Appendix~\ref{app:proofs} under the standard heuristic assumptions of \mup~\citep{yang2022tensorv}: before the scaling of Table~\ref{tab:widthparam} is applied, the pre-activations of a group with active input width $sg$ are $\Ord{\sqrt{sg}}$ at initialization and move by $\Ord{sg}$ after one step, which is why the initialization and the learning rate carry different exponents.

\begin{theorem}[\muI: simultaneous maximal updates]
\label{thm:shell}
Under the assumptions of Appendix~\ref{app:proofs} and the rule of Table~\ref{tab:widthparam}, consider a hidden group with supported boundary $s$ and any sub-network $t\in\mathcal S$ with $t\ge s$. In every layer, its pre-activations inside sub-network $t$ (i) have $\Ord{1}$ coordinates at initialization and (ii) move by $\Ord{\eta_0}$ after one AdamW step on that group, with constants independent of $s$, $t$, $M$, and $m$. The logits of every readout $U_s$ have the same two properties.
\end{theorem}
Equation~\ref{eq:widthratio} is what makes this possible: a group with boundary $s$ sums over exactly $sg$ input coordinates in every sub-network containing it, so one initialization scale and one learning rate satisfy both conditions for all $t\ge s$ at once. Training several sub-networks jointly does not change this, since they all feed a group the same input and Adam is invariant to the gradient scale.

\begin{proposition}[A shared rule leaves groups mismatched]
\label{prop:noniid}
If all groups of a hidden matrix share one initialization variance, learning rate and multiplier, the groups with boundaries $s<s'$ differ by factors $\Ord{\sqrt{s'/s}}$ in pre-activation size and $\Ord{s'/s}$ in update size, whatever the shared values. For unbounded $M/s_1$, no shared rule thus keeps every group maximal in the sense of Theorem~\ref{thm:shell}.
\end{proposition}
Under SP, the groups of a hidden matrix share $(\sigma_0,\eta_0)$, so by Proposition~\ref{prop:noniid} re-tuning only trades the core against the outermost group, which stay $\sqrt{M/s_1}$ and $M/s_1$ apart in initialization and update; the best HPs are compromises depending on $m$, $M$ and $\mathcal S$. Figure~\ref{fig:mutransfer} and Table~\ref{tab:muinternal} show both effects: the SP optimum departs from the proxy's, and SP trails \muE{} even after re-tuning.

\begin{corollary}[\muT: two-level transfer]
\label{cor:transfer}
Theorem~\ref{thm:shell} holds for every base width $g=md_0$ and every supported set $\mathcal S\subseteq\{1,\dots,M\}$. Hence the values $(\sigma_0,\eta_0,\lambda_0)$ tuned on the width-$d_0$ proxy (i) keep every group and readout of a \Ours{} supernet maximally updated, and (ii) keep doing so when the supernet is widened (larger $m$) or given more shells (larger $M$).
\end{corollary}
As with \mup{}, the corollary fixes the scales of the initialization and the updates, not the optimal HP values; Figure~\ref{fig:mutransfer} verifies that the proxy's tuned HP values remain near-optimal at both levels.

\paragraph{\muE for other elastic architectures.}
\looseness=-1
\muE{} carries over to other elastic architectures, but $m$ and $s$ are defined per architecture. MatFormer~\citep{devvrit2024matformer} keeps the model width fixed and nests only the FFN hidden width, so only its FFN down projection has a fan-in that varies with the sub-network. It uses $m=D_{\mathrm{model}}/d_0$, with $s$ the sub-network's FFN hidden width in units of a base FFN width; the down projection takes the factor $ms$, and the other hidden matrices and the readout take $m$ (Table~\ref{tab:param-ffn}). Fully elastic models, a transformer adaptation of slimmable networks~\citep{yu2019slimmable}, slice shared dense matrices along the residual, attention and FFN widths, so every hidden matrix has a varying fan-in; their $m$ and $s$ follow Equation~\ref{eq:widthratio}, and every hidden matrix takes the factor $ms$ (Table~\ref{tab:param-full}). StairFormer~\citep{trost2026fullynested}, with equal-width groups, takes Table~\ref{tab:widthparam} with $\mathcal S=\{1,\dots,M\}$ in its hidden matrices. Theorem~\ref{thm:shell} and Corollary~\ref{cor:transfer} extend to all three (Appendix~\ref{app:proof-other}), and Figure~\ref{fig:architecture-comparison} compares MatFormer, fully elastic models and \Ours{}.
\vspace{-2mm}
\subsection{Training and inference}
\label{sec:training}
\vspace{-2mm}
\paragraph{Recipe.}
(1) Tune $(\sigma_0,\eta_0,\lambda_0)$ on a non-elastic proxy of width $d_0$.
(2) Choose the base width $g$ and the supported set $\mathcal S$.
(3) Compute $m=g/d_0$ and instantiate the parameterization in Table~\ref{tab:widthparam}.
(4) Train the supernet on the joint objective below.

\paragraph{Joint objective.}
At step $\tau$, with selected widths $\mathcal A_\tau\subseteq\mathcal S$ and loss weights $\omega_s\ge0$,
\begin{equation}
 \mathcal{L}_\tau=\frac{\sum_{s\in\mathcal A_\tau}\omega_s\,\ell_s}{\sum_{s\in\mathcal A_\tau}\omega_s},
 \qquad
 \ell_s=\CE(z_s,y)+\lambda_z\,\E\!\left[\Big(\log\textstyle\sum_{v}e^{z_{s,v}}\Big)^{2}\right],
\label{eq:loss}
\end{equation}
where $z_s$ are the logits of Equation~\ref{eq:head}, $\lambda_z$ is the $z$-loss coefficient, the expectation runs over the non-padded target positions, and every $\ell_s$ backpropagates into the shared trunk.
The narrower widths may be selected by any rule, but every $\mathcal A_\tau$ contains the full width $M$, so the trunk always runs at width $D_{\mathrm{model}}$ and every trunk parameter receives a gradient at every step.

\paragraph{Cost and scalability.}
Let $C_{\mathrm{trunk}}(s)$ and $C_{\mathrm{out}}(s)$ be the forward costs of the trunk and of one terminal norm and readout at width $sg$. Training the family jointly and training each sub-network separately cost
\begin{equation}
 C_{\mathrm{joint}}=C_{\mathrm{trunk}}(M)+\sum_{s\in\mathcal S}C_{\mathrm{out}}(s),
 \qquad
 C_{\mathrm{separate}}=\sum_{s\in\mathcal S}\Big[C_{\mathrm{trunk}}(s)+C_{\mathrm{out}}(s)\Big].
\label{eq:cost}
\end{equation}
By Proposition~\ref{prop:prefix}, one trunk term serves every width, whereas $C_{\mathrm{separate}}$ pays one per width, so the gap grows with $|\mathcal S|$ and $C_{\mathrm{joint}}$ stays comparable to training the widest sub-network alone.
The same holds for memory: all readout gradients combine in one trunk backward pass, so the trunk's activations, parameters and optimizer state are those of the widest one.
Training all sub-networks at every step works best (Appendix~\ref{app:ablation-studies}) and gives every step the same computation graph, parameter access pattern and workload.
In MatFormer and fully elastic models, hidden states are not nested, so each width needs its own trunk pass: sampling widths per step~\citep{devvrit2024matformer,yu2019universally} varies the workload, while training all widths at every step~\citep{yu2019slimmable} pays $C_{\mathrm{separate}}$.

\paragraph{Inference.}
For a chosen width $s\in\mathcal S$, we retain the embedding prefix, the core and shells up to $s$, and the matching terminal norm and readout, keeping the forward multipliers of Table~\ref{tab:widthparam}. By Proposition~\ref{prop:prefix}, the extracted model reproduces its supernet logits exactly and runs as a standalone transformer of width $sg$ without further training.

\vspace{-2mm}
\section{Experiments}
\label{sec:experiments}
\vspace{-2mm}
After describing the setup (Section~\ref{sec:setup}), we evaluate the sub-networks of one \Ours{} run against parameter-matched MatFormer, fully elastic and separately trained models, and \muE{} against SP on every elastic family (Section~\ref{sec:main}).
We then compare \Ours{} with StairFormer and run an architecture search at equal parameters and width count (Section~\ref{sec:arch-compare}); finally, we ablate the training and design choices to derive practical guidance for elastic pretraining (Section~\ref{sec:ablation}).
\vspace{-1mm}
\subsection{Experimental details}
\label{sec:setup}
\vspace{-1mm}
\paragraph{Models.}
Our primary model is a \Ours{} supernet with base width $g=128$ and supported set $\mathcal S=\{4,5,6,7,8\}$: a $512$-dimensional core followed by four $128$-dimensional shells, giving five sub-networks of widths $512$ to $1024$.
We report main results at three scales, with sub-networks of about $119$M--$285$M, $375$M--$906$M and $675$M--$1.69$B parameters (\Ours{} widths $512$--$1024$, $1024$--$2048$ and $1024$--$2048$, respectively). All models are $32$-layer transformers with gated attention and SwiGLU FFNs.
Tables~\ref{tab:models} and~\ref{tab:arch} in Appendix~\ref{app:hparams} list all configurations.

\paragraph{Datasets.}
All models are trained on DCLM-Baseline~\citep{li2024dclm}, a filtered Common Crawl corpus, with the GPT-NeoX-20B tokenizer~\citep{black2022gptneox} and a sequence length of $2{,}048$.
By default, a run takes $25{,}000$ steps of $128$ sequences ($6.55$B tokens).
This budget gives every sub-network more than $22$ tokens per parameter, above the compute-optimal ratio of about $20$~\citep{hoffmann2022chinchilla}.
We monitor the training and validation losses (per-token cross-entropy in nats) of every sub-network, the latter on held-out C4 text~\citep{raffel2020t5,dodge2021c4}.
Following OLMoE~\citep{muennighoff2025olmoe} and DCLM, we also probe 5-shot accuracy on ARC-Easy~\citep{clark2018arc}, PIQA~\citep{bisk2020piqa} and HellaSwag~\citep{zellers2019hellaswag} during training and evaluate the final sub-networks on $13$ datasets spanning world knowledge, symbolic problem solving, commonsense reasoning, language understanding and reading comprehension (Appendix~\ref{app:eval}).

\paragraph{Baselines.}
We compare \Ours{} against five baselines, each providing networks of five sizes: Low, Medium, High, Extra and Max.
(i)~\emph{MatFormer}~\citep{devvrit2024matformer} nests only the FFN hidden width.
(ii)~\emph{Fully elastic}, a transformer adaptation of slimmable networks~\citep{yu2019slimmable}, nests the residual, attention and FFN widths by slicing shared dense matrices.
(iii)~\emph{StairFormer}~\citep{trost2026fullynested} shares the prefix structure of \Ours{} but uses equal-width groups and normalizes over cumulative prefixes.
(iv)~\emph{Standalone} trains five separate models, each on the full budget, and thus takes $5\times$ the steps of one elastic run.
(v)~\emph{Sampled standalone}, the standalone baseline of MatFormer, trains the five \emph{Standalone} models within the steps of one elastic run.
MatFormer, the fully elastic model and sampled standalone sample one sub-network uniformly at random per step, whereas StairFormer and \Ours{} train all sub-networks in one full-width pass at every step.
\textbf{In the main results, all baselines match the \Ours{} sub-networks in parameter count to within $3\%$ (Table~\ref{tab:main-params}).}
Appendix~\ref{app:hparams} details the baselines' architecture and training, and Appendix~\ref{app:other-param} gives their parameterization.

\paragraph{Hyperparameters.}
All models are trained with AdamW. Unless otherwise specified, they use the base learning rate $\eta_0=2\times10^{-3}$, initialization scale $\sigma_0=3.3\times10^{-3}$ and weight decay $\lambda_0=0.01$, with \muE{} for the elastic models and \mup{} for the non-elastic ones, both with $d_0=128$.
\Ours{} has $m=1$ and trains all five widths at every step with loss weights $(1,1,1,1,4)$ (Equation~\ref{eq:loss}).
Main results at the $119$M--$285$M and $375$M--$906$M scales are averaged over three runs  with the same seed; compute limits the $675$M--$1.69$B scale to one run.
Ablations vary only the setting under study; Table~\ref{tab:hparams} in Appendix~\ref{app:hparams} lists the remaining HPs and the hardware.
\vspace{-2mm}
\subsection{Main results}
\label{sec:main}
\vspace{-2mm}
\begin{table}[t]
\vspace{-4mm}
\caption{Final validation loss and mean 5-shot accuracy over $13$ downstream tasks at three scales, labeled by sub-network size. Entries at the $119$M--$285$M and $375$M--$906$M scales average three runs; those at $675$M--$1.69$B are single runs. $^\dagger$Five separately trained models, together taking $5\times$ the steps of any other row. Bold marks the best of the four methods that train all five sizes in one run.}
\label{tab:main-scales}
\centering
\small
\setlength{\tabcolsep}{2pt}
\begin{tabular}{@{}c@{\hspace{3pt}}lcccccc@{\hspace{9pt}}cccccc@{}}
\toprule
 & & \multicolumn{6}{c}{Val.\ loss $\downarrow$} & \multicolumn{6}{c}{5-shot acc.\ (\%) $\uparrow$} \\
\cmidrule(r{4pt}){3-8} \cmidrule(l{0pt}){9-14}
 & Method & Low & Medium & High & Extra & Max & Avg. & Low & Medium & High & Extra & Max & Avg. \\
\midrule
\multirow{5.2}{*}[1.1pt]{\rotatebox[origin=c]{90}{\strut $119$M--$285$M}} & Standalone$^\dagger$ & 3.369 & 3.308 & 3.264 & 3.219 & 3.191 & 3.270 & 35.6 & 35.8 & 36.6 & 37.1 & 37.1 & 36.5 \\ \addlinespace[2pt]
 & Sampled standalone & 3.667 & 3.609 & 3.583 & 3.532 & 3.524 & 3.583 & 33.8 & 33.4 & 34.0 & 34.2 & 34.4 & 34.0 \\
 & MatFormer & 3.425 & 3.357 & 3.325 & 3.312 & 3.314 & 3.347 & 35.1 & 35.2 & 35.8 & 35.9 & 35.9 & 35.6 \\
 & Fully elastic & 3.429 & 3.354 & 3.315 & 3.295 & 3.290 & 3.336 & 34.9 & 35.5 & 35.7 & 35.9 & 36.0 & 35.6 \\
 & \textbf{\Ours} & \textbf{3.409} & \textbf{3.338} & \textbf{3.294} & \textbf{3.264} & \textbf{3.240} & \textbf{3.309} & \textbf{35.5} & \textbf{36.6} & \textbf{36.6} & \textbf{36.9} & \textbf{36.9} & \textbf{36.5} \\
\midrule
\multirow{5.2}{*}[1.1pt]{\rotatebox[origin=c]{90}{\strut $375$M--$906$M}} & Standalone$^\dagger$ & 3.007 & 2.963 & 2.926 & 2.897 & 2.870 & 2.933 & 40.4 & 41.8 & 42.5 & 43.2 & 43.1 & 42.2 \\ \addlinespace[2pt]
 & Sampled standalone & 3.208 & 3.175 & 3.141 & 3.115 & 3.093 & 3.147 & 37.7 & 37.8 & 38.0 & 38.7 & 39.0 & 38.2 \\
 & MatFormer & 3.072 & 3.005 & 2.975 & 2.963 & 2.963 & 2.996 & 39.7 & 40.3 & 40.4 & 41.2 & 41.3 & 40.6 \\
 & Fully elastic & \textbf{3.052} & \textbf{2.995} & \textbf{2.966} & 2.952 & 2.948 & 2.983 & 40.3 & 41.0 & 41.6 & 42.0 & 42.0 & 41.4 \\
 & \textbf{\Ours} & 3.055 & 3.000 & 2.967 & \textbf{2.943} & \textbf{2.925} & \textbf{2.978} & \textbf{40.8} & \textbf{41.4} & \textbf{42.0} & \textbf{42.6} & \textbf{43.0} & \textbf{41.9} \\
\midrule
\multirow{5.2}{*}[1.1pt]{\rotatebox[origin=c]{90}{\strut $675$M--$1.69$B}} & Standalone$^\dagger$ & 2.912 & 2.859 & 2.821 & 2.793 & 2.767 & 2.830 & 43.4 & 43.9 & 44.7 & 46.2 & 46.1 & 44.8 \\ \addlinespace[2pt]
 & Sampled standalone & 3.037 & 2.998 & 2.970 & 2.956 & 2.924 & 2.977 & 40.2 & 40.8 & 41.6 & 41.0 & 43.2 & 41.4 \\
 & MatFormer & \textbf{2.914} & 2.872 & 2.849 & 2.839 & 2.838 & 2.862 & 42.8 & 43.3 & 43.4 & 44.0 & 43.6 & 43.4 \\
 & Fully elastic & 2.940 & 2.884 & 2.854 & 2.838 & 2.833 & 2.870 & 42.1 & 43.4 & \textbf{43.9} & 44.2 & \textbf{44.6} & 43.6 \\
 & \textbf{\Ours} & \textbf{2.914} & \textbf{2.866} & \textbf{2.835} & \textbf{2.813} & \textbf{2.796} & \textbf{2.845} & \textbf{43.7} & \textbf{43.7} & \textbf{43.9} & \textbf{44.3} & 44.5 & \textbf{44.0} \\
\bottomrule
\end{tabular}
\vspace{-3mm}
\end{table}

\paragraph{Comparison with baselines across scales.}
Table~\ref{tab:main-scales} compares \Ours{} with MatFormer, the fully elastic model and the two standalone baselines at three scales (curves and per-task results in Appendix~\ref{app:main-full}); all three elastic models use \muE{} (Appendix~\ref{app:other-param}).
\Ours{} is the supernet of Section~\ref{sec:setup} at $119$M--$285$M, doubles its base width to $g=256$ at $375$M--$906$M, and also raises the FFN expansion from $e=1$ to $4$ at $675$M--$1.69$B; these two scales train for $50{,}000$ and $100{,}000$ steps of $256$ sequences (Appendix~\ref{app:hparams}).
Among the methods that train all five sizes in one run, \Ours{} has the lowest mean validation loss ($0.028$, $0.005$ and $0.018$ nats below the best baseline) and the highest mean 5-shot accuracy at every scale, and is best or tied-best at $12$ of the $15$ tiers in validation loss and $14$ in 5-shot accuracy.
Its lead grows with size, from at most $0.016$ nats at Low to $0.023$--$0.050$ at Max, likely because the Max-only parameters of the elastic baselines receive a gradient on only one step in five on average, whereas every parameter of \Ours{} receives one at every step.
At a fixed step budget, weight sharing also helps: the fully elastic model beats sampled standalone, its unshared counterpart, by $0.11$ to $0.25$ nats on average.
With a fifth of their optimizer steps, \Ours{} comes within $0.039$, $0.045$ and $0.014$ nats of the standalone models in mean validation loss and within $0.8$ points in mean 5-shot accuracy.
Appendix~\ref{app:cost} analyzes the training and inference cost of all methods at $119$M--$285$M.

\begin{figure}[t]
\vspace{-4mm}
\centering
\includegraphics[width=\linewidth]{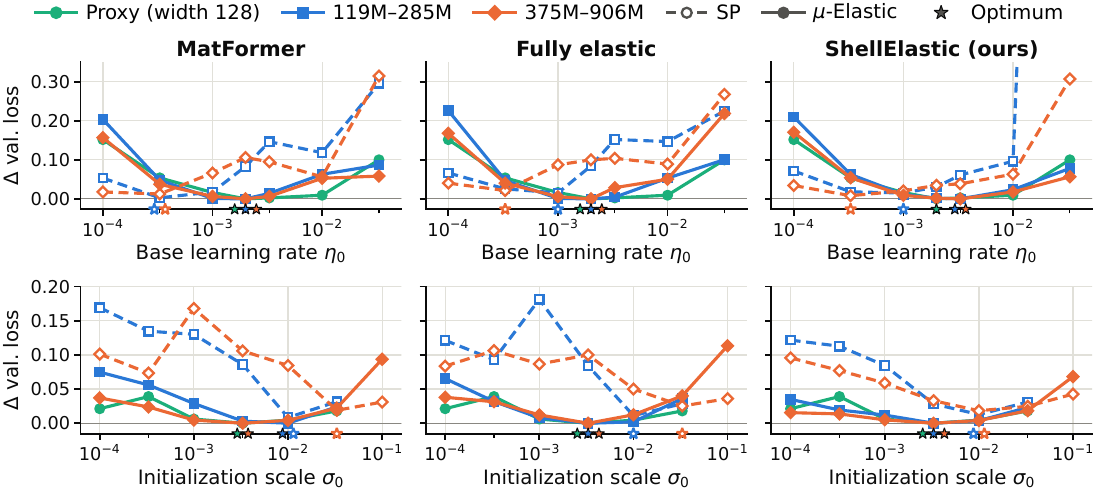}
\vspace{-6mm}
\caption{Sweeps of the learning rate $\eta_0$ and initialization scale $\sigma_0$, the other at its base value, for the proxy and each elastic model at both scales under \muE{} and SP. $\Delta$ val.\ loss is the mean validation loss of the five sub-networks minus the minimum of the \muE{} line of the same model and scale (for the proxy, its own line). Appendix~\ref{app:mutransfer-tiers} shows the sweeps for each sub-network.}
\label{fig:mutransfer}
\vspace{-2mm}
\end{figure}

\begin{table}[t]
\vspace{-3mm}
\caption{Effect of \muE{} on each elastic model at the $119$M--$285$M and $375$M--$906$M scales. For each model, the first row is the best of the $12$ ($13$ at $375$M--$906$M) SP configurations of Figure~\ref{fig:mutransfer} at that scale (lowest mean validation loss), and the w/ \muE{} row is its result in Table~\ref{tab:main-scales}, with the base values of Section~\ref{sec:setup} and no re-tuning. Table~\ref{tab:muinternal-full} gives the values with four decimals.}
\label{tab:muinternal}
\centering
\small
\setlength{\tabcolsep}{2pt}
\begin{tabular}{@{}c@{\hspace{3pt}}lcccccc@{\hspace{9pt}}cccccc@{}}
\toprule
 & & \multicolumn{6}{c}{Train loss $\downarrow$} & \multicolumn{6}{c}{Val.\ loss $\downarrow$} \\
\cmidrule(r{4pt}){3-8} \cmidrule(l{0pt}){9-14}
 & Method & Low & Medium & High & Extra & Max & Avg. & Low & Medium & High & Extra & Max & Avg. \\
\midrule
\multirow{6.4}{*}[1.1pt]{\rotatebox[origin=c]{90}{\strut $119$M--$285$M}} & MatFormer & 3.261 & 3.184 & 3.157 & \textbf{3.140} & \textbf{3.139} & 3.176 & 3.434 & 3.360 & 3.327 & 3.314 & \textbf{3.311} & 3.349 \\
 & \quad w/ \muE{} & \textbf{3.254} & \textbf{3.182} & \textbf{3.156} & 3.141 & 3.144 & \textbf{3.175} & \textbf{3.425} & \textbf{3.357} & \textbf{3.325} & \textbf{3.312} & 3.314 & \textbf{3.347} \\ \addlinespace[2pt]
 & Fully elastic & 3.272 & 3.192 & 3.154 & 3.130 & 3.126 & 3.175 & 3.445 & 3.368 & 3.325 & 3.303 & 3.297 & 3.348 \\
 & \quad w/ \muE{} & \textbf{3.257} & \textbf{3.178} & \textbf{3.145} & \textbf{3.123} & \textbf{3.120} & \textbf{3.165} & \textbf{3.429} & \textbf{3.354} & \textbf{3.315} & \textbf{3.295} & \textbf{3.290} & \textbf{3.336} \\ \addlinespace[2pt]
 & \textbf{\Ours} & 3.247 & 3.180 & 3.136 & 3.104 & 3.079 & 3.149 & 3.423 & 3.354 & 3.308 & 3.275 & 3.249 & 3.322 \\
 & \quad \textbf{w/ \muE{}} & \textbf{3.230} & \textbf{3.164} & \textbf{3.123} & \textbf{3.094} & \textbf{3.071} & \textbf{3.137} & \textbf{3.409} & \textbf{3.338} & \textbf{3.294} & \textbf{3.264} & \textbf{3.240} & \textbf{3.309} \\
\midrule
\multirow{6.4}{*}[1.1pt]{\rotatebox[origin=c]{90}{\strut $375$M--$906$M}} & MatFormer & 2.914 & 2.841 & 2.815 & 2.796 & 2.793 & 2.832 & 3.089 & 3.019 & 2.988 & 2.974 & 2.970 & 3.008 \\
 & \quad w/ \muE{} & \textbf{2.899} & \textbf{2.828} & \textbf{2.803} & \textbf{2.786} & \textbf{2.787} & \textbf{2.821} & \textbf{3.072} & \textbf{3.005} & \textbf{2.975} & \textbf{2.963} & \textbf{2.963} & \textbf{2.996} \\ \addlinespace[2pt]
 & Fully elastic & 2.906 & 2.842 & 2.813 & 2.790 & 2.786 & 2.827 & 3.081 & 3.020 & 2.987 & 2.969 & 2.963 & 3.004 \\
 & \quad w/ \muE{} & \textbf{2.879} & \textbf{2.817} & \textbf{2.793} & \textbf{2.773} & \textbf{2.771} & \textbf{2.807} & \textbf{3.052} & \textbf{2.995} & \textbf{2.966} & \textbf{2.952} & \textbf{2.948} & \textbf{2.983} \\ \addlinespace[2pt]
 & \textbf{\Ours} & 2.888 & 2.830 & 2.796 & 2.770 & \textbf{2.748} & 2.806 & 3.070 & 3.009 & 2.973 & 2.946 & \textbf{2.924} & 2.984 \\
 & \quad \textbf{w/ \muE{}} & \textbf{2.875} & \textbf{2.822} & \textbf{2.790} & \textbf{2.767} & 2.749 & \textbf{2.801} & \textbf{3.055} & \textbf{3.000} & \textbf{2.967} & \textbf{2.943} & 2.925 & \textbf{2.978} \\
\bottomrule
\end{tabular}
\vspace{-3mm}
\end{table}

\paragraph{Transferability of \muE{}.}
Under \muE{}, the proxy's optimum, $\eta_0=2\times10^{-3}$ and $\sigma_0=3.3\times10^{-3}$ (the base values of Section~\ref{sec:setup}), transfers to all three elastic models at both levels of \muT{} (Figure~\ref{fig:mutransfer}, Corollary~\ref{cor:transfer}).
At the $119$M--$285$M scale, both values are optimal or adjacent to the optimum ($\eta_0$: exact for MatFormer and the fully elastic model, adjacent for \Ours{}; $\sigma_0$: exact for the fully elastic model and \Ours{}, adjacent for MatFormer).
At the $375$M--$906$M scale, which enlarges $m$, $M$ or both and doubles the batch size and steps (Appendix~\ref{app:hparams}), the same holds, and $\sigma_0$ becomes exact for MatFormer too.
The base values are thus within $0.003$ nats of the best point of every \muE{} sweep and within $0.005$ nats for every sub-network (Appendix~\ref{app:mutransfer-tiers}).
Under SP, the optimum departs from the proxy's and shifts with scale, as Section~\ref{sec:param} anticipates, and at the base values SP is $0.03$--$0.11$ nats worse.
Weight decay transfers as well (Figure~\ref{fig:mutransfer-wd} in Appendix~\ref{app:transfer-extra}): under \muE{}, the mean loss of \Ours{} is flat over the same range as the proxy's, from $\lambda_0=10^{-3}$ to $10^{-1}$ (within $0.003$ nats), whereas SP is higher at every swept value.

\paragraph{Simultaneous maximal updates of \muE{}.}
Without width-dependent scaling, re-tuning one global value only trades one width against another (Proposition~\ref{prop:noniid}).
Table~\ref{tab:muinternal} tests this limitation at the $119$M--$285$M and $375$M--$906$M scales, giving SP the best of its $12$ ($13$ at $375$M--$906$M) configurations in Figure~\ref{fig:mutransfer} and \muE{} only the base values.
At the $119$M--$285$M scale, for the fully elastic model and \Ours{}, \muE{} still gives lower training and validation losses at every tier, by $0.011$ and $0.013$ nats in mean validation loss, as it keeps all widths maximally updated at once (Theorem~\ref{thm:shell}).
At the $375$M--$906$M scale, tuned SP remains \mbox{$0.006$--$0.021$} nats worse than \muE{} for all three models (Appendix~\ref{app:mutransfer-tiers}).
Under tuned SP, \Ours{} still beats both baselines by $0.026$--$0.027$ and $0.020$--$0.024$ nats at the two scales, so its lead is not due to \muE{} alone.
\vspace{-2mm}
\subsection{Architecture comparison and search}
\label{sec:arch-compare}
\vspace{-2mm}

\begin{figure}[t]
\vspace{-4mm}
\centering
\includegraphics[width=\linewidth]{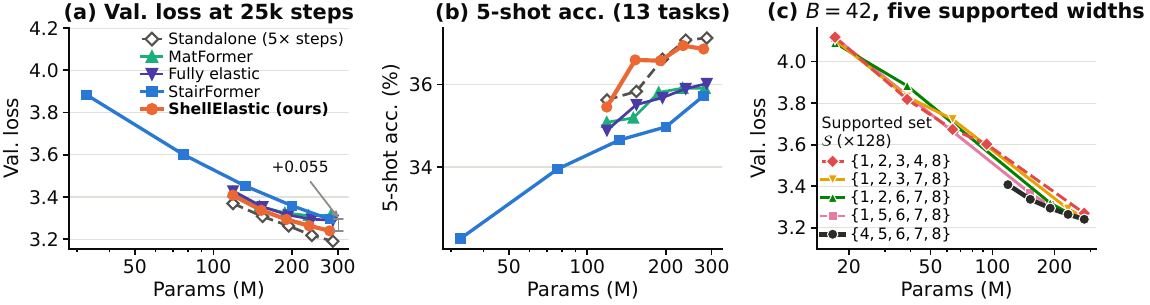}
\caption{Architecture comparison and search against parameters at step $25{,}000$. (a,\,b)~Comparison with StairFormer: validation loss and average 5-shot accuracy (\%) over $13$ tasks. (c)~Architecture search: validation loss of each supported width. Full results are in Appendices~\ref{app:stairformer} and~\ref{app:shape-all}.}
\label{fig:stairformer}
\label{fig:arch-placement}
\vspace{-3mm}
\end{figure}

\paragraph{Comparison with StairFormer.}
StairFormer~\citep{trost2026fullynested} nests sub-networks as prefixes of one forward pass; it uses equal-width groups, cumulative prefix normalization and one shared readout.
Here it has $g=208$ and \mbox{$\mathcal S=\{1,2,3,4,5\}$} in our notation, whereas \Ours{} has a $512$-wide core and four $128$-wide shells (\mbox{$\mathcal S=\{4,5,6,7,8\}$}), group-local normalization and per-width readouts.
Like \Ours{}, it trains all sub-networks at every step; we give it the same \muE{} and HPs (Section~\ref{sec:setup}).
Its narrower sub-networks are much smaller ($32.6$M to $200.4$M); only its Max matches \Ours{} in size.
At the parameter count of each \Ours{} sub-network, its validation loss is $0.055$ to $0.083$ nats higher (Figure~\ref{fig:stairformer}a; curves in Figure~\ref{fig:stairformer-curves}), and even its Max, with $1.5\times$ the parameters, is $0.001$ nats above the High of \Ours{}.
At Max, \Ours{} also leads in average 5-shot accuracy, $36.9$ against $35.7$ (Figure~\ref{fig:stairformer}b and Table~\ref{tab:stairformer-tasks}).

\paragraph{Architecture search at equal parameters and width count.}
Since $\mathcal S$ fixes both the extractable widths and the parameter count (a hidden matrix keeps $B(\mathcal S)=\sum_k(s_k-s_{k-1})s_k$ of its $M^2$ blocks of size $g\times g$), we ask where extra width should go when $B$ and the number of supported widths are fixed.
At $B=42$ with five supported widths, each candidate set has four groups of width $g$ and one of width $4g$ and differs only in where this wide group sits (block masks in Figure~\ref{fig:arch-masks}); at $B=39$ with six, the wide group has width $3g$ (Appendix~\ref{app:shape-all}).
We train every supported width at every step and compare the sets at each width they share with equal parameter counts (Figure~\ref{fig:arch-placement}c).
The core is the best place for the wide group: the core placement has the lowest loss, or is within $0.006$ nats of it, at every shared width, whereas a shell placement makes the first width containing the wide group $0.025$ to $0.052$ nats worse.
This motivates the wide core and thin shells of \Ours{}.

\vspace{-2mm}
\subsection{Ablation studies}
\label{sec:ablation}
\vspace{-2mm}

\begin{table}[t]
\vspace{-4mm}
\caption{Ablations at the $119$M--$285$M scale: final validation loss and mean 5-shot accuracy over $13$ downstream tasks and the five tiers. Each block ends with the default, and bold marks the best value of each column within a block. Appendix~\ref{app:ablation-studies} gives all ablations.}
\label{tab:ablations-main}
\centering
\small
\setlength{\tabcolsep}{3pt}
\begin{tabular}{@{}lcccccc@{\hspace{8pt}}c@{}}
\toprule
 & \multicolumn{6}{c}{Val.\ loss $\downarrow$} & 5-shot acc.\ $\uparrow$ \\
\cmidrule(r{4pt}){2-7} \cmidrule(l{0pt}){8-8}
Setting & Low & Medium & High & Extra & Max & Avg. & Avg. \\
\midrule
\multicolumn{8}{@{}l}{\emph{Selected sub-networks per step $\mathcal A_\tau$}} \\
Sandwich: Low, Max and one other & \textbf{3.4019} & 3.3617 & 3.3183 & 3.2816 & \textbf{3.2395} & 3.3206 & 36.1 \\
Random: one of the five sub-networks & 3.4311 & 3.3786 & 3.3542 & 3.3410 & 3.3356 & 3.3681 & 34.8 \\
Max + random: Max and one other & 3.4944 & 3.4079 & 3.3447 & 3.2956 & 3.2397 & 3.3564 & 35.5 \\
All: all five sub-networks (default) & 3.4086 & \textbf{3.3375} & \textbf{3.2944} & \textbf{3.2641} & 3.2403 & \textbf{3.3090} & \textbf{36.5} \\
\midrule
\multicolumn{8}{@{}l}{\emph{In-place distillation from Max}} \\
Forward KL, weight $0.1$ & 3.4047 & 3.3355 & 3.2937 & 3.2641 & 3.2410 & 3.3078 & 36.4 \\
Forward KL, weight $0.5$ & 3.3920 & 3.3287 & 3.2931 & 3.2675 & 3.2482 & 3.3059 & \textbf{36.6} \\
Forward KL, weight $1$ & \textbf{3.3748} & \textbf{3.3184} & \textbf{3.2873} & 3.2653 & 3.2483 & \textbf{3.2988} & 36.1 \\
Reverse KL, weight $0.5$ & 3.3868 & 3.3261 & 3.2887 & 3.2632 & 3.2437 & 3.3017 & 35.9 \\
Symmetric KL, weight $0.5$ & 3.3875 & 3.3256 & 3.2901 & \textbf{3.2629} & 3.2435 & 3.3019 & 36.0 \\
None (default) & 3.4086 & 3.3375 & 3.2944 & 3.2641 & \textbf{3.2403} & 3.3090 & 36.5 \\
\midrule
\multicolumn{8}{@{}l}{\emph{Ensembles of the five sub-networks}} \\
Learned per-token width router & 3.7735 & 3.5611 & 3.4137 & 3.3055 & \textbf{3.2269} & 3.4562 & 35.1 \\
Distillation from loss-weighted ensemble & \textbf{3.3902} & \textbf{3.3275} & \textbf{3.2926} & 3.2673 & 3.2471 & \textbf{3.3049} & \textbf{37.0} \\
Distillation from soft oracle & 3.3982 & 3.3350 & 3.2980 & 3.2709 & 3.2501 & 3.3105 & 36.3 \\
None (default) & 3.4086 & 3.3375 & 3.2944 & \textbf{3.2641} & 3.2403 & 3.3090 & 36.5 \\
\midrule
\multicolumn{8}{@{}l}{\emph{Stage-residual gradient scaling}} \\
$\gamma=0$ & 3.4053 & 3.3740 & 3.3530 & 3.3348 & 3.3115 & 3.3557 & 36.1 \\
$\gamma=0.3$ & \textbf{3.3806} & \textbf{3.3317} & 3.2991 & 3.2762 & 3.2485 & \textbf{3.3072} & 36.6 \\
$\gamma=0.5$ & 3.3939 & 3.3360 & 3.2967 & 3.2700 & 3.2447 & 3.3083 & 36.0 \\
$\gamma=0.8$ & 3.4030 & 3.3365 & 3.2946 & 3.2658 & 3.2419 & 3.3084 & \textbf{37.0} \\
$\gamma=1$ (default) & 3.4086 & 3.3375 & \textbf{3.2944} & \textbf{3.2641} & \textbf{3.2403} & 3.3090 & 36.5 \\
\bottomrule
\end{tabular}
\vspace{-3mm}
\end{table}

All ablations use the $119$M--$285$M supernet of Section~\ref{sec:setup} and change one setting at a time; the default results are those of Table~\ref{tab:main-scales}.
Of the settings we ablate, sub-network selection, in-place distillation, ensembles of the sub-networks, stage-residual gradient scaling and the trunk gradient of the narrower widths can change the mean validation loss by $0.010$ to $0.20$ nats, whereas the loss weights, normalization, readouts, base width and head dimension, and optimization HPs change it by at most $0.007$ nats.
We discuss the first four below (Table~\ref{tab:ablations-main}) and the others in Appendix~\ref{app:ablation-studies}.

\paragraph{Selected sub-networks per step.}
Besides training all five sub-networks at every step (All, the default), we test three rules for $\mathcal A_\tau$ in Equation~\ref{eq:loss}: \emph{Sandwich} (Low and Max plus one other), \emph{Max + random} (Max plus one other) and \emph{Random} (one sub-network, as in MatFormer), drawing the random ones uniformly at random.
The loss of each sub-network follows how often it is trained: Sandwich and Max + random match All at Max but are up to $0.024$ and $0.086$ nats worse at those they select less often, and Random is worse at every sub-network, by $0.095$ nats at Max.
All gives the lowest mean loss, $0.012$ to $0.059$ nats below the others, and the highest mean 5-shot accuracy.

\paragraph{In-place distillation.}
Elastic recipes often distill the narrower sub-networks from the widest one~\citep{yu2019universally}.
We add to the losses of Low to Extra a KL divergence toward the output distribution of Max (forward KL with weights $0.1$ to $1$; reverse and symmetric KL with weight $0.5$).
This lowers the loss of Low by up to $0.034$ nats and the mean loss by up to $0.010$ nats, the largest gain of any ablation, but raises that of Max by up to $0.008$ nats, improves the mean 5-shot accuracy by at most $0.1$ points and makes a step $2.2$ to $2.6$ times as long, so we do not use it.

\begin{figure}[t]
\vspace{-4mm}
\centering
\includegraphics[width=\linewidth]{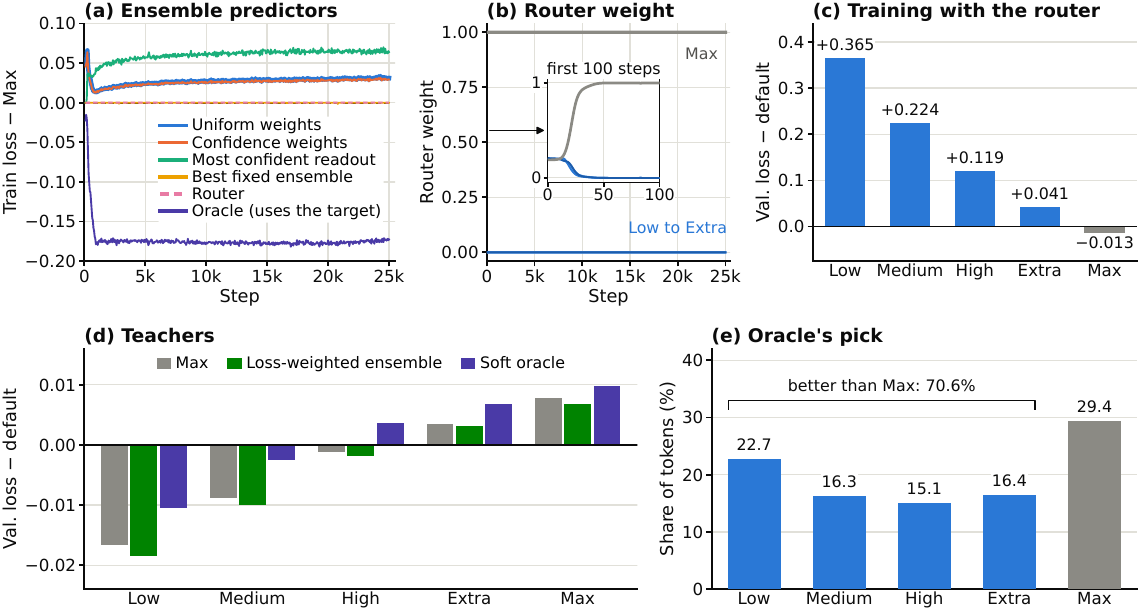}
\caption{Ensembles of the sub-networks. (a)~Train loss minus that of Max during training for the five predictors and the oracle, which uses the target. (b)~Router weight on each tier. (c)~Validation loss of each tier minus the default when training with the router. (d)~Validation loss of each tier minus the default when distilling toward Max, the loss-weighted ensemble or the soft oracle. (e)~Share of held-out tokens on which each tier assigns the target the highest probability. The default is the mean of the three \Ours{} runs of Table~\ref{tab:main-scales}.}
\label{fig:ensemble}
\vspace{-3mm}
\end{figure}

\paragraph{Ensembles of the sub-networks.}
One trunk pass yields the output distributions $p_s$ of all five sub-networks, which may err on different tokens, so a per-token ensemble $\sum_s\pi_sp_s$ could beat Max as a predictor at little extra cost, or replace it as a stronger teacher (details in Appendix~\ref{app:ablation-studies}).
\begin{itemize}[leftmargin=*,topsep=2pt,itemsep=1pt]
\item \emph{Ensemble as predictors.} We evaluate four ensembles of a trained supernet: uniform weights, confidence weights ($\pi_s\propto\max_vp_s(v)$), the most confident readout per token, and the best fixed ensemble, whose weights minimize the loss.
We also train a router, an MLP that maps the last-layer hidden state of Max to per-token weights $\pi_s$, together with its own supernet on the $\pi$-weighted sum of the five losses.
None of the ensemble methods beat Max alone (Figure~\ref{fig:ensemble}, Table~\ref{tab:ensemble-mixtures}): the best fixed ensemble only matches it, the other three are $0.03$ to $0.065$ nats worse in train loss, and the router puts all weight on Max within the first $50$ steps (Figure~\ref{fig:ensemble}b), so Max improves by $0.013$ nats while the other sub-networks, now barely trained, lose up to $0.365$ nats (Figure~\ref{fig:ensemble}c).
\item \emph{Oracle.} An oracle that picks per token the readout with the highest target probability is $0.19$ nats below Max (Table~\ref{tab:ensemble-oracle}), as a narrower sub-network is better on $71\%$ of tokens (Figure~\ref{fig:ensemble}e).
Two supernets trained separately with the same settings differ far more: averaging their Max lowers the loss by $0.056$ nats at twice the cost, and their oracle gains $0.23$ nats against $0.19$ for the five sub-networks on held-out text.
Our ensembles show that finding the better sub-network per token without the target is hard, but we believe a better ensemble could close part of this gap, a promising way to substantially improve \Ours{} at almost no extra cost.
\item \emph{Ensemble as teachers.} Every sub-network, including Max, distills (forward KL, weight $0.5$) toward the loss-weighted ensemble ($\pi_s\propto\omega_s$), which predicts $0.012$ nats worse than Max, or the soft oracle ($\pi_s\propto p_s(y)^\beta$), a smoothed oracle that predicts $0.14$ nats better (Table~\ref{tab:ensemble-mixtures}).
Yet neither teacher helps Max, not even the soft oracle: like in-place distillation with the same KL (Figure~\ref{fig:ensemble}d), they improve Low by up to $0.018$ nats and worsen Max by up to $0.010$ nats.
\end{itemize}

\paragraph{Stage-residual gradient scaling.}
Since the readout is linear (Equation~\ref{eq:head}), the logits of width $s_k$ are a sum of $k$ terms, one from the last-layer coordinates of each of its groups, so the supernet can be viewed as a sequence of stages: the core gives a base prediction and each shell adds a correction, as in residual learning~\citep{he2016deep}.
To make each shell learn this correction, we keep the forward pass but, in the backward pass, scale by $\gamma$ the gradient that each sub-network's loss sends to the trunk through the narrower sub-network it contains; the gradient through its own shell is unchanged (details in Appendix~\ref{app:ablation-studies}).
Thus $\gamma=1$ is the usual joint training of Equation~\ref{eq:loss}, while with $\gamma=0$ each loss reaches the trunk only through its own shell, which therefore learns only to correct the prediction of the narrower sub-network.
A $\gamma$ of $0.3$ to $0.8$ lowers the loss of Low by up to $0.028$ nats but raises those of Extra and Max by up to $0.012$ nats, leaving the mean loss within $0.002$ nats, whereas $\gamma=0$ raises the mean loss by $0.047$ nats, as the last-layer features of the narrower groups no longer adapt to the wider losses; we keep $\gamma=1$.
The stage view still points to a property that MatFormer and fully elastic models lack: since no group reads a later one and every group is normalized on its own, a new shell can be appended to a trained supernet and trained as a new stage with all trained weights frozen, so every trained sub-network keeps its exact outputs.

\vspace{-2mm}
\section{Conclusion}
\label{sec:conclusion}
\vspace{-2mm}
In \Ours{}, one full-width pass gives the exact hidden states of every sub-network, so joint training costs about as much as the widest one alone.
With \muE{} and proxy-tuned HPs, one run yields maximally updated sub-networks that beat MatFormer and fully elastic models at matched sizes and approach standalone models that together take five times the steps.
Our experiments favor training all sub-networks at every step and a wide core, and suggest two directions: better sub-network ensembles, whose oracle is $0.19$ nats below Max, and appending shells to a trained supernet.

\subsection*{AI use statement}
In this work, we used generative AI tools to provide critical ingredients for proving our mathematical claims, to assist in writing the proofs (Appendices~\ref{app:prefixproof} and~\ref{app:proofs}), and to implement parts of our method.
Additionally, we used generative AI tools to draft parts of the paper, edit it for readability, generate and check the code for the result figures and tables, and cross-check the statements of the paper against the training code.
We have reviewed all AI-assisted work. We take responsibility for the final content of this work, including text, claims, or artifacts produced with the aid of generative AI.

\subsection*{Ethics statement}
This work studies training methods for language models on publicly available web text and does not involve human subjects.
Elastic models reduce the cost of serving several model sizes, which lowers the energy cost of deployment; they inherit the biases of their training data like any language model.

\subsection*{Reproducibility statement}
Section~\ref{sec:setup} and Appendices~\ref{app:hparams} and~\ref{app:eval} list every architectural, optimization and evaluation setting; Table~\ref{tab:widthparam} and Appendix~\ref{app:other-param} give the complete parameterization rules; Appendices~\ref{app:prefixproof} and~\ref{app:proofs} contain the proofs, and Appendix~\ref{app:prefix-check} verifies Proposition~\ref{prop:prefix} numerically; Appendices~\ref{app:main-full}, \ref{app:transfer-extra}, \ref{app:arch-compare} and~\ref{app:ablation-studies} give the complete main, transfer, architecture-search and ablation results.
The training code, configurations and run logs used for all tables will be released.

\bibliography{references}

\begin{thebibliography}{45}
\providecommand{\natexlab}[1]{#1}
\providecommand{\url}[1]{\texttt{#1}}
\expandafter\ifx\csname urlstyle\endcsname\relax
  \providecommand{\doi}[1]{doi: #1}\else
  \providecommand{\doi}{doi: \begingroup \urlstyle{rm}\Url}\fi

\bibitem[Bisk et~al.(2020)Bisk, Zellers, Gao, Choi, et~al.]{bisk2020piqa}
Yonatan Bisk, Rowan Zellers, Jianfeng Gao, Yejin Choi, et~al.
\newblock Piqa: Reasoning about physical commonsense in natural language.
\newblock In \emph{Proceedings of the AAAI conference on artificial
  intelligence}, volume~34, pp.\  7432--7439, 2020.

\bibitem[Black et~al.(2022)Black, Biderman, Hallahan, Anthony, Gao, Golding,
  He, Leahy, McDonell, Phang, et~al.]{black2022gptneox}
Sidney Black, Stella Biderman, Eric Hallahan, Quentin Anthony, Leo Gao,
  Laurence Golding, Horace He, Connor Leahy, Kyle McDonell, Jason Phang, et~al.
\newblock Gpt-neox-20b: An open-source autoregressive language model.
\newblock In \emph{Proceedings of BigScience Episode\# 5--Workshop on
  Challenges \& Perspectives in Creating Large Language Models}, pp.\  95--136,
  2022.

\bibitem[Bordelon et~al.(2024)Bordelon, Noci, Li, Hanin, and
  Pehlevan]{bordelon2024depthwise}
Blake Bordelon, Lorenzo Noci, Mufan Li, Boris Hanin, and Cengiz Pehlevan.
\newblock Depthwise hyperparameter transfer in residual networks: Dynamics and
  scaling limit.
\newblock In \emph{International Conference on Learning Representations},
  volume 2024, pp.\  22088--22127, 2024.

\bibitem[Cai et~al.(2020)Cai, Gan, Wang, Zhang, and Han]{cai2020onceforall}
Han Cai, Chuang Gan, Tianzhe Wang, Zhekai Zhang, and Song Han.
\newblock Once-for-all: Train one network and specialize it for efficient
  deployment.
\newblock In \emph{International Conference on Learning Representations}, 2020.
\newblock URL \url{https://openreview.net/forum?id=HylxE1HKwS}.

\bibitem[Cai et~al.(2024)Cai, Muralidharan, Heinrich, Yin, Wang, Kautz, and
  Molchanov]{cai2024flextron}
Ruisi Cai, Saurav Muralidharan, Greg Heinrich, Hongxu Yin, Zhangyang Wang, Jan
  Kautz, and Pavlo Molchanov.
\newblock Flextron: Many-in-one flexible large language model.
\newblock In Ruslan Salakhutdinov, Zico Kolter, Katherine Heller, Adrian
  Weller, Nuria Oliver, Jonathan Scarlett, and Felix Berkenkamp (eds.),
  \emph{Proceedings of the 41st International Conference on Machine Learning},
  volume 235 of \emph{Proceedings of Machine Learning Research}, pp.\
  5298--5311. PMLR, 21--27 Jul 2024.
\newblock URL \url{https://proceedings.mlr.press/v235/cai24e.html}.

\bibitem[Clark et~al.(2019)Clark, Lee, Chang, Kwiatkowski, Collins, and
  Toutanova]{clark2019boolq}
Christopher Clark, Kenton Lee, Ming-Wei Chang, Tom Kwiatkowski, Michael
  Collins, and Kristina Toutanova.
\newblock Boolq: Exploring the surprising difficulty of natural yes/no
  questions.
\newblock In \emph{Proceedings of the 2019 Conference of the North American
  Chapter of the Association for Computational Linguistics: Human Language
  Technologies, Volume 1 (Long and Short Papers)}, pp.\  2924--2936, 2019.

\bibitem[Clark et~al.(2018)Clark, Cowhey, Etzioni, Khot, Sabharwal, Schoenick,
  and Tafjord]{clark2018arc}
Peter Clark, Isaac Cowhey, Oren Etzioni, Tushar Khot, Ashish Sabharwal, Carissa
  Schoenick, and Oyvind Tafjord.
\newblock Think you have solved question answering? try arc, the ai2 reasoning
  challenge.
\newblock \emph{arXiv preprint arXiv:1803.05457}, 2018.

\bibitem[Devlin et~al.(2019)Devlin, Chang, Lee, and Toutanova]{devlin2019bert}
Jacob Devlin, Ming-Wei Chang, Kenton Lee, and Kristina Toutanova.
\newblock Bert: Pre-training of deep bidirectional transformers for language
  understanding.
\newblock In \emph{Proceedings of the 2019 conference of the North American
  chapter of the association for computational linguistics: human language
  technologies, volume 1 (long and short papers)}, pp.\  4171--4186, 2019.

\bibitem[Devvrit et~al.(2024)Devvrit, Kudugunta, Kusupati, Dettmers, Chen,
  Dhillon, Tsvetkov, Hajishirzi, Kakade, Farhadi, et~al.]{devvrit2024matformer}
Fnu Devvrit, Sneha Kudugunta, Aditya Kusupati, Tim Dettmers, Kaifeng Chen,
  Inderjit~S Dhillon, Yulia Tsvetkov, Hannaneh Hajishirzi, Sham~M Kakade, Ali
  Farhadi, et~al.
\newblock Matformer: Nested transformer for elastic inference.
\newblock In \emph{The Thirty-eighth Annual Conference on Neural Information
  Processing Systems}, 2024.

\bibitem[Dey et~al.(2026)Dey, Zhang, Noci, Li, Bordelon, Bergsma, Pehlevan,
  Hanin, and Hestness]{dey2025completep}
Nolan Dey, Bin Zhang, Lorenzo Noci, Mufan Li, Blake Bordelon, Shane Bergsma,
  Cengiz Pehlevan, Boris Hanin, and Joel Hestness.
\newblock Don't be lazy: Completep enables compute-efficient deep transformers.
\newblock \emph{Advances in Neural Information Processing Systems},
  38:\penalty0 137707--137739, 2026.

\bibitem[Dodge et~al.(2021)Dodge, Sap, Marasovi{\'c}, Agnew, Ilharco,
  Groeneveld, Mitchell, and Gardner]{dodge2021c4}
Jesse Dodge, Maarten Sap, Ana Marasovi{\'c}, William Agnew, Gabriel Ilharco,
  Dirk Groeneveld, Margaret Mitchell, and Matt Gardner.
\newblock Documenting large webtext corpora: A case study on the colossal clean
  crawled corpus.
\newblock In \emph{Proceedings of the 2021 conference on empirical methods in
  natural language processing}, pp.\  1286--1305, 2021.

\bibitem[Everett et~al.(2024)Everett, Xiao, Wortsman, Alemi, Novak, Liu, Gur,
  Sohl-Dickstein, Kaelbling, Lee, and Pennington]{everett2024scaling}
Katie~E Everett, Lechao Xiao, Mitchell Wortsman, Alexander~A Alemi, Roman
  Novak, Peter~J Liu, Izzeddin Gur, Jascha Sohl-Dickstein, Leslie~Pack
  Kaelbling, Jaehoon Lee, and Jeffrey Pennington.
\newblock Scaling exponents across parameterizations and optimizers.
\newblock In Ruslan Salakhutdinov, Zico Kolter, Katherine Heller, Adrian
  Weller, Nuria Oliver, Jonathan Scarlett, and Felix Berkenkamp (eds.),
  \emph{Proceedings of the 41st International Conference on Machine Learning},
  volume 235 of \emph{Proceedings of Machine Learning Research}, pp.\
  12666--12700. PMLR, 21--27 Jul 2024.
\newblock URL \url{https://proceedings.mlr.press/v235/everett24a.html}.

\bibitem[{Gemma Team}(2025)]{gemmateam2025gemma3n}
{Gemma Team}.
\newblock Gemma 3n.
\newblock \url{https://ai.google.dev/gemma/docs/gemma-3n}, 2025.

\bibitem[Gordon et~al.(2012)Gordon, Kozareva, and Roemmele]{gordon2012copa}
Andrew Gordon, Zornitsa Kozareva, and Melissa Roemmele.
\newblock Semeval-2012 task 7: Choice of plausible alternatives: An evaluation
  of commonsense causal reasoning.
\newblock In \emph{* SEM 2012: The First Joint Conference on Lexical and
  Computational Semantics--Volume 1: Proceedings of the main conference and the
  shared task, and Volume 2: Proceedings of the Sixth International Workshop on
  Semantic Evaluation (SemEval 2012)}, pp.\  394--398, 2012.

\bibitem[He et~al.(2016)He, Zhang, Ren, and Sun]{he2016deep}
Kaiming He, Xiangyu Zhang, Shaoqing Ren, and Jian Sun.
\newblock Deep residual learning for image recognition.
\newblock In \emph{Proceedings of the IEEE conference on computer vision and
  pattern recognition}, pp.\  770--778, 2016.

\bibitem[Hendrycks et~al.(2020)Hendrycks, Burns, Basart, Zou, Mazeika, Song,
  and Steinhardt]{hendrycks2021mmlu}
Dan Hendrycks, Collin Burns, Steven Basart, Andy Zou, Mantas Mazeika, Dawn
  Song, and Jacob Steinhardt.
\newblock Measuring massive multitask language understanding.
\newblock \emph{arXiv preprint arXiv:2009.03300}, 2020.

\bibitem[Hinton et~al.(2015)Hinton, Vinyals, and Dean]{hinton2015distilling}
Geoffrey Hinton, Oriol Vinyals, and Jeff Dean.
\newblock Distilling the knowledge in a neural network.
\newblock \emph{arXiv preprint arXiv:1503.02531}, 2015.

\bibitem[Hoffmann et~al.(2022)Hoffmann, Borgeaud, Mensch, Buchatskaya, Cai,
  Rutherford, de~Las~Casas, Hendricks, Welbl, Clark, Hennigan, Noland,
  Millican, van~den Driessche, Damoc, Guy, Osindero, Simonyan, Elsen, Vinyals,
  Rae, and Sifre]{hoffmann2022chinchilla}
Jordan Hoffmann, Sebastian Borgeaud, Arthur Mensch, Elena Buchatskaya, Trevor
  Cai, Eliza Rutherford, Diego de~Las~Casas, Lisa~Anne Hendricks, Johannes
  Welbl, Aidan Clark, Tom Hennigan, Eric Noland, Katie Millican, George van~den
  Driessche, Bogdan Damoc, Aurelia Guy, Simon Osindero, Karen Simonyan, Erich
  Elsen, Oriol Vinyals, Jack~W. Rae, and Laurent Sifre.
\newblock Training compute-optimal large language models.
\newblock In \emph{Proceedings of the 36th International Conference on Neural
  Information Processing Systems}, NIPS '22, Red Hook, NY, USA, 2022. Curran
  Associates Inc.
\newblock ISBN 9781713871088.

\bibitem[Hou et~al.(2020)Hou, Huang, Shang, Jiang, Chen, and
  Liu]{hou2020dynabert}
Lu~Hou, Zhiqi Huang, Lifeng Shang, Xin Jiang, Xiao Chen, and Qun Liu.
\newblock Dynabert: Dynamic bert with adaptive width and depth.
\newblock \emph{Advances in Neural Information Processing Systems},
  33:\penalty0 9782--9793, 2020.

\bibitem[Jacot et~al.(2018)Jacot, Gabriel, and Hongler]{jacot2018ntk}
Arthur Jacot, Franck Gabriel, and Cl{\'e}ment Hongler.
\newblock Neural tangent kernel: Convergence and generalization in neural
  networks.
\newblock \emph{Advances in neural information processing systems}, 31, 2018.

\bibitem[Kingma \& Ba(2014)Kingma and Ba]{kingma2015adam}
Diederik~P Kingma and Jimmy Ba.
\newblock Adam: A method for stochastic optimization.
\newblock \emph{arXiv preprint arXiv:1412.6980}, 2014.

\bibitem[Kusupati et~al.(2022)Kusupati, Bhatt, Rege, Wallingford, Sinha,
  Ramanujan, Howard-Snyder, Chen, Kakade, Jain, et~al.]{kusupati2022mrl}
Aditya Kusupati, Gantavya Bhatt, Aniket Rege, Matthew Wallingford, Aditya
  Sinha, Vivek Ramanujan, William Howard-Snyder, Kaifeng Chen, Sham Kakade,
  Prateek Jain, et~al.
\newblock Matryoshka representation learning.
\newblock \emph{Advances in Neural Information Processing Systems},
  35:\penalty0 30233--30249, 2022.

\bibitem[Li et~al.(2024)Li, Fang, Smyrnis, Ivgi, Jordan, Gadre, Bansal, Guha,
  Keh, Arora, et~al.]{li2024dclm}
Jeffrey Li, Alex Fang, Georgios Smyrnis, Maor Ivgi, Matt Jordan, Samir Gadre,
  Hritik Bansal, Etash Guha, Sedrick Keh, Kushal Arora, et~al.
\newblock Datacomp-lm: In search of the next generation of training sets for
  language models.
\newblock \emph{Advances in Neural Information Processing Systems},
  37:\penalty0 14200--14282, 2024.

\bibitem[Loshchilov \& Hutter(2019)Loshchilov and Hutter]{loshchilov2019adamw}
Ilya Loshchilov and Frank Hutter.
\newblock Decoupled weight decay regularization.
\newblock In \emph{International Conference on Learning Representations}, 2019.
\newblock URL \url{https://openreview.net/forum?id=Bkg6RiCqY7}.

\bibitem[Mei et~al.(2018)Mei, Montanari, and Nguyen]{mei2018meanfield}
Song Mei, Andrea Montanari, and Phan-Minh Nguyen.
\newblock A mean field view of the landscape of two-layer neural networks.
\newblock \emph{Proceedings of the National Academy of Sciences}, 115\penalty0
  (33):\penalty0 E7665--E7671, 2018.

\bibitem[Mlodozeniec et~al.(2026)Mlodozeniec, Ablin, B{\'e}thune, Busbridge,
  Klein, Ramapuram, et~al.]{mlodozeniec2025completed}
Bruno Mlodozeniec, Pierre Ablin, Louis B{\'e}thune, Dan Busbridge, Michal
  Klein, Jason Ramapuram, et~al.
\newblock Completed hyperparameter transfer across modules, width, depth, batch
  and duration.
\newblock In \emph{International Conference on Learning Representations},
  volume 2026, pp.\  45300--45325, 2026.

\bibitem[Muennighoff et~al.(2025)Muennighoff, Soldaini, Groeneveld, Lo,
  Morrison, Min, Shi, Walsh, Tafjord, Lambert, et~al.]{muennighoff2025olmoe}
Niklas Muennighoff, Luca Soldaini, Dirk Groeneveld, Kyle Lo, Jacob Morrison,
  Sewon Min, Weijia Shi, Pete Walsh, Oyvind Tafjord, Nathan Lambert, et~al.
\newblock Olmoe: Open mixture-of-experts language models.
\newblock In \emph{International Conference on Learning Representations},
  volume 2025, pp.\  62061--62121, 2025.

\bibitem[Paperno et~al.(2016)Paperno, Kruszewski, Lazaridou, Pham, Bernardi,
  Pezzelle, Baroni, Boleda, and Fern{\'a}ndez]{paperno2016lambada}
Denis Paperno, Germ{\'a}n Kruszewski, Angeliki Lazaridou, Ngoc-Quan Pham,
  Raffaella Bernardi, Sandro Pezzelle, Marco Baroni, Gemma Boleda, and Raquel
  Fern{\'a}ndez.
\newblock The lambada dataset: Word prediction requiring a broad discourse
  context.
\newblock In \emph{Proceedings of the 54th annual meeting of the association
  for computational linguistics (volume 1: Long papers)}, pp.\  1525--1534,
  2016.

\bibitem[Patel et~al.(2021)Patel, Bhattamishra, and Goyal]{patel2021svamp}
Arkil Patel, Satwik Bhattamishra, and Navin Goyal.
\newblock Are nlp models really able to solve simple math word problems?
\newblock In \emph{Proceedings of the 2021 conference of the North American
  chapter of the association for computational linguistics: human language
  technologies}, pp.\  2080--2094, 2021.

\bibitem[Raffel et~al.(2020)Raffel, Shazeer, Roberts, Lee, Narang, Matena,
  Zhou, Li, and Liu]{raffel2020t5}
Colin Raffel, Noam Shazeer, Adam Roberts, Katherine Lee, Sharan Narang, Michael
  Matena, Yanqi Zhou, Wei Li, and Peter~J Liu.
\newblock Exploring the limits of transfer learning with a unified text-to-text
  transformer.
\newblock \emph{Journal of machine learning research}, 21\penalty0
  (140):\penalty0 1--67, 2020.

\bibitem[Sakaguchi et~al.(2021)Sakaguchi, Bras, Bhagavatula, and
  Choi]{sakaguchi2021winogrande}
Keisuke Sakaguchi, Ronan~Le Bras, Chandra Bhagavatula, and Yejin Choi.
\newblock Winogrande: An adversarial winograd schema challenge at scale.
\newblock \emph{Communications of the ACM}, 64\penalty0 (9):\penalty0 99--106,
  2021.

\bibitem[Shazeer(2020)]{shazeer2020glu}
Noam Shazeer.
\newblock Glu variants improve transformer.
\newblock \emph{arXiv preprint arXiv:2002.05202}, 2020.

\bibitem[Sohl-Dickstein et~al.(2020)Sohl-Dickstein, Novak, Schoenholz, and
  Lee]{sohldickstein2020standard}
Jascha Sohl-Dickstein, Roman Novak, Samuel~S Schoenholz, and Jaehoon Lee.
\newblock On the infinite width limit of neural networks with a standard
  parameterization.
\newblock \emph{arXiv preprint arXiv:2001.07301}, 2020.

\bibitem[Su et~al.(2024)Su, Ahmed, Lu, Pan, Bo, and Liu]{su2024rope}
Jianlin Su, Murtadha Ahmed, Yu~Lu, Shengfeng Pan, Wen Bo, and Yunfeng Liu.
\newblock Roformer: Enhanced transformer with rotary position embedding.
\newblock \emph{Neurocomputing}, 568:\penalty0 127063, 2024.

\bibitem[Taghibakhshi et~al.(2025)Taghibakhshi, Sreenivas, Muralidharan, Cai,
  Chochowski, Mahabaleshwarkar, Suhara, Olabiyi, Korzekwa, Patwary,
  et~al.]{taghibakhshi2025nemotronelastic}
Ali Taghibakhshi, Sharath~Turuvekere Sreenivas, Saurav Muralidharan, Ruisi Cai,
  Marcin Chochowski, Ameya~Sunil Mahabaleshwarkar, Yoshi Suhara, Oluwatobi
  Olabiyi, Daniel Korzekwa, Mostofa Patwary, et~al.
\newblock Nemotron elastic: Towards efficient many-in-one reasoning llms.
\newblock \emph{arXiv preprint arXiv:2511.16664}, 2025.

\bibitem[Trost et~al.(2026)Trost, Yun, Cooper, Orlanski, and
  Sala]{trost2026fullynested}
Avi Trost, Alexander Yun, John Cooper, Gabriel Orlanski, and Frederic Sala.
\newblock Fully nested transformers.
\newblock In \emph{AdaptFM: Resource-Adaptive Foundation Model Inference},
  2026.

\bibitem[Vaswani et~al.(2017)Vaswani, Shazeer, Parmar, Uszkoreit, Jones, Gomez,
  Kaiser, and Polosukhin]{vaswani2017attention}
Ashish Vaswani, Noam Shazeer, Niki Parmar, Jakob Uszkoreit, Llion Jones,
  Aidan~N Gomez, {\L}ukasz Kaiser, and Illia Polosukhin.
\newblock Attention is all you need.
\newblock \emph{Advances in neural information processing systems}, 30, 2017.

\bibitem[Yang et~al.(2021)Yang, Hu, Babuschkin, Sidor, Liu, Farhi, Ryder,
  Pachocki, Chen, and Gao]{yang2022tensorv}
Ge~Yang, Edward Hu, Igor Babuschkin, Szymon Sidor, Xiaodong Liu, David Farhi,
  Nick Ryder, Jakub Pachocki, Weizhu Chen, and Jianfeng Gao.
\newblock Tuning large neural networks via zero-shot hyperparameter transfer.
\newblock \emph{Advances in Neural Information Processing Systems},
  34:\penalty0 17084--17097, 2021.

\bibitem[Yang \& Hu(2021)Yang and Hu]{yang2021featurelearning}
Greg Yang and Edward~J Hu.
\newblock Tensor programs iv: Feature learning in infinite-width neural
  networks.
\newblock In \emph{International Conference on Machine Learning}, pp.\
  11727--11737. PMLR, 2021.

\bibitem[Yang et~al.(2024)Yang, Yu, Zhu, and Hayou]{yang2024depthmup}
Greg Yang, Dingli Yu, Chen Zhu, and Soufiane Hayou.
\newblock Tensor programs vi: Feature learning in infinite depth neural
  networks.
\newblock In \emph{International Conference on Learning Representations},
  volume 2024, pp.\  55099--55150, 2024.

\bibitem[Yu \& Huang(2019)Yu and Huang]{yu2019universally}
Jiahui Yu and Thomas~S Huang.
\newblock Universally slimmable networks and improved training techniques.
\newblock In \emph{Proceedings of the IEEE/CVF international conference on
  computer vision}, pp.\  1803--1811, 2019.

\bibitem[Yu et~al.(2018)Yu, Yang, Xu, Yang, and Huang]{yu2019slimmable}
Jiahui Yu, Linjie Yang, Ning Xu, Jianchao Yang, and Thomas Huang.
\newblock Slimmable neural networks.
\newblock \emph{arXiv preprint arXiv:1812.08928}, 2018.

\bibitem[Zellers et~al.(2019)Zellers, Holtzman, Bisk, Farhadi, and
  Choi]{zellers2019hellaswag}
Rowan Zellers, Ari Holtzman, Yonatan Bisk, Ali Farhadi, and Yejin Choi.
\newblock Hellaswag: Can a machine really finish your sentence?
\newblock In \emph{Proceedings of the 57th annual meeting of the association
  for computational linguistics}, pp.\  4791--4800, 2019.

\bibitem[Zhang \& Sennrich(2019)Zhang and Sennrich]{zhang2019rmsnorm}
Biao Zhang and Rico Sennrich.
\newblock Root mean square layer normalization.
\newblock \emph{Advances in neural information processing systems}, 32, 2019.

\bibitem[Zhong et~al.(2024)Zhong, Cui, Guo, Liang, Lu, Wang, Saied, Chen, and
  Duan]{zhong2024agieval}
Wanjun Zhong, Ruixiang Cui, Yiduo Guo, Yaobo Liang, Shuai Lu, Yanlin Wang, Amin
  Saied, Weizhu Chen, and Nan Duan.
\newblock Agieval: A human-centric benchmark for evaluating foundation models.
\newblock In \emph{Findings of the association for computational linguistics:
  NAACL 2024}, pp.\  2299--2314, 2024.

\end{thebibliography}
\bibliographystyle{iclr2027_conference}

\appendix
\clearpage
\etocdepthtag.toc{mtappendix}
\etocsettagdepth{mtchapter}{none}
\etocsettagdepth{mtappendix}{subsection}
\renewcommand{\contentsname}{Appendix Contents}
{\hypersetup{hidelinks}\tableofcontents}
\clearpage

\section{Architecture details}
\label{app:architecture}

Figure~\ref{fig:full-shellelastic} shows the complete \Ours{} model, Figure~\ref{fig:block} details its attention and FFN sublayers, and Figure~\ref{fig:architecture-comparison} places the three elastic families side by side.
Weight matrices are drawn with output rows and input columns and colored only where a connection is allowed; white lines mark the supported widths, and black lines delimit the slices of a dense matrix that the narrower sub-networks use.
All three families are drawn with $\mathcal S=\{4,5,6,7,8\}$, the default supported set of \Ours{}, so that their matrices are directly comparable; Table~\ref{tab:models} lists the dimensions of the trained models.
The forward multipliers of Table~\ref{tab:widthparam} are omitted throughout, and the per-head query/key normalization and rotary position transformations appear only in Figure~\ref{fig:block}.

\paragraph{Comparison with other elastic architectures.}
MatFormer keeps the model width, attention, normalization, embedding and readout at full width and nests only the FFN hidden width.
A fully elastic model slices dense matrices along both axes and normalizes over the active width.
In both cases, channel mixing is dense, so the hidden states of a narrow sub-network are not contained in those of a wide one, and each width needs its own trunk pass.
\Ours{} replaces the dense matrices with block-lower-triangular ones and the full-width normalization with group-local normalization, which is exactly what lets the supported prefixes share a single trunk pass (Proposition~\ref{prop:prefix}).

\begin{figure}[t]
\centering
\includegraphics[width=\linewidth]{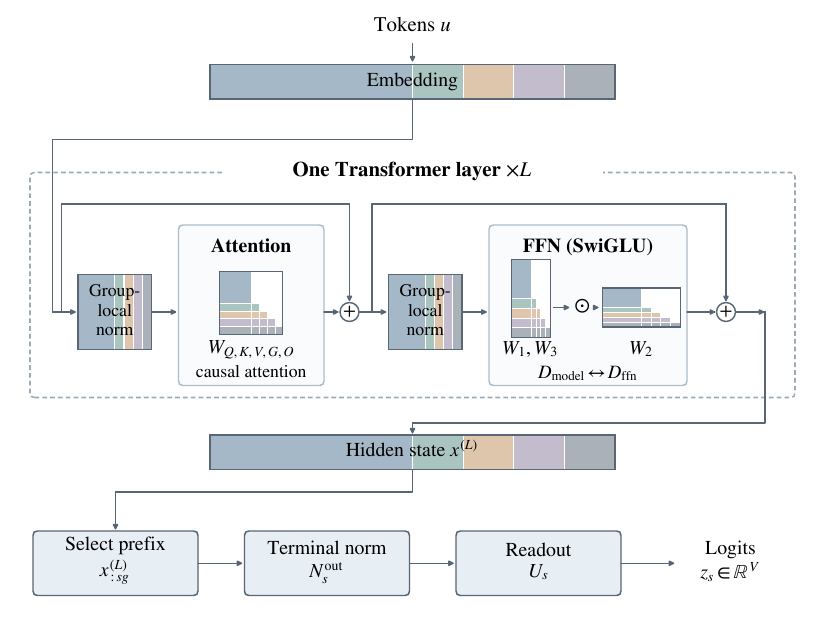}
\caption{\textbf{The complete \Ours{} model.} Embedding, one attention--FFN layer repeated $L$ times, and the terminal norm and readout of the selected width. Group-local branch normalization and structured projections preserve every supported prefix through the trunk.}
\label{fig:full-shellelastic}
\end{figure}

\begin{figure}[t]
\centering
\includegraphics[width=\linewidth]{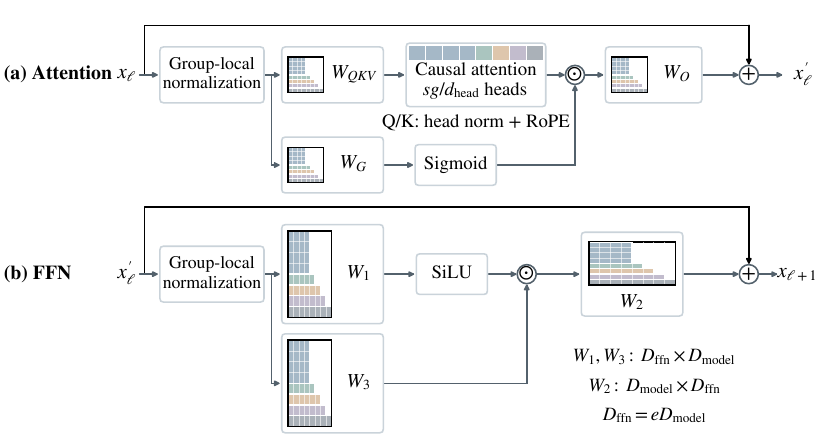}
\caption{\textbf{Prefix-preserving attention and FFN.} Sub-network $s$ uses the first $sg/d_{\mathrm{head}}$ attention heads. The FFN's up, gate, and down projections expand and contract the same groups.}
\label{fig:block}
\end{figure}

\begin{figure}[p]
\centering
\includegraphics[width=0.95\linewidth]{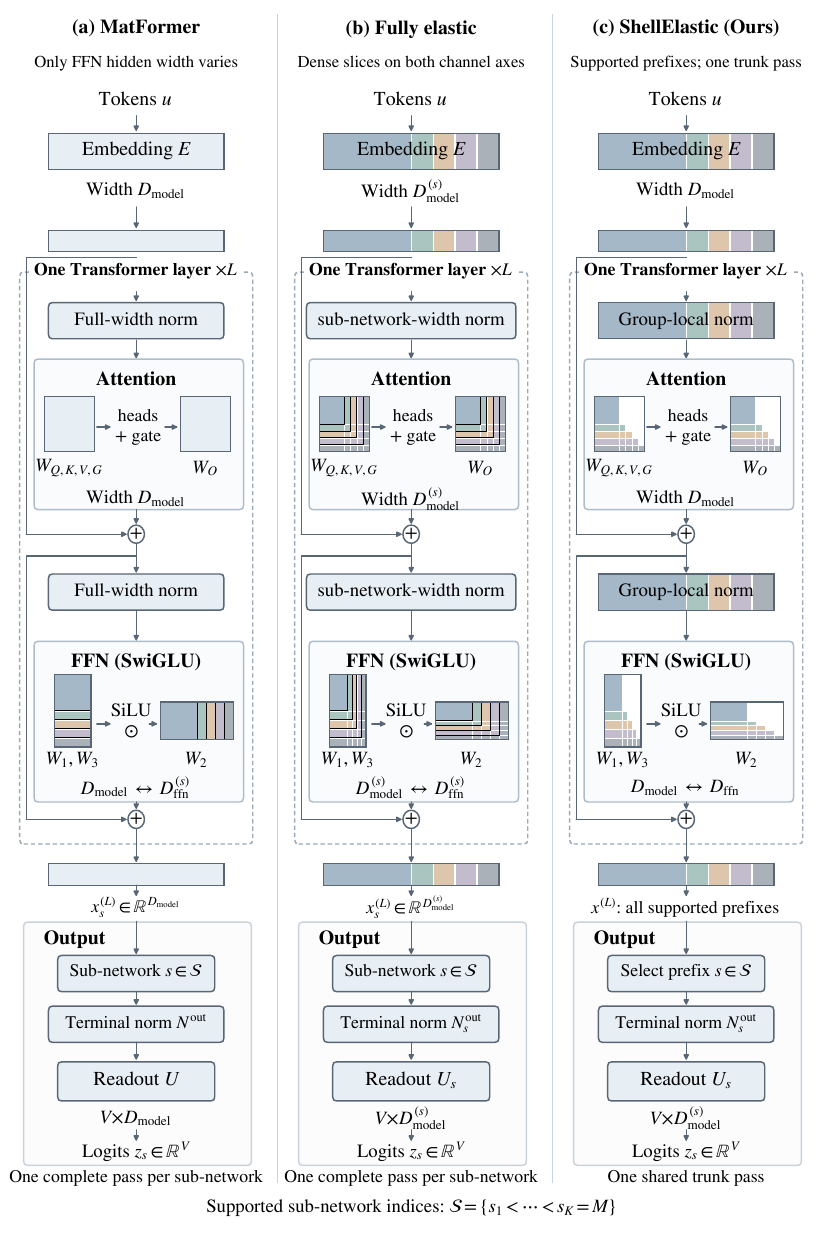}
\caption{\textbf{Complete architecture comparison.} MatFormer varies the FFN hidden dimension; fully elastic models vary the model dimension and the FFN hidden dimension together by slicing dense matrices; \Ours{} also varies both, but through block-lower-triangular projections in which each output group reads only a prefix of the input. Each layer frame repeats $L$ times. Proportions are illustrative.}
\label{fig:architecture-comparison}
\end{figure}

\FloatBarrier
\section{Exact computation reuse}
\label{app:prefixproof}

This appendix proves Proposition~\ref{prop:prefix}.
Appendix~\ref{app:why-diag} defines prefix-preserving maps and shows that the structured projections are prefix-preserving at the supported widths, Appendix~\ref{app:prefix-main} composes the operations of the trunk, and Appendix~\ref{app:prefix-check} checks the result numerically and delimits what it does and does not claim.

\subsection{Prefix-preserving maps}
\label{app:why-diag}

Throughout this appendix, $s\le t$ are supported widths in $\mathcal S$, and $P_s$ keeps the first $sg$ coordinates of every token's hidden state (the first $esg$ coordinates on the FFN hidden axis), as in Proposition~\ref{prop:prefix}.
For a map $F$ that the trunk applies at every supported width, $F_s$ denotes its restriction to width $sg$, that is, the map that sub-network $s$ applies.

\begin{definition}[Prefix-preserving map]
\label{def:prefix}
A map $F$ is \emph{prefix-preserving} if
\begin{equation}
 P_s\circ F_t=F_s\circ P_s\qquad\text{for all } s,t\in\mathcal S,\ s\le t.
\label{eq:prefixpreserving}
\end{equation}
\end{definition}

We first show that the structured projections of Section~\ref{sec:arch} have this property.
The lemma allows different base granularities on the two axes, since the output groups of the FFN up and gate projections are $e$ times as wide as their input groups, and those of the down projection $e$ times as narrow.

\begin{lemma}[Structured linear maps]
\label{lem:linear}
Consider a linear map with base input and output granularities $g_{\rm in}$ and $g_{\rm out}$, in which output group $k$ has $(s_k-s_{k-1})g_{\rm out}$ rows and reads the first $s_kg_{\rm in}$ input coordinates.
Let $P^{\rm in}_s$ and $P^{\rm out}_s$ keep the first $sg_{\rm in}$ inputs and the first $sg_{\rm out}$ outputs, and let $W_s$ be the top-left $sg_{\rm out}\times sg_{\rm in}$ block of the weight matrix. Then, with the row-vector convention,
\begin{equation}
 P^{\rm out}_s\!\left(xW_t^\top\right)=P^{\rm in}_s(x)\,W_s^\top,\qquad s,t\in\mathcal S,\quad s\le t.
\label{eq:linearcommute}
\end{equation}
\end{lemma}

\begin{proof}
Let $s=s_j$. The first $sg_{\rm out}$ outputs are those of the groups $k\le j$, and output coordinate $i$ of group $k$ is
\begin{equation}
 \big(xW_t^\top\big)_i=\sum_{c=1}^{s_kg_{\rm in}}(W_t)_{ic}\,x_c .
\label{eq:grouprow}
\end{equation}
Since $k\le j$ gives $s_k\le s_j=s$, this sum reads only the first $sg_{\rm in}$ inputs, and all of its weights lie in the top-left block $W_s$. Hence it equals $\big(P^{\rm in}_s(x)\,W_s^\top\big)_i$.
\end{proof}

\paragraph{Unsupported widths.}
At an index $u\notin\mathcal S$, there is a group with $s_{k-1}<u<s_k$.
A cut at $u$ keeps part of that group but removes inputs it is allowed to read (Equation~\ref{eq:grouprow}), so Equation~\ref{eq:linearcommute} does not hold in general, and the architecture guarantees no reusable sub-network there.
For $\mathcal S=\{4,5,6,7,8\}$, the indices $1,2,3$ lie inside the dense core and are not supported.

\subsection{Proof of Proposition~\ref{prop:prefix}}
\label{app:prefix-main}

\begin{proof}
We show that every operation of a layer is prefix-preserving (Steps~1--4) and then compose the operations over the $L$ layers (Step~5).

\textbf{Step 1: projections.}
The query, key, value, gate and output projections of attention and the up, gate and down projections of the FFN are structured linear maps. Lemma~\ref{lem:linear} therefore applies to each of them, with $g_{\rm in}=g_{\rm out}=g$ for the attention projections, $(g_{\rm in},g_{\rm out})=(g,eg)$ for the FFN up and gate projections, and $(eg,g)$ for the FFN down projection.

\textbf{Step 2: group-local normalization.}
By Equation~\ref{eq:norm}, $N(x)_{\mathcal G_k}=N_k(x_{\mathcal G_k})$ uses only the coordinates and gains of $\mathcal G_k$, so the first $sg$ outputs depend only on the first $sg$ inputs, and $P_s\circ N_t=N_s\circ P_s$.
The same holds for any finer partition and for a cumulative prefix norm, whose statistics for a coordinate use all coordinates up to the end of its group or of its block of $g$ coordinates (Equation~\ref{eq:normvariants}); for any of the first $sg$ coordinates, that end is at most $sg$.
A norm over the full active width does not preserve prefixes, because adding a group changes the statistics of the earlier coordinates.

\textbf{Step 3: attention.}
Since $g$ is a multiple of $d_{\mathrm{head}}$, every head lies inside one group, and sub-network $s$ uses the first $sg/d_{\mathrm{head}}$ heads.
By Step~1, their queries, keys and values depend only on the first $sg$ inputs.
Query--key normalization, the rotary transformation and the softmax act within a head, so the heads that sub-network $t$ adds leave the outputs of these heads unchanged.
The sigmoid gate and the output projection of Equation~\ref{eq:attention} are covered by Steps~1 and~4.

\textbf{Step 4: elementwise operations, residuals and embedding.}
The SwiGLU product and the other elementwise operations act coordinatewise, so they commute with $P_s$.
A residual branch $x\mapsto x+F(x)$ with a prefix-preserving $F$ is prefix-preserving, since $P_s\big(x+F_t(x)\big)=P_s(x)+F_s\big(P_s(x)\big)$.
The embedding satisfies $P_s(E_t[u])=E_s[u]$ by construction.

\textbf{Step 5: composition over layers.}
If $F$ and $G$ are prefix-preserving, then so is $G\circ F$:
\begin{equation}
 P_s\circ G_t\circ F_t=G_s\circ P_s\circ F_t=G_s\circ F_s\circ P_s .
\label{eq:compose}
\end{equation}
Each layer is a composition of the operations of Steps~1--4 and is therefore prefix-preserving, and by induction over the $L$ layers, as in Proposition~2.3 of StairFormer~\citep{trost2026fullynested}, so is the trunk: $P_s\circ T_t=T_s\circ P_s$. Applying this identity to the embedded sequence and using Step~4,
\begin{equation}
 P_s\big(T_t(E_t[u])\big)=T_s\big(P_s(E_t[u])\big)=T_s\big(E_s[u]\big),
\label{eq:prefixfinal}
\end{equation}
which is the claim.
\end{proof}

\paragraph{Consequence for the joint objective.}
The terminal norm $N^{\rm out}_s$ and the readout $U_s$ need not commute with $P_s$, since $N^{\rm out}_s$ normalizes over all $sg$ selected coordinates.
They are applied \emph{after} the prefix has been selected (Equation~\ref{eq:head}), so each sub-network's terminal map receives exactly the input it would receive in a separate evaluation.
Hence each $\ell_s$, and the gradient of the joint objective of Equation~\ref{eq:loss}, is the same function of the shared parameters as in separate evaluations, up to floating-point summation order.

\subsection{Numerical verification and scope}
\label{app:prefix-check}

\paragraph{Numerical verification.}
We instantiated the full trained configuration on the meta device to check module selection and shapes, and ran a reduced two-layer model on CPU with the same architecture switches, comparing one shared trunk pass against five separate evaluations of the same nested weights.
The logits and the weighted cross-entropy agreed exactly; the largest absolute difference between the two sets of parameter gradients was $1.68\times10^{-8}$, and every trainable parameter received a gradient.
The check replaces the fused SwiGLU and cross-entropy kernels of the GPU path with reference expressions, so it supports the implementation of Proposition~\ref{prop:prefix}, not its throughput.

\paragraph{Scope of the statement.}
The proposition assumes the dense-FFN graph of Section~\ref{sec:arch}, a fixed supported set $\mathcal S$ and prefix-preserving normalization. Beyond these assumptions:
\begin{itemize}
\item dropout needs matched masks for an equality test;
\item data-dependent routing or width-specific trunk modules need a separate argument;
\item Equation~\ref{eq:grouplinear} fixes the logical dependencies and not a storage layout or a kernel: every implementation of that map has the prefix property.
\end{itemize}

\section{Proofs of the \texorpdfstring{\muE{}}{mu-Elastic} results}
\label{app:proofs}

This appendix proves Theorem~\ref{thm:shell}, Proposition~\ref{prop:noniid} and Corollary~\ref{cor:transfer}.
Appendix~\ref{app:assumptions} states the assumptions and fixes the notation, Appendices~\ref{app:proof-thm}--\ref{app:proof-cor} give the three proofs, and Appendix~\ref{app:proof-other} extends Theorem~\ref{thm:shell} and Corollary~\ref{cor:transfer} to the other elastic architectures of Section~\ref{sec:param}.

\subsection{Assumptions and notation}
\label{app:assumptions}

\begin{assumption}[Heuristic scaling]
\label{ass:mup}
We use the heuristic framework of Appendix~J of the $\mu$Transfer paper~\citep{yang2022tensorv}:
\begin{enumerate}[label=(A\arabic*),leftmargin=2.8em,itemsep=2pt,topsep=3pt]
\item \emph{Inputs.} The input to every hidden matrix has $\Ord{1}$ coordinates with bounded second moments and is independent of that matrix at initialization.
\item \emph{Central limit.} If $w_1,\dots,w_n$ are i.i.d.\ with mean zero and standard deviation $b$ and independent of $x_1,\dots,x_n=\Ord{1}$, then $\sum_{c=1}^{n}w_cx_c=\Ord{\sqrt n\,b}$.
\item \emph{Law of large numbers.} After a gradient step, the update $\Delta W$ of a hidden matrix is aligned with the activations it multiplies, so $\sum_{c=1}^{n}\Delta W_{jc}\,x_c=\Ord{n\,\delta}$ when $|\Delta W_{jc}|=\Ord{\delta}$.
\item \emph{Adam.} With learning rate $\rho$ and a non-degenerate second-moment estimate, every coordinate of an Adam update has magnitude $\Ord{\rho}$, independently of the gradient scale.
\end{enumerate}
\end{assumption}

\paragraph{Notation.}
Consider a hidden matrix and the block row $W_k\in\R^{|\mathcal G_k|\times sg}$ of Equation~\ref{eq:grouplinear} that produces output group $\mathcal G_k$, whose supported boundary is $s=s_k$.
Inside a sub-network $t\in\mathcal S$ with $t\ge s$, let $x$ be the input of the matrix, and let
\begin{equation}
 y_j=\sum_{c=1}^{sg}(W_k)_{jc}\,x_c,\qquad j\in\mathcal G_k,
\label{eq:preact}
\end{equation}
be the pre-activation of coordinate $j$.
After one optimizer step, the update $\Delta W_k$ moves it by $\Delta y_j=\sum_{c=1}^{sg}(\Delta W_k)_{jc}\,x_c$.
Read widths are measured in base widths: the block row of $\mathcal G_k$ reads $s_k=s$ of them, and the readout of sub-network $s$ reads $s$.
The proxy width $d_0$ and the base values $(\sigma_0,\eta_0,\lambda_0)$ are fixed, and the constants hidden in $\Ord{\cdot}$ do not depend on $g$, $s$, $t$, $M$, $m$ or $e$.

\subsection{Proof of Theorem~\ref{thm:shell}}
\label{app:proof-thm}

\begin{proof}
\textbf{Step 1: the fan-in does not depend on the sub-network.}
The sum in Equation~\ref{eq:preact} has exactly $sg$ terms in every sub-network $t\ge s$: the block row reads no coordinate beyond $sg$ (Equation~\ref{eq:grouplinear}), and sub-network $t$ contains all of them because $tg\ge sg$.
One initialization scale and one learning rate, set by $q(s)=ms=sg/d_0$ (Equation~\ref{eq:widthratio}), therefore serve every sub-network that contains the group.

\textbf{Step 2: initialization.}
Table~\ref{tab:widthparam} draws the entries of $W_k$ i.i.d.\ with standard deviation $\sigma_0/\sqrt{ms}=\sigma_0\sqrt{d_0/(sg)}$. By (A1) and (A2),
\begin{equation}
 \E\big[y_j^2\big]=sg\cdot\frac{\sigma_0^2d_0}{sg}\cdot\Ord{1}=\Ord{\sigma_0^2d_0}.
\label{eq:proof-init}
\end{equation}

\textbf{Step 3: one AdamW step.}
The learning rate of $W_k$ is $\eta_0/(ms)=\eta_0d_0/(sg)$, so by (A4) every coordinate of its update satisfies $|(\Delta W_k)_{jc}|=\Ord{\eta_0d_0/(sg)}$, and by (A3)
\begin{equation}
 \Delta y_j=\sum_{c=1}^{sg}(\Delta W_k)_{jc}\,x_c=\Ord{sg\cdot\frac{\eta_0d_0}{sg}}=\Ord{\eta_0d_0}.
\label{eq:proof-update}
\end{equation}
The decoupled weight decay of AdamW, which is scaled by the learning rate, multiplies each coordinate by
\begin{equation}
 1-\frac{\eta_0}{ms}\cdot\lambda_0ms=1-\eta_0\lambda_0
\label{eq:proof-wd}
\end{equation}
per step, the same factor for every decayed parameter group and every width.

\textbf{Step 4: FFN down projection.}
In the FFN down projection, the block row of $\mathcal G_k$ reads $esg$ coordinates, and its output carries the forward multiplier $1/e$ (Table~\ref{tab:widthparam}).
With initialization $\sigma_0\sqrt e/\sqrt{ms}$ and learning rate $\eta_0/(ms)$, the computations of Steps~2 and~3 become
\begin{align}
 \E\big[y_j^2\big]&=e^{-2}\cdot esg\cdot\frac{e\,\sigma_0^2d_0}{sg}\cdot\Ord{1}=\Ord{\sigma_0^2d_0},\\
 \Delta y_j&=\Ord{e^{-1}\cdot esg\cdot\frac{\eta_0d_0}{sg}}=\Ord{\eta_0d_0},
\end{align}
so both scales are also free of $e$.

\textbf{Step 5: readouts and embedding.}
The readout $U_s$ reads the $sg$ coordinates of $N^{\rm out}_s(x^{(L)}_{:sg})$ (Equation~\ref{eq:head}), which is exactly the fan-in of the group with boundary $s$, and Table~\ref{tab:widthparam} gives it the same initialization and learning rate, following the per-layer scaling of~\citet{everett2024scaling}.
Steps~2 and~3 therefore give the same two bounds for its logits at initialization and for their change under the update of $U_s$, and the learned scalar $\alpha_s$ rescales them only by $\Ord{1}$.
The embedding is a lookup over a fixed vocabulary of size $V$, which does not change with width, and reading its first $sg$ columns does not change the scale of any coordinate.

\textbf{Step 6: joint training.}
Under the joint objective of Equation~\ref{eq:loss}, the gradient of $W_k$ sums the contributions of the selected sub-networks $t\ge s$.
By Proposition~\ref{prop:prefix}, all of them feed $W_k$ the same input $x$, so every contribution is an outer product with $x$, and the summed gradient, whatever its scale, keeps each row aligned with $x$; (A3) therefore applies, and by (A4) the Adam update has coordinates of size $\Ord{\eta_0/(ms)}$.
The group receives a gradient whenever some selected $t\ge s$ has $\omega_t>0$, which holds for $t=M$ when $\omega_M>0$, and the readout $U_s$ whenever $s\in\mathcal A_\tau$ and $\omega_s>0$.
Steps~2--5 therefore hold for the jointly trained supernet.

The constants $\Ord{\sigma_0^2d_0}$ and $\Ord{\eta_0d_0}$ of Equations~\ref{eq:proof-init} and~\ref{eq:proof-update} depend on neither $s$, $t$, $M$ nor $m$, which proves (i) and (ii) for every hidden group and every readout.
\end{proof}

\subsection{Proof of Proposition~\ref{prop:noniid}}
\label{app:proof-prop}

\begin{proof}
Consider a hidden matrix whose input satisfies (A1), such as a projection that reads a group-local norm, and let all of its groups share the initialization standard deviation $b$, the learning rate $\rho$ and the forward multiplier, which we absorb into $b$ and $\rho$.

\textbf{Fixed ratios.}
The group with boundary $s$ sums over $sg$ coordinates (Step~1 of Appendix~\ref{app:proof-thm}), so by (A2)--(A4)
\begin{equation}
 y_j=\Ord{\sqrt{sg}\,b},\qquad \Delta y_j=\Ord{sg\,\rho}.
\label{eq:shared}
\end{equation}
For boundaries $s<s'$, the pre-activations and updates of the two groups therefore differ by the factors $\Ord{\sqrt{s'/s}}$ and $\Ord{s'/s}$, which do not depend on $(b,\rho)$: changing the shared pair rescales all groups together and leaves these factors unchanged.

\textbf{Calibrations.}
Calibrating the shared pair for the group with boundary $r\in\mathcal S$ means taking the values of Table~\ref{tab:widthparam} for that group, $b=\sigma_0\sqrt{d_0/(rg)}$ and $\rho=\eta_0d_0/(rg)$. By Equation~\ref{eq:shared}, the group with boundary $s$ then has
\begin{equation}
 y_j=\Ord{\sigma_0\sqrt{d_0}\cdot\sqrt{s/r}},\qquad \Delta y_j=\Ord{\eta_0d_0\cdot s/r},
\label{eq:shared-ratio}
\end{equation}
so it deviates from the scales $\Ord{\sigma_0\sqrt{d_0}}$ and $\Ord{\eta_0d_0}$ of Theorem~\ref{thm:shell} by the factors $\sqrt{s/r}$ and $s/r$.
For $r=M$, the group with boundary $s$ has pre-activations of size $\Ord{\sqrt{s/M}}$ and updates of size $\Ord{\eta_0\,s/M}$, up to the constants $\sigma_0\sqrt{d_0}$ and $d_0$. For $r=s_1$, it is inflated by $\sqrt{s/s_1}$ and $s/s_1$, and the outermost group by $\sqrt{M/s_1}$ and $M/s_1$.

\textbf{No shared pair is uniformly maximal.}
Let the core and the outermost group deviate from the scales of Theorem~\ref{thm:shell} by the factors $a_1$ and $a_K$ at initialization. By the fixed ratios, $a_K/a_1=\Ord{\sqrt{M/s_1}}$, so $\max(a_1,a_1^{-1},a_K,a_K^{-1})$ grows at least as $(M/s_1)^{1/4}$; the same argument for the updates gives $(M/s_1)^{1/2}$.
Over supernets with unbounded $M/s_1$, no shared pair keeps both factors within constants independent of $M$, so none is maximal in the sense of Theorem~\ref{thm:shell} for every group.
When $M/s_1$ is bounded, the mismatch is a bounded factor: for $\mathcal S=\{4,5,6,7,8\}$, $M/s_1=2$, and any shared pair leaves the core and the outermost group $\sqrt2$ apart at initialization and $2$ apart in the update.
\end{proof}

\subsection{Proof of Corollary~\ref{cor:transfer}}
\label{app:proof-cor}

\begin{proof}
The proof of Theorem~\ref{thm:shell} holds for every base width $g=md_0$ and every supported set $\mathcal S\subseteq\{1,\dots,M\}$.
Its constants, $\Ord{\sigma_0^2d_0}$ in Equation~\ref{eq:proof-init}, $\Ord{\eta_0d_0}$ in Equation~\ref{eq:proof-update} and the decay factor $1-\eta_0\lambda_0$ in Equation~\ref{eq:proof-wd}, depend only on the proxy width $d_0$ and on the tuned values: $m$ and $M$ enter only through the active input widths, which the rule of Table~\ref{tab:widthparam} cancels.
The proxy is the special case $M=1$, $g=d_0$ (so $m=s=1$) of the same computation, so the values tuned on it set the same constants in every group and readout of every target.
This gives (i) for any \Ours{} supernet and (ii) under any increase of $m$ or $M$.
\end{proof}

\subsection{Extension to other elastic architectures}
\label{app:proof-other}

The proofs above use the block structure of \Ours{} only in Step~1 of Appendix~\ref{app:proof-thm}, to show that a group has the same fan-in and the same input in every sub-network that contains it.
For the other elastic architectures of Section~\ref{sec:param}, under their rules in Appendix~\ref{app:other-param}, Theorem~\ref{thm:shell} holds with sub-network $s$ in place of the group with boundary $s$: the hidden matrices of every sub-network satisfy (i), and (ii) at each step that trains that sub-network. The readouts are treated per architecture below.

\textbf{StairFormer.}
Its hidden matrices have the connectivity of \Ours{} with $\mathcal S=\{1,\dots,M\}$ and follow Table~\ref{tab:widthparam}, and all sub-networks are trained at every step, so the proof of Theorem~\ref{thm:shell} applies verbatim.
Its single readout, of which sub-network $s$ reads the first $sg$ columns, takes the split of Equation~\ref{eq:split}, which gives these columns the initialization $\sigma_0/\sqrt{ms}$ and the learning rate $\eta_0/(ms)$ of $U_s$ in Table~\ref{tab:widthparam}; since every sub-network feeds them the prefix $x^{(L)}_{:sg}$ up to its terminal norm, Steps~5 and~6 apply.

\textbf{Fully elastic models.}
Sub-network $s$ reads the top-left slice $W_s$ of every hidden matrix, whose fan-in $sg$ ($esg$ for the FFN down projection) varies with $s$.
By Equation~\ref{eq:split}, $W_s$ has the initialization standard deviation $\sigma_0/\sqrt{ms}$ ($\sigma_0\sqrt e/\sqrt{ms}$ with the output multiplier $1/e$ for the FFN down projection), and since $\Delta$ moves by $\Ord{\eta_0/m}$ per coordinate under (A4), $W_s$ moves by $\Ord{\eta_0/(ms)}$ per coordinate.
These are the values of Table~\ref{tab:widthparam} at fan-in $sg$, so Steps~2--4 give the pre-activations of sub-network $s$ the scales $\Ord{\sigma_0\sqrt{d_0}}$ at initialization and $\Ord{\eta_0d_0}$ after one step.
At a step that trains sub-network $s$, the update of $\Delta$ is computed from the input of sub-network $s$, so (A3) applies to the update of $W_s$; (ii) does not need the steps that train other sub-networks.
The shared readout takes the same split with the multiplier $\alpha_s$ (Appendix~\ref{app:other-param}), so the columns that sub-network $s$ reads have the initialization $\sigma_0/\sqrt{ms}$ and the learning rate $\eta_0/(ms)$ of $U_s$ in Table~\ref{tab:widthparam}, and Step~5 gives its logits the same two scales.

\textbf{MatFormer.}
Only the FFN down projection has a fan-in that varies with the sub-network, $sh_0=esD_{\mathrm{model}}=esmd_0$.
With the split and the output multiplier $1/e$, its slice has the effective initialization $\sigma_0/\sqrt{ems}$ and the effective learning rate $\eta_0/(ems)$, so Step~4 gives the same constants with $sD_{\mathrm{model}}$ in place of $sg$.
Every other hidden matrix and the shared readout have the fixed fan-in $D_{\mathrm{model}}=md_0$ and take the factor $m$, which is the rule of Table~\ref{tab:widthparam} with $s=1$ and gives the same constants by Steps~2, 3 and~5.
As for fully elastic models, (A3) applies at each step to the sub-network trained at that step.

\textbf{Transfer.}
In all three architectures, $m$ and $s$ enter only through fan-ins that the rules cancel, so the constants $\Ord{\sigma_0\sqrt{d_0}}$ and $\Ord{\eta_0d_0}$ and the decay factor $1-\eta_0\lambda_0$, which $\Delta$ obtains from its learning rate $\eta_0/m$ and weight decay $\lambda_0m$, depend only on $d_0$ and the base values.
The proof of Corollary~\ref{cor:transfer} therefore carries over to every $m$ and every supported set.

\section{\texorpdfstring{\muE{}}{mu-Elastic} for other elastic architectures}
\label{app:other-param}

All three elastic baselines use the proxy width $d_0=128$ and the base values $(\sigma_0,\eta_0,\lambda_0)$ of \Ours{}, and their $m$ and $s$ follow Section~\ref{sec:param}.
Tables~\ref{tab:param-full} and~\ref{tab:param-ffn} list the complete rules of the fully elastic model and MatFormer.

\paragraph{Fully elastic models.}
A fully elastic model stores every hidden matrix as one dense tensor at the full width, and sub-network $s$ reads its top-left slice, whose fan-in $sg$ ($esg$ for the FFN down projection) varies with $s$.
\muE{} gives this slice the factor $ms$: initialization $\sigma_0/\sqrt{ms}$, learning rate $\eta_0/(ms)$ and weight decay $\lambda_0ms$.
One stored tensor cannot follow this rule for every $s$, because the initialization and the update scale with different powers of $s$.
We therefore split each hidden tensor into a frozen initialization $W^{\mathrm{init}}$, drawn at $\sigma_0/\sqrt m$, and a trainable delta $\Delta$, initialized to zero and trained with learning rate $\eta_0/m$ and weight decay $\lambda_0m$. Sub-network $s$ uses
\begin{equation}
 W_s=s^{-1/2}\,W^{\mathrm{init}}[\text{slice}]+s^{-1}\,\Delta[\text{slice}].
\label{eq:split}
\end{equation}
At initialization, $W_s$ has standard deviation $s^{-1/2}\sigma_0/\sqrt m=\sigma_0/\sqrt{ms}$; under Adam, $\Delta$ moves by $\Ord{\eta_0/m}$ per coordinate, so $W_s$ moves by $\Ord{\eta_0/(ms)}$.
Both match Table~\ref{tab:widthparam} for every $s$, while weight decay, which acts on $\Delta$ only, pulls $W_s$ toward its initialization.
The main results use one shared sliced readout, of which sub-network $s$ reads the first $sg$ columns; it takes the same split with the multiplier $\alpha_s$ (Table~\ref{tab:param-full}), which gives these columns the initialization $\sigma_0/\sqrt{ms}$ and the learning rate $\eta_0/(ms)$ of the per-width readout $U_s$ of Table~\ref{tab:widthparam}.
Our baseline has $g=64$ ($m=1/2$), $\mathcal S=\{6,7,8,9,10\}$ and $e=4$; attention nests by heads at $d_{\mathrm{head}}=64$, the norms act over the active width, the embedding is sliced, and one sub-network is sampled uniformly at random per step.

\paragraph{MatFormer.}
MatFormer keeps $D_{\mathrm{model}}$ fixed, so $m=D_{\mathrm{model}}/d_0$, and nests only the FFN hidden width, $D^{(s)}_{\mathrm{ffn}}=s\,h_0$ with a base FFN width $h_0$; here $e=h_0/D_{\mathrm{model}}$ takes the place of the FFN expansion, so that the down projection of sub-network $s$ reads $esD_{\mathrm{model}}$ coordinates.
Only the FFN down projection has a fan-in that varies across sub-networks; it takes the factor $ms$ through the split of Equation~\ref{eq:split}, and the other hidden matrices and the readout, whose fan-in $D_{\mathrm{model}}$ is fixed, take the non-elastic \mup{} factor $m$, the attention and FFN up/gate matrices through the same split with multiplier $1$; the embedding and gains keep the base values (Table~\ref{tab:param-ffn}).
Our baseline has $D_{\mathrm{model}}=512$ ($m=4$), $h_0=128$ ($e=1/4$) and $\mathcal S=\{4,9,15,22,30\}$; attention, normalization, the embedding and a single readout are shared at full width, and one sub-network is sampled uniformly at random per step, as in the original MatFormer recipe.

\paragraph{StairFormer.}
With equal-width nested blocks, StairFormer has the connectivity of \Ours{} with $\mathcal S=\{1,\dots,M\}$, so its hidden matrices follow Table~\ref{tab:widthparam} without a split.
Its single readout, of which sub-network $s$ reads the first $sg$ columns, has a fan-in that varies with $s$ and takes the factor $ms$ through the split of Equation~\ref{eq:split}.
Our StairFormer baseline has five $208$-wide groups with $d_{\mathrm{head}}=104$, so that each group holds two complete heads.

\begin{table}[t]
\caption{\textbf{\muE{} for the fully elastic baseline.} $s$ is the sub-network width in units of $g$; multipliers act on the sliced tensor in the forward pass (Equation~\ref{eq:split}), and the FFN down projection multiplies its output, from both terms, by $1/e$.}
\label{tab:param-full}
\centering
\footnotesize
\setlength{\tabcolsep}{4pt}
\begin{tabular}{@{}lcccc@{}}
\toprule
Parameter & Initialization & Learning rate & Weight decay & Forward multiplier \\
\midrule
Token embedding (first $s\dbase$ columns) & $\sigma_0$ & $\eta_0$ & $\lambda_0$ & $1$ \\
$W^{\mathrm{init}}$ of attention and FFN up/gate & $\sigma_0/\sqrt{m}$ & frozen & -- & $s^{-1/2}$ \\
$W^{\mathrm{init}}$ of FFN down & $\sigma_0\sqrt{e}/\sqrt{m}$ & frozen & -- & $s^{-1/2}$ \\
$\Delta$ of every hidden matrix & $0$ & $\eta_0/m$ & $\lambda_0m$ & $s^{-1}$ \\
$W^{\mathrm{init}}$ of the readout (first $s\dbase$ columns) & $\sigma_0/\sqrt{m}$ & frozen & -- & $\alpha_ss^{-1/2}$ \\
$\Delta$ of the readout (first $s\dbase$ columns) & $0$ & $\eta_0/m$ & $\lambda_0m$ & $\alpha_ss^{-1}$ \\
Norm gains, readout scalars $\alpha_s$ & $1$ & $\eta_0$ & $0$ & -- \\
\bottomrule
\end{tabular}
\end{table}

\begin{table}[t]
\caption{\textbf{\muE{} for MatFormer.} $s=D^{(s)}_{\mathrm{ffn}}/h_0$ is the FFN hidden width in units of the base FFN width $h_0$; the hidden matrices other than the FFN down projection and the readout have a fixed fan-in and take the non-elastic \mup{} factor $m=D_{\mathrm{model}}/d_0$, and the embedding and gains keep the base values. The FFN down projection multiplies its output, from both terms, by $1/e$.}
\label{tab:param-ffn}
\centering
\footnotesize
\setlength{\tabcolsep}{4pt}
\begin{tabular}{@{}lcccc@{}}
\toprule
Parameter & Initialization & Learning rate & Weight decay & Forward multiplier \\
\midrule
Token embedding (full width) & $\sigma_0$ & $\eta_0$ & $\lambda_0$ & $1$ \\
$W^{\mathrm{init}}$ of attention and FFN up/gate & $\sigma_0/\sqrt{m}$ & frozen & -- & $1$ \\
$\Delta$ of attention and FFN up/gate & $0$ & $\eta_0/m$ & $\lambda_0m$ & $1$ \\
$W^{\mathrm{init}}$ of FFN down (first $D^{(s)}_{\mathrm{ffn}}$ columns) & $\sigma_0\sqrt{e}/\sqrt{m}$ & frozen & -- & $s^{-1/2}$ \\
$\Delta$ of FFN down & $0$ & $\eta_0/m$ & $\lambda_0m$ & $s^{-1}$ \\
Shared readout (fan-in $D_{\mathrm{model}}$) & $\sigma_0/\sqrt{m}$ & $\eta_0/m$ & $\lambda_0m$ & $\alpha$ (one scalar) \\
Norm gains, readout scalar $\alpha$ & $1$ & $\eta_0$ & $0$ & -- \\
\bottomrule
\end{tabular}
\end{table}

\section{Training details}
\label{app:hparams}
\newcolumntype{L}[1]{>{\raggedright\arraybackslash}p{#1}}

This appendix gives the model configurations of the three scales (Appendix~\ref{app:configs}), the optimization settings (Appendix~\ref{app:optim}), the parameter counts (Appendix~\ref{app:params}), and the training and inference cost (Appendix~\ref{app:cost}).

\subsection{Model configurations}
\label{app:configs}

Table~\ref{tab:models} lists the base width, supported set and size of every model, and Table~\ref{tab:arch} the remaining architecture settings; Figure~\ref{fig:block} shows one layer of \Ours{} and Figure~\ref{fig:full-shellelastic} the complete model.
Ablations vary only the setting under study relative to the defaults of the $119$M--$285$M scale.

\paragraph{$119$M--$285$M scale.}
The main results at this scale use the models of the first block of Table~\ref{tab:models}, trained on $6.55$B DCLM-Baseline tokens with the HPs of Table~\ref{tab:hparams}.
\Ours{} has $g=128$ with one head per base width ($d_{\mathrm{head}}=128$), $M=8$ and $\mathcal S=\{4,5,6,7,8\}$, so each hidden matrix keeps $42$ of its $64$ blocks of size $128\times128$, and the FFN expansion is $e=1$.
The two branch norms of every layer are group-local over $\mathcal G_1,\mathcal G_2,\mathcal G_3,\mathcal G_4,\mathcal G_5$ of widths $(512,128,128,128,128)$, queries and keys are normalized per head, and each supported width has its own terminal norm $N^{\rm out}_s$ over all $sg$ coordinates, readout $U_s$ and learned scalar $\alpha_s$; all five widths are selected at every step with loss weights $(1,1,1,1,4)$.
The proxy is the non-elastic special case $M=1$ at width $d_0=128$, and each standalone model is a separately trained model of width $384$ to $640$ with FFN expansion $4$ and $d_{\mathrm{head}}=64$; the sampled standalone baseline trains these five models in one run of the default length, updating one model sampled uniformly at random per step. Since each of these models takes one step in five on average, the sampled standalone baseline rescales their optimizer HPs to this five times shorter horizon with the token-budget rule of \citet{mlodozeniec2025completed}, consistent with the scaling rules of Section~\ref{sec:param}.

\paragraph{$375$M--$906$M scale.}
The main results at this scale use the models of the second block of Table~\ref{tab:models}.
\Ours{} has $g=256$ with two heads per base width ($d_{\mathrm{head}}=128$, $m=2$) and otherwise the settings above: $M=8$, $\mathcal S=\{4,5,6,7,8\}$, $e=1$, groups $\mathcal G_1,\mathcal G_2,\mathcal G_3,\mathcal G_4,\mathcal G_5$ of widths $(1024,256,256,256,256)$ and loss weights $(1,1,1,1,4)$.
MatFormer has $D_{\mathrm{model}}=1024$ ($m=8$), $h_0=128$ ($e=1/8$) and $\mathcal S=\{8,17,27,38,51\}$, and the fully elastic model has $g=64$ ($m=1/2$), $\mathcal S=\{13,15,17,19,21\}$ and $e=2.75$.
The standalone models have widths $832$ to $1344$ ($g=64$, $M=13,15,17,19,21$) with FFN expansion $2.75$ and $d_{\mathrm{head}}=64$, and the sampled standalone baseline trains these five models in one run.
Every run at this scale takes $50{,}000$ steps of $256$ sequences ($26.2$B tokens), decays the learning rate linearly from step $40{,}000$ to $0$ at step $50{,}000$ and reports the validation loss at step $50{,}000$; all other settings are those of the $119$M--$285$M scale.

\paragraph{$675$M--$1.69$B scale.}
The main results at this scale use the models of the third block of Table~\ref{tab:models}.
\Ours{} is the supernet of the $375$M--$906$M scale with FFN expansion $e=4$.
MatFormer has $D_{\mathrm{model}}=1024$ ($m=8$), $h_0=256$ ($e=1/4$) and $\mathcal S=\{16,24,34,45,57\}$, and the fully elastic model and the standalone models keep the widths of the $375$M--$906$M scale with FFN expansion $7.25$.
Every run at this scale takes $100{,}000$ steps of $256$ sequences ($52.4$B tokens), decays the learning rate linearly from step $85{,}000$ to $0$ at step $100{,}000$ and reports the validation loss at step $100{,}000$.
For the doubled horizon, the optimizer HPs are rescaled to $\eta_0=1.41\times10^{-3}$, $\lambda_0=7.07\times10^{-3}$, $\beta_1=\beta_2=0.975$ and $\epsilon=1.41\times10^{-14}$, and the sampled standalone baseline rescales its own values in the same way; each model is fully sharded over the $8$ GPUs, and all other settings are those of the $375$M--$906$M scale.

\begin{table}[t]
\caption{\textbf{Model configurations.} Each model is defined by its base width $g$ and supported set $\mathcal S$ (Section~\ref{sec:arch}): its sub-networks have widths $sg$ for $s\in\mathcal S$, and a non-elastic model has $\mathcal S=\{M\}$. MatFormer nests only the FFN hidden width, so its $g$ and $\mathcal S$ refer to that width at a fixed $D_{\mathrm{model}}$ ($512$, and $1024$ at the two larger scales). All models have $32$ layers and a $50{,}432$-token vocabulary; parameters count the non-zero weights of each sub-network (Appendix~\ref{app:params}).}
\label{tab:models}
\centering
\small
\setlength{\tabcolsep}{1.8pt}
\begin{tabular}{@{}lcccccc@{}}
\toprule
Model & $\dbase$ & $\mathcal S$ & Sub-network widths & $d_{\mathrm{head}}$ & $D_{\mathrm{ffn}}$ & Parameters (M) \\
\midrule
Proxy & $128$ & $\{1\}$ & $128$ & $128$ & $128$ & $17.1$ \\
\midrule
\multicolumn{7}{@{}l}{\textit{$119$M--$285$M scale}} \\
Standalone & $64$ & $\{M\}$, $M=6,7,8,9,10$ & $384$--$640$ & $64$ & $4D_{\mathrm{model}}$ & $119.0$--$287.5$ \\
MatFormer & $128$ & $\{4,9,15,22,30\}$ & FFN $512$--$3840$ & $128$ & $3840$ & $118.8$--$282.4$ \\
Fully elastic & $64$ & $\{6,7,8,9,10\}$ & $384$--$640$ & $64$ & $2560$ & $119.0$--$287.5$ \\
\Ours{} & $128$ & $\{4,5,6,7,8\}$ & $512$--$1024$ & $128$ & $1024$ & $118.8$--$279.6$ \\
\midrule
\multicolumn{7}{@{}l}{\textit{$375$M--$906$M scale}} \\
Standalone & $64$ & $\{M\}$, $M=13,15,17,19,21$ & $832$--$1344$ & $64$ & $2.75D_{\mathrm{model}}$ & $377.5$--$901.6$ \\
MatFormer & $128$ & $\{8,17,27,38,51\}$ & FFN $1024$--$6528$ & $128$ & $6528$ & $371.9$--$912.9$ \\
Fully elastic & $64$ & $\{13,15,17,19,21\}$ & $832$--$1344$ & $64$ & $3696$ & $377.5$--$901.6$ \\
\Ours{} & $256$ & $\{4,5,6,7,8\}$ & $1024$--$2048$ & $128$ & $2048$ & $371.9$--$911.5$ \\
\midrule
\multicolumn{7}{@{}l}{\textit{$675$M--$1.69$B scale}} \\
Standalone & $64$ & $\{M\}$, $M=13,15,17,19,21$ & $832$--$1344$ & $64$ & $7.25D_{\mathrm{model}}$ & $676.6$--$1682.0$ \\
MatFormer & $256$ & $\{16,24,34,45,57\}$ & FFN $4096$--$14592$ & $128$ & $14592$ & $673.8$--$1705.6$ \\
Fully elastic & $64$ & $\{13,15,17,19,21\}$ & $832$--$1344$ & $64$ & $9744$ & $676.6$--$1682.0$ \\
\Ours{} & $256$ & $\{4,5,6,7,8\}$ & $1024$--$2048$ & $128$ & $8192$ & $673.8$--$1704.2$ \\
\bottomrule
\end{tabular}
\end{table}

\begin{table}[t]
\caption{\textbf{\boldmath Architecture settings at the $119$M--$285$M scale.} The upper block is shared by all models, and the lower block lists the settings that differ between them (base widths, supported sets and sizes of all three scales in Table~\ref{tab:models}).}
\label{tab:arch}
\centering
\footnotesize
\setlength{\tabcolsep}{3pt}
\begin{tabular}{@{}L{0.205\linewidth}lL{0.12\linewidth}L{0.145\linewidth}L{0.17\linewidth}L{0.18\linewidth}@{}}
\toprule
\multicolumn{6}{@{}l}{\emph{Shared by all models}} \\
Layers & \multicolumn{5}{l@{}}{$32$ pre-norm layers with residual-branch multiplier $1$} \\
Normalization & \multicolumn{5}{l@{}}{RMSNorm with learnable gains and $\epsilon_N=10^{-5}$; query--key RMSNorm per head} \\
Attention & \multicolumn{5}{l@{}}{rotary embeddings with base $10{,}000$, logit scale $d_{\mathrm{head}}^{-1/2}$, elementwise sigmoid gate} \\
FFN and biases & \multicolumn{5}{l@{}}{SwiGLU; no biases} \\
Embeddings & \multicolumn{5}{l@{}}{untied, over the $50{,}432$-token GPT-NeoX vocabulary} \\
Initialization & \multicolumn{5}{l@{}}{truncated normal at $\pm3$ standard deviations} \\
\midrule
Setting & Proxy & Standalone & MatFormer & Fully elastic & \Ours{} \\
\midrule
Elastic dimensions & -- & -- & FFN hidden width & residual, heads, FFN hidden width, embedding, readout & residual, heads, FFN hidden width, embedding, terminal norm, readout \\ \cmidrule{1-6}
Trunk normalization & RMSNorm & RMSNorm & RMSNorm & over the active width & group-local over $\mathcal G_1,\mathcal G_2,\mathcal G_3,\mathcal G_4,\mathcal G_5$ \\ \cmidrule{1-6}
Terminal norm and readout & single & single & shared, one learned $\alpha$ & shared and sliced, learned $\alpha_s$ & per width: $N^{\rm out}_s$, $U_s$, learned $\alpha_s$ \\ \cmidrule{1-6}
Sub-networks per step & -- & -- & one, uniformly at random & one, uniformly at random & all five, $\omega_s=(1,1,1,1,4)$ \\ \cmidrule{1-6}
Parameterization & \mup{}, $m=1$ & \mup{}, $m=\text{width}/128$ & \muE{}, $m=4$ (Table~\ref{tab:param-ffn}) & \muE{}, $m=1/2$ (Table~\ref{tab:param-full}) & \muE{}, $m=1$ (Table~\ref{tab:widthparam}) \\ \cmidrule{1-6}
Parameters (M) & $17.1$ & $119.0$--$287.5$ & $118.8$--$282.4$ & $119.0$--$287.5$ & $118.8$--$279.6$ \\
\bottomrule
\end{tabular}
\end{table}

\subsection{Optimization}
\label{app:optim}

Table~\ref{tab:hparams} lists the optimization settings used unless otherwise specified: the proxy sweeps, the elastic models, the baselines and the standalone models differ only in their architecture and in the per-group multipliers of their parameterization.
The training objective includes the $z$-loss of Equation~\ref{eq:loss} with $\lambda_z=10^{-4}$; all reported losses, both training and validation, are the cross-entropy alone, averaged over non-padded tokens.

\begin{table}[t]
\caption{\textbf{Training HPs} used unless otherwise specified (Appendix~\ref{app:configs} lists the changes at the two larger scales). The base values $\eta_0$, $\sigma_0$ and $\lambda_0$ are scaled per parameter group by \muE{} (Table~\ref{tab:widthparam} and Appendix~\ref{app:other-param}) or by \mup{}; SP runs use them verbatim.}
\label{tab:hparams}
\centering
\small
\setlength{\tabcolsep}{4pt}
\begin{tabular}{@{}ll@{}}
\toprule
HP & Value \\
\midrule
Sequence length & $2{,}048$ \\
Batch size & $128$ sequences ($262{,}144$ tokens) \\
Training steps & $25{,}000$ ($6.55$B tokens) \\
Learning-rate schedule & warmup-stable-decay \\
Learning-rate warmup & linear, first $1{,}000$ steps \\
Learning-rate decay & linear to $0$, from step $20{,}000$ to step $25{,}000$ \\
Base learning rate $\eta_0$ & $2\times10^{-3}$ \\
Base initialization $\sigma_0$ & $3.3\times10^{-3}$, truncated normal at $\pm3$ standard deviations \\
Optimizer & AdamW with decoupled weight decay \\
Base weight decay $\lambda_0$ & $0.01$; none on norm gains and $\alpha_s$ \\
$\beta_1$, $\beta_2$ & $0.95$, $0.95$ \\
AdamW $\epsilon$ & $10^{-14}$ \\
Gradient clipping & global norm $4$ \\
Step skipping & after step $5{,}000$, a step is skipped if its gradient norm exceeds $20\times$ the clip value \\
$z$-loss coefficient $\lambda_z$ & $10^{-4}$ \\
Precision & bf16 mixed precision with fp32 master weights and fp32 gradient reduction \\
Parallelism & hybrid-sharded data parallelism (4-GPU shards, two replicas) \\
Hardware & one node of 8 A100, H100 or H200 GPUs \\
Seeds & $95$ (model), $1$ (data) \\
Validation & C4 validation loss every $2{,}500$ steps, reported at step $25{,}000$ \\
\bottomrule
\end{tabular}
\end{table}

\subsection{Parameter counts}
\label{app:params}

The parameters of a sub-network count the non-zero weights it reads, its embedding columns, and its own terminal norm and readout; the two tensors of a frozen-plus-delta split (Equation~\ref{eq:split}) count once.
The baselines were sized to these counts (Table~\ref{tab:arch}), and Table~\ref{tab:main-params} lists them for every sub-network of the main results at the three scales.

\begin{table}[t]
\caption{\textbf{Parameters (M) of the sub-networks of Table~\ref{tab:main-scales}} at each scale and tier, counted as in Appendix~\ref{app:params}; Avg.\ averages the five tiers. Standalone, sampled standalone and the fully elastic model share the same five dense architectures.}
\label{tab:main-params}
\centering
\begin{tabular}{@{}c@{\hspace{3pt}}lcccccc@{}}
\toprule
 & Method & Low & Medium & High & Extra & Max & Avg. \\
\midrule
\multirow{5}{*}{\rotatebox[origin=c]{90}{\strut $119$M--$285$M}} & Standalone & 119.0 & 154.4 & 194.3 & 238.7 & 287.5 & 198.8 \\
 & Sampled standalone & 119.0 & 154.4 & 194.3 & 238.7 & 287.5 & 198.8 \\
 & MatFormer & 118.8 & 150.3 & 188.0 & 232.1 & 282.4 & 194.3 \\
 & Fully elastic & 119.0 & 154.4 & 194.3 & 238.7 & 287.5 & 198.8 \\
 & \textbf{\Ours} & 118.8 & 152.7 & 190.8 & 233.1 & 279.6 & 195.0 \\
\midrule
\multirow{5}{*}{\rotatebox[origin=c]{90}{\strut $375$M--$906$M}} & Standalone & 377.5 & 487.7 & 611.8 & 749.8 & 901.6 & 625.7 \\
 & Sampled standalone & 377.5 & 487.7 & 611.8 & 749.8 & 901.6 & 625.7 \\
 & MatFormer & 371.9 & 485.1 & 610.9 & 749.3 & 912.9 & 626.0 \\
 & Fully elastic & 377.5 & 487.7 & 611.8 & 749.8 & 901.6 & 625.7 \\
 & \textbf{\Ours} & 371.9 & 481.6 & 608.1 & 751.4 & 911.5 & 624.9 \\
\midrule
\multirow{5}{*}{\rotatebox[origin=c]{90}{\strut $675$M--$1.69$B}} & Standalone & 676.6 & 885.8 & 1123.2 & 1388.5 & 1682.0 & 1151.2 \\
 & Sampled standalone & 676.6 & 885.8 & 1123.2 & 1388.5 & 1682.0 & 1151.2 \\
 & MatFormer & 673.8 & 875.2 & 1126.8 & 1403.7 & 1705.6 & 1157.0 \\
 & Fully elastic & 676.6 & 885.8 & 1123.2 & 1388.5 & 1682.0 & 1151.2 \\
 & \textbf{\Ours} & 673.8 & 878.0 & 1117.7 & 1393.1 & 1704.2 & 1153.4 \\
\bottomrule
\end{tabular}
\end{table}

\subsection{Training and inference cost}
\label{app:cost}

\paragraph{Compute.}
By Equation~\ref{eq:cost}, a step of \Ours{} runs one trunk pass at the full width and adds only one terminal norm and readout per supported width.
Updating every width once takes one trunk pass per width with MatFormer, the fully elastic model and the standalone models, and sampling one width per step instead makes the cost vary from step to step.

\paragraph{Measurement protocol.}
For Table~\ref{tab:main-scale1-cost}, each tier runs as its own sub-network on A100-40GB GPUs, with the model sharded as in training.
Latency is the mean of three passes after two warm-up passes (forward and backward for training, forward only for inference), and peak training memory is PyTorch's peak allocated memory per GPU during a training step, which for \Ours{} trains all five tiers.

\paragraph{Activation checkpointing.}
The runs of Table~\ref{tab:main-scale1-cost} train with activation checkpointing on every other transformer block ($16$ of $32$), whose forward pass is recomputed during the backward pass.
Training latency therefore includes this recomputation (on H100 GPUs, the Standalone models train $14$--$19\%$ faster without it), and peak training memory keeps only the inputs of the checkpointed blocks (without checkpointing, Standalone needs $1.9\times$ the memory); inference latency is a plain forward pass and is unaffected.

\paragraph{Sampled standalone.}
Sampled standalone trains the same five models as Standalone, one per step, so its per-tier latency and memory are those of Standalone: in its training runs, the logged latencies agree within $2\%$, and the $2.2$--$2.6$\,GB of extra memory it logged comes from our implementation keeping the four inactive models' weights and optimizer states on the GPU, which offloading them would remove.

\paragraph{Observations.}
Two observations emerge:
\begin{itemize}
\item \textbf{The shared trunk is the saving of the architecture.} One training step of \Ours{} updates all five tiers, whereas updating every tier once takes five steps with MatFormer and the fully elastic model (one per tier) and one step of each model with Standalone.
\item \textbf{Fused kernels lower the per-tier cost.} At equal parameters, all methods need about the same effective FLOPs per tier, so the lower per-tier latency and memory of \Ours{} in Table~\ref{tab:main-scale1-cost} come from our fused implementation: Triton kernels compute the RMSNorms of every transformer block, including the query and key normalizations, and read the queries and keys directly from the output of the QKV projection without copying them.
\end{itemize}

\begin{table}[t]
\caption{\textbf{\boldmath Latency and peak training memory at the $119$M--$285$M scale} on A100-40GB GPUs, for one GPU's batch of $16$ sequences of $2{,}048$ tokens, with each tier run as its own sub-network (protocol in Appendix~\ref{app:cost}). Entries are means over runs, with the standard deviation as a subscript; Avg.\ averages the tiers. $^\dagger$Five separately trained models.}
\label{tab:main-scale1-cost}
\centering
\setlength{\tabcolsep}{4pt}
\begin{tabular}{@{}lcccccc@{}}
\toprule
\multirow{2}{*}{Method} & Tier~1 & Tier~2 & Tier~3 & Tier~4 & Tier~5 & \multirow{2}{*}{Avg.} \\
 & Low & Medium & High & Extra & Max & \\
\midrule
\multicolumn{7}{@{}l}{\textit{Train latency (ms, forward + backward)}} \\
Standalone$^\dagger$ & 570\textsubscript{$\pm$2} & 644\textsubscript{$\pm$3} & 723\textsubscript{$\pm$6} & 834\textsubscript{$\pm$5} & 917\textsubscript{$\pm$10} & 737\textsubscript{$\pm$3} \\ \addlinespace[2pt]
Sampled standalone & \multicolumn{6}{c}{\textit{as Standalone}} \\
MatFormer & 641\textsubscript{$\pm$1} & 677\textsubscript{$\pm$2} & 722\textsubscript{$\pm$4} & 776\textsubscript{$\pm$5} & 834\textsubscript{$\pm$6} & 730\textsubscript{$\pm$4} \\
Fully elastic & 597\textsubscript{$\pm$4} & 671\textsubscript{$\pm$3} & 748\textsubscript{$\pm$5} & 858\textsubscript{$\pm$7} & 933\textsubscript{$\pm$9} & 761\textsubscript{$\pm$6} \\
\textbf{\Ours} & 437\textsubscript{$\pm$11} & 527\textsubscript{$\pm$13} & 603\textsubscript{$\pm$8} & 690\textsubscript{$\pm$2} & 747\textsubscript{$\pm$3} & 601\textsubscript{$\pm$6} \\
\midrule
\multicolumn{7}{@{}l}{\textit{Inference latency (ms, forward)}} \\
Standalone$^\dagger$ & 151\textsubscript{$\pm$2} & 175\textsubscript{$\pm$1} & 199\textsubscript{$\pm$2} & 228\textsubscript{$\pm$2} & 254\textsubscript{$\pm$2} & 201\textsubscript{$\pm$1} \\ \addlinespace[2pt]
Sampled standalone & \multicolumn{6}{c}{\textit{as Standalone}} \\
MatFormer & 175\textsubscript{$\pm$1} & 186\textsubscript{$\pm$1} & 200\textsubscript{$\pm$1} & 216\textsubscript{$\pm$2} & 235\textsubscript{$\pm$2} & 202\textsubscript{$\pm$1} \\
Fully elastic & 160\textsubscript{$\pm$2} & 183\textsubscript{$\pm$1} & 207\textsubscript{$\pm$1} & 236\textsubscript{$\pm$1} & 260\textsubscript{$\pm$1} & 209\textsubscript{$\pm$1} \\
\textbf{\Ours} & 141\textsubscript{$\pm$10} & 167\textsubscript{$\pm$15} & 178\textsubscript{$\pm$12} & 205\textsubscript{$\pm$1} & 218\textsubscript{$\pm$2} & 182\textsubscript{$\pm$8} \\
\midrule
\multicolumn{7}{@{}l}{\textit{Peak training memory (GB per GPU)}} \\
Standalone$^\dagger$ & 15.0\textsubscript{$\pm$.0} & 17.5\textsubscript{$\pm$.0} & 20.1\textsubscript{$\pm$.0} & 22.6\textsubscript{$\pm$.0} & 25.2\textsubscript{$\pm$.0} & 20.1\textsubscript{$\pm$.0} \\ \addlinespace[2pt]
Sampled standalone & \multicolumn{6}{c}{\textit{as Standalone}} \\
MatFormer & 15.8\textsubscript{$\pm$.0} & 17.9\textsubscript{$\pm$.0} & 20.3\textsubscript{$\pm$.0} & 23.2\textsubscript{$\pm$.0} & 26.5\textsubscript{$\pm$.0} & 20.7\textsubscript{$\pm$.0} \\
Fully elastic & 16.0\textsubscript{$\pm$.0} & 18.5\textsubscript{$\pm$.0} & 20.9\textsubscript{$\pm$.0} & 23.4\textsubscript{$\pm$.0} & 25.8\textsubscript{$\pm$.0} & 20.9\textsubscript{$\pm$.0} \\
\textbf{\Ours} & \multicolumn{5}{c}{21.6\textsubscript{$\pm$.0} (a step that trains all five tiers)} & -- \\
\bottomrule
\end{tabular}
\end{table}

\section{Evaluation details}
\label{app:eval}

\paragraph{Validation loss.}
The validation loss is the next-token cross-entropy on held-out C4 text, averaged over non-padded tokens, computed every $2{,}500$ steps and reported at the final step.
Each supported sub-network is extracted from the supernet as in Section~\ref{sec:training} and evaluated on its own, on the same validation batches as every other model, so losses are comparable across widths, families and the standalone models.

\paragraph{Downstream tasks.}
During pretraining, we also probe 5-shot accuracy on ARC-Easy, PIQA and HellaSwag at the same interval.
After pretraining, we evaluate every sub-network on $13$ datasets spanning five skill areas, following the protocol of DCLM and OLMoE, in a 5-shot setting on the standard community splits (Table~\ref{tab:eval-datasets}).
Multiple-choice tasks are scored by accuracy, with a random baseline of $25\%$ for four-choice and $50\%$ for two-choice items; WinoGrande uses its cloze formulation, LAMBADA is scored by exact match of the final word, and SVAMP by exact match of the generated numerical answer.
For elastic models, we report the accuracy of every sub-network and their mean.
The datasets of each skill area are described below.

\begin{table}[htp]
\caption{\textbf{Downstream evaluation suite.} All tasks are evaluated 5-shot; ARC-Easy, PIQA and HellaSwag are also probed during pretraining. MC: multiple-choice accuracy; EM: exact match.}
\label{tab:eval-datasets}
\centering
\footnotesize
\setlength{\tabcolsep}{4pt}
\begin{tabular}{@{}lllrlc@{}}
\toprule
Field & Dataset & Split & \#Items & Metric & Random \\
\midrule
\multirow{3}{*}{World knowledge} & MMLU & test & $14{,}042$ & MC accuracy & $25\%$ \\
 & ARC-Easy & test & $2{,}376$ & MC accuracy & $25\%$ \\
 & ARC-Challenge & test & $1{,}172$ & MC accuracy & $25\%$ \\
\midrule
Symbolic problem solving & SVAMP & test & $300$ & EM & $0\%$ \\
\midrule
\multirow{2}{*}{Commonsense reasoning} & COPA & val & $100$ & MC accuracy & $50\%$ \\
 & PIQA & val & $1{,}838$ & MC accuracy & $50\%$ \\
\midrule
\multirow{3}{*}{Language understanding} & HellaSwag & val & $10{,}042$ & MC accuracy & $25\%$ \\
 & WinoGrande & val & $1{,}267$ & cloze accuracy & $50\%$ \\
 & LAMBADA & test & $5{,}153$ & final-word EM & $0\%$ \\
\midrule
\multirow{4}{*}{Reading comprehension} & BoolQ & val & $3{,}270$ & MC accuracy & $50\%$ \\
 & AGIEval LSAT-RC & test & $268$ & MC accuracy & $25\%$ \\
 & AGIEval LSAT-LR & test & $510$ & MC accuracy & $25\%$ \\
 & AGIEval SAT-EN & test & $206$ & MC accuracy & $25\%$ \\
\bottomrule
\end{tabular}
\end{table}

\paragraph{World knowledge.}
Exam-style and science questions that probe academic knowledge.
\begin{itemize}
\item \textbf{MMLU}~\citep{hendrycks2021mmlu} ($14{,}042$ items, four-choice): exam-style questions from $57$ subjects across the sciences, humanities and social sciences.
\item \textbf{ARC-Easy} and \textbf{ARC-Challenge}~\citep{clark2018arc} ($2{,}376$ and $1{,}172$ items, four-choice): grade-school science questions; the Challenge subset holds the questions that retrieval and co-occurrence methods answer incorrectly.
\end{itemize}

\paragraph{Symbolic problem solving.}
Arithmetic reasoning with a generated answer.
\begin{itemize}
\item \textbf{SVAMP}~\citep{patel2021svamp} ($300$ problems, free-response): short arithmetic word problems with numerical answers; the model generates a derivation followed by the answer, which is scored by exact match.
\end{itemize}

\paragraph{Commonsense reasoning.}
Everyday causal and physical knowledge.
\begin{itemize}
\item \textbf{COPA}~\citep{gordon2012copa} ($100$ items, two-choice): choose the more plausible cause or effect of a premise.
\item \textbf{PIQA}~\citep{bisk2020piqa} ($1{,}838$ items, two-choice): choose the more sensible solution to a physical task.
\end{itemize}

\paragraph{Language understanding.}
Narrative coherence, coreference and the use of broad context.
\begin{itemize}
\item \textbf{HellaSwag}~\citep{zellers2019hellaswag} ($10{,}042$ items, four-choice): choose the most plausible continuation of a partial narrative.
\item \textbf{WinoGrande}~\citep{sakaguchi2021winogrande} ($1{,}267$ items, two-choice): fill a blank with one of two candidates so that the sentence is coherent.
\item \textbf{LAMBADA}~\citep{paperno2016lambada} ($5{,}153$ passages): predict the final word of a passage that requires broad discourse context.
\end{itemize}

\paragraph{Reading comprehension.}
Questions answered from a given passage or argument.
\begin{itemize}
\item \textbf{BoolQ}~\citep{clark2019boolq} ($3{,}270$ items, yes/no): answer a binary question about a short passage.
\item \textbf{AGIEval LSAT-RC, LSAT-LR and SAT-EN}~\citep{zhong2024agieval} ($268$, $510$ and $206$ items, four-choice): LSAT reading-comprehension and logical-reasoning questions and SAT English questions.
\end{itemize}

\section{Complete main results}
\label{app:main-full}

Tables~\ref{tab:main-scale1-tier1}--\ref{tab:main-scale3-tier5} give every metric of every sub-network in Table~\ref{tab:main-scales}, one table per scale and tier, and Figures~\ref{fig:main-scale1}--\ref{fig:main-scale3} show the training curves of each scale:
\begin{itemize}
\item \textbf{\boldmath $119$M--$285$M scale} ($25{,}000$ steps, three runs): Figure~\ref{fig:main-scale1} and Tables~\ref{tab:main-scale1-tier1}--\ref{tab:main-scale1-tier5};
\item \textbf{\boldmath $375$M--$906$M scale} ($50{,}000$ steps, three runs): Figure~\ref{fig:main-scale2} and Tables~\ref{tab:main-scale2-tier1}--\ref{tab:main-scale2-tier5};
\item \textbf{\boldmath $675$M--$1.69$B scale} ($100{,}000$ steps, single runs): Figure~\ref{fig:main-scale3} and Tables~\ref{tab:main-scale3-tier1}--\ref{tab:main-scale3-tier5}.
\end{itemize}

\paragraph{Tables.}
Each table gives the parameter count, the train loss (cross-entropy without the $z$-loss, mean over the last $500$ steps), the validation loss at the last step, and the 5-shot accuracy (\%) on each of the $13$ downstream tasks of Appendix~\ref{app:eval}, grouped by field as there (LSAT-RC, LSAT-LR and SAT-EN are from AGIEval), with their average (the accuracy of Table~\ref{tab:main-scales}).
At the $119$M--$285$M and $375$M--$906$M scales, entries are the mean over three runs of the same configuration and seed, with the standard deviation as a subscript; at the $675$M--$1.69$B scale, every entry is a single run.
Standalone$^\dagger$ is five separately trained models, each trained for the full number of steps, which together take $5\times$ the optimizer steps of any other column.
Bold marks the best entry among the four methods that train all five sizes within one run.

\paragraph{Curves.}
Each figure has one column per tier and plots, against the global training step, the train loss (the logged means over $50$-step intervals), the validation loss, and 5-shot accuracy on ARC-Easy, HellaSwag and PIQA, the three benchmarks evaluated during training.
Sampled standalone, MatFormer and the fully elastic model sample one sub-network uniformly at random per step, \Ours{} updates all sizes at every step, and Standalone trains each size separately for the full number of steps ($5\times$ the steps).
Values are plotted as logged, without smoothing or clipping, and y-axes are shared within each row.

\begin{figure}[p]
\centering
\includegraphics[width=\linewidth]{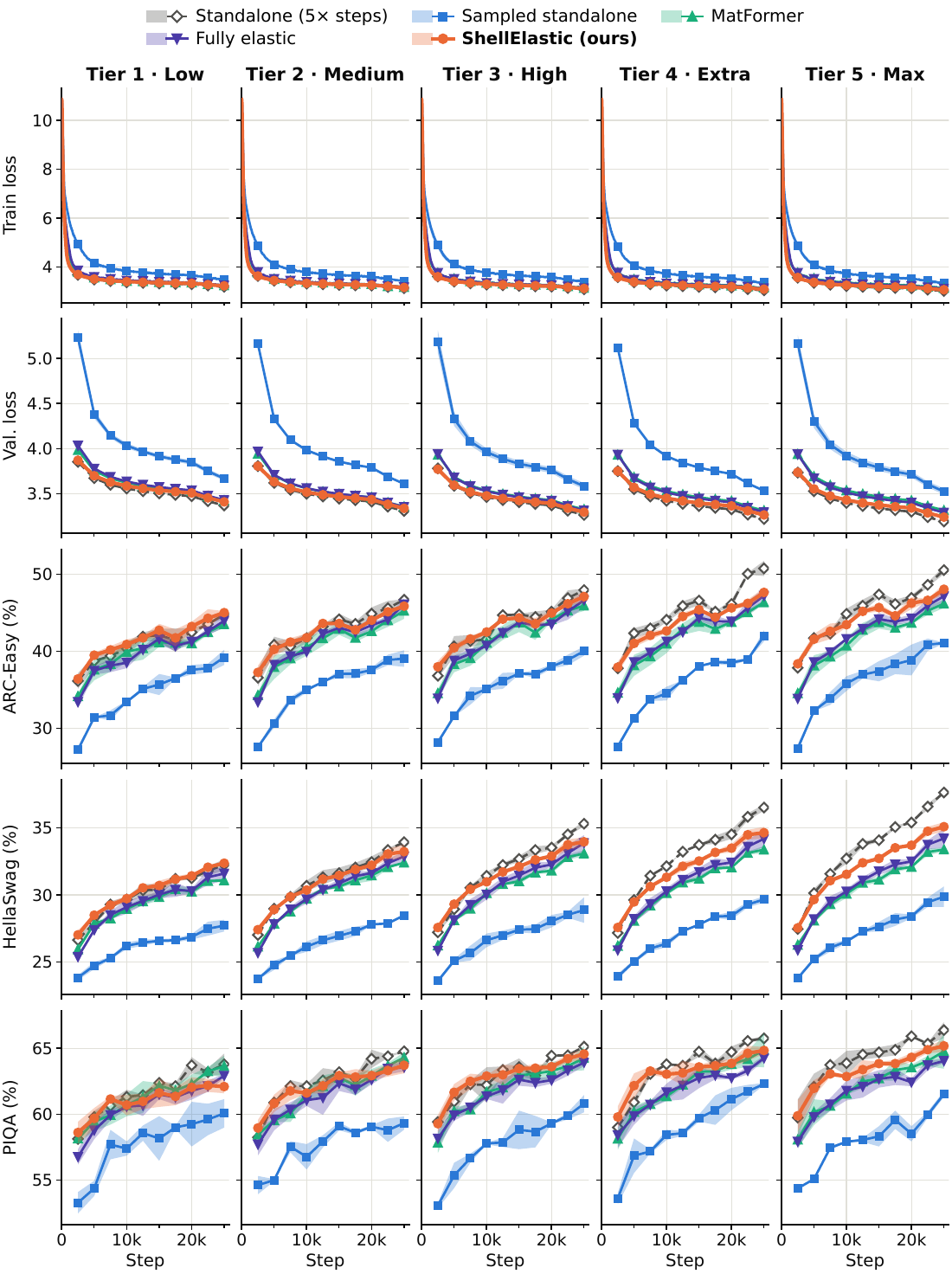}
\caption{\textbf{\boldmath Training curves at the $119$M--$285$M scale} (Table~\ref{tab:main-scales}). Columns are the five tiers, from Tier~1 (Low, about $119$M parameters) to Tier~5 (Max, about $285$M); rows are the train loss, the validation loss and 5-shot accuracy on ARC-Easy, HellaSwag and PIQA. Lines are means over three runs, with shaded bands of $\pm1$ standard deviation; conventions in Appendix~\ref{app:main-full}.}
\label{fig:main-scale1}
\end{figure}

\begin{table}[tbp]
\caption{\textbf{\boldmath Per-task results at the $119$M--$285$M scale, Tier~1 (Low, about $119$M parameters).} Sub-network width $384$ for Standalone, sampled standalone and the fully elastic model, FFN hidden width $512$ at $D_{\mathrm{model}}=512$ for MatFormer, and $512$ for \Ours{}. Mean over three runs, with the standard deviation as a subscript.}
\label{tab:main-scale1-tier1}
\centering
\setlength{\tabcolsep}{3pt}
\begin{tabular}{@{}lccccc@{}}
\toprule
\multicolumn{1}{@{}l@{}}{\textbf{Tier~1\,$\cdot$\,Low}} & Standalone$^\dagger$ & Sampled standalone & MatFormer & Fully elastic & \textbf{\Ours} \\
\midrule
Params (M) & 119.0 & 119.0 & 118.8 & 119.0 & 118.8 \\
Train loss $\downarrow$ & 3.1951\textsubscript{$\pm$.0035} & 3.4911\textsubscript{$\pm$.0235} & 3.2536\textsubscript{$\pm$.0032} & 3.2572\textsubscript{$\pm$.0091} & \textbf{3.2305}\textsubscript{$\pm$.0015} \\
Val.\ loss $\downarrow$ & 3.3695\textsubscript{$\pm$.0034} & 3.6674\textsubscript{$\pm$.0242} & 3.4248\textsubscript{$\pm$.0036} & 3.4291\textsubscript{$\pm$.0097} & \textbf{3.4086}\textsubscript{$\pm$.0009} \\
\midrule
\multicolumn{6}{@{}l}{\emph{5-shot accuracy (\%) $\uparrow$}} \\
MMLU & 24.5\textsubscript{$\pm$.3} & 24.4\textsubscript{$\pm$.2} & \textbf{24.8}\textsubscript{$\pm$.5} & 24.7\textsubscript{$\pm$.7} & 24.7\textsubscript{$\pm$.6} \\
ARC-Easy & 44.8\textsubscript{$\pm$.5} & 39.2\textsubscript{$\pm$1.0} & 43.5\textsubscript{$\pm$.5} & 44.0\textsubscript{$\pm$.9} & \textbf{45.0}\textsubscript{$\pm$.2} \\
ARC-Challenge & 23.3\textsubscript{$\pm$1.7} & 22.7\textsubscript{$\pm$.5} & \textbf{23.9}\textsubscript{$\pm$.7} & 23.3\textsubscript{$\pm$1.2} & 23.3\textsubscript{$\pm$.5} \\ \addlinespace[3pt]
SVAMP & 5.7\textsubscript{$\pm$.5} & 6.3\textsubscript{$\pm$.9} & \textbf{6.8}\textsubscript{$\pm$1.5} & 6.4\textsubscript{$\pm$.8} & 5.7\textsubscript{$\pm$.3} \\ \addlinespace[3pt]
COPA & 61.0\textsubscript{$\pm$.0} & 54.0\textsubscript{$\pm$5.7} & 57.7\textsubscript{$\pm$.6} & 58.7\textsubscript{$\pm$1.5} & \textbf{59.7}\textsubscript{$\pm$.6} \\
PIQA & 63.8\textsubscript{$\pm$.8} & 60.1\textsubscript{$\pm$1.1} & \textbf{63.6}\textsubscript{$\pm$.8} & 62.9\textsubscript{$\pm$.2} & 62.1\textsubscript{$\pm$.5} \\ \addlinespace[3pt]
HellaSwag & 32.2\textsubscript{$\pm$.4} & 27.7\textsubscript{$\pm$.4} & 31.1\textsubscript{$\pm$.1} & 31.6\textsubscript{$\pm$.5} & \textbf{32.4}\textsubscript{$\pm$.2} \\
WinoGrande & 52.5\textsubscript{$\pm$.1} & \textbf{52.1}\textsubscript{$\pm$.2} & 50.7\textsubscript{$\pm$.8} & 49.0\textsubscript{$\pm$1.3} & 50.7\textsubscript{$\pm$.8} \\
LAMBADA & 25.7\textsubscript{$\pm$1.1} & 17.4\textsubscript{$\pm$1.7} & 24.7\textsubscript{$\pm$.4} & 23.6\textsubscript{$\pm$1.0} & \textbf{25.9}\textsubscript{$\pm$.7} \\ \addlinespace[3pt]
BoolQ & 51.1\textsubscript{$\pm$7.2} & 54.9\textsubscript{$\pm$7.6} & 50.4\textsubscript{$\pm$5.8} & 53.0\textsubscript{$\pm$.6} & \textbf{55.1}\textsubscript{$\pm$7.2} \\
LSAT-RC & 25.2\textsubscript{$\pm$1.3} & \textbf{25.2}\textsubscript{$\pm$.3} & 24.3\textsubscript{$\pm$1.7} & 25.1\textsubscript{$\pm$1.8} & 24.8\textsubscript{$\pm$1.8} \\
LSAT-LR & 26.9\textsubscript{$\pm$.0} & \textbf{28.8}\textsubscript{$\pm$.3} & 27.9\textsubscript{$\pm$.6} & 25.4\textsubscript{$\pm$1.3} & 26.2\textsubscript{$\pm$1.0} \\
SAT-EN & 26.5\textsubscript{$\pm$1.7} & \textbf{26.9}\textsubscript{$\pm$.3} & 26.7\textsubscript{$\pm$1.3} & 26.1\textsubscript{$\pm$1.7} & 25.6\textsubscript{$\pm$.3} \\
\midrule
Average & 35.6\textsubscript{$\pm$.9} & 33.8\textsubscript{$\pm$.2} & 35.1\textsubscript{$\pm$.3} & 34.9\textsubscript{$\pm$.2} & \textbf{35.5}\textsubscript{$\pm$.4} \\
\bottomrule
\end{tabular}
\end{table}

\begin{table}[tbp]
\caption{\textbf{\boldmath Per-task results at the $119$M--$285$M scale, Tier~2 (Medium, about $153$M parameters).} Sub-network width $448$ for Standalone, sampled standalone and the fully elastic model, FFN hidden width $1152$ at $D_{\mathrm{model}}=512$ for MatFormer, and $640$ for \Ours{}. Mean over three runs, with the standard deviation as a subscript.}
\label{tab:main-scale1-tier2}
\centering
\setlength{\tabcolsep}{3pt}
\begin{tabular}{@{}lccccc@{}}
\toprule
\multicolumn{1}{@{}l@{}}{\textbf{Tier~2\,$\cdot$\,Medium}} & Standalone$^\dagger$ & Sampled standalone & MatFormer & Fully elastic & \textbf{\Ours} \\
\midrule
Params (M) & 154.4 & 154.4 & 150.3 & 154.4 & 152.7 \\
Train loss $\downarrow$ & 3.1344\textsubscript{$\pm$.0019} & 3.4301\textsubscript{$\pm$.0095} & 3.1821\textsubscript{$\pm$.0034} & 3.1783\textsubscript{$\pm$.0083} & \textbf{3.1641}\textsubscript{$\pm$.0031} \\
Val.\ loss $\downarrow$ & 3.3081\textsubscript{$\pm$.0016} & 3.6091\textsubscript{$\pm$.0104} & 3.3566\textsubscript{$\pm$.0037} & 3.3536\textsubscript{$\pm$.0083} & \textbf{3.3375}\textsubscript{$\pm$.0027} \\
\midrule
\multicolumn{6}{@{}l}{\emph{5-shot accuracy (\%) $\uparrow$}} \\
MMLU & 25.0\textsubscript{$\pm$.4} & \textbf{24.9}\textsubscript{$\pm$.3} & \textbf{24.9}\textsubscript{$\pm$.2} & 24.7\textsubscript{$\pm$.3} & 23.8\textsubscript{$\pm$.7} \\
ARC-Easy & 46.7\textsubscript{$\pm$.1} & 39.1\textsubscript{$\pm$1.1} & 45.2\textsubscript{$\pm$1.0} & \textbf{46.1}\textsubscript{$\pm$.8} & 45.8\textsubscript{$\pm$.1} \\
ARC-Challenge & 24.4\textsubscript{$\pm$.4} & 23.0\textsubscript{$\pm$.7} & 23.6\textsubscript{$\pm$.0} & \textbf{24.0}\textsubscript{$\pm$.4} & 23.5\textsubscript{$\pm$.4} \\ \addlinespace[3pt]
SVAMP & 6.8\textsubscript{$\pm$.2} & 5.8\textsubscript{$\pm$.7} & 6.7\textsubscript{$\pm$1.5} & \textbf{7.2}\textsubscript{$\pm$1.4} & 6.3\textsubscript{$\pm$.9} \\ \addlinespace[3pt]
COPA & 59.0\textsubscript{$\pm$1.4} & 58.0\textsubscript{$\pm$1.4} & 60.0\textsubscript{$\pm$1.7} & 60.3\textsubscript{$\pm$1.5} & \textbf{61.7}\textsubscript{$\pm$1.2} \\
PIQA & 64.8\textsubscript{$\pm$.0} & 59.3\textsubscript{$\pm$.5} & \textbf{64.3}\textsubscript{$\pm$.4} & 63.5\textsubscript{$\pm$.3} & 63.7\textsubscript{$\pm$.6} \\ \addlinespace[3pt]
HellaSwag & 33.9\textsubscript{$\pm$.1} & 28.5\textsubscript{$\pm$.1} & 32.4\textsubscript{$\pm$.1} & 32.9\textsubscript{$\pm$.3} & \textbf{33.2}\textsubscript{$\pm$.5} \\
WinoGrande & 51.7\textsubscript{$\pm$2.2} & \textbf{51.5}\textsubscript{$\pm$2.7} & 50.5\textsubscript{$\pm$.6} & 50.1\textsubscript{$\pm$.3} & 49.3\textsubscript{$\pm$1.8} \\
LAMBADA & 28.4\textsubscript{$\pm$.3} & 18.9\textsubscript{$\pm$.3} & 26.4\textsubscript{$\pm$.6} & 26.4\textsubscript{$\pm$1.2} & \textbf{29.1}\textsubscript{$\pm$.7} \\ \addlinespace[3pt]
BoolQ & 49.9\textsubscript{$\pm$4.8} & 48.1\textsubscript{$\pm$6.2} & 50.3\textsubscript{$\pm$6.1} & 50.7\textsubscript{$\pm$.6} & \textbf{57.2}\textsubscript{$\pm$3.4} \\
LSAT-RC & 24.6\textsubscript{$\pm$4.2} & 24.8\textsubscript{$\pm$2.9} & 22.0\textsubscript{$\pm$3.3} & 24.6\textsubscript{$\pm$.4} & \textbf{25.6}\textsubscript{$\pm$1.4} \\
LSAT-LR & 24.9\textsubscript{$\pm$.6} & 26.3\textsubscript{$\pm$1.9} & 25.9\textsubscript{$\pm$1.5} & 26.3\textsubscript{$\pm$3.0} & \textbf{29.2}\textsubscript{$\pm$.9} \\
SAT-EN & 25.7\textsubscript{$\pm$.0} & 25.5\textsubscript{$\pm$1.7} & 25.2\textsubscript{$\pm$3.0} & 24.8\textsubscript{$\pm$1.3} & \textbf{27.2}\textsubscript{$\pm$3.0} \\
\midrule
Average & 35.8\textsubscript{$\pm$1.0} & 33.4\textsubscript{$\pm$1.1} & 35.2\textsubscript{$\pm$.2} & 35.5\textsubscript{$\pm$.4} & \textbf{36.6}\textsubscript{$\pm$.6} \\
\bottomrule
\end{tabular}
\end{table}

\begin{table}[tbp]
\caption{\textbf{\boldmath Per-task results at the $119$M--$285$M scale, Tier~3 (High, about $192$M parameters).} Sub-network width $512$ for Standalone, sampled standalone and the fully elastic model, FFN hidden width $1920$ at $D_{\mathrm{model}}=512$ for MatFormer, and $768$ for \Ours{}. Mean over three runs, with the standard deviation as a subscript.}
\label{tab:main-scale1-tier3}
\centering
\setlength{\tabcolsep}{3pt}
\begin{tabular}{@{}lccccc@{}}
\toprule
\multicolumn{1}{@{}l@{}}{\textbf{Tier~3\,$\cdot$\,High}} & Standalone$^\dagger$ & Sampled standalone & MatFormer & Fully elastic & \textbf{\Ours} \\
\midrule
Params (M) & 194.3 & 194.3 & 188.0 & 194.3 & 190.8 \\
Train loss $\downarrow$ & 3.0916\textsubscript{$\pm$.0086} & 3.4087\textsubscript{$\pm$.0336} & 3.1561\textsubscript{$\pm$.0036} & 3.1447\textsubscript{$\pm$.0077} & \textbf{3.1229}\textsubscript{$\pm$.0018} \\
Val.\ loss $\downarrow$ & 3.2635\textsubscript{$\pm$.0092} & 3.5826\textsubscript{$\pm$.0344} & 3.3250\textsubscript{$\pm$.0036} & 3.3145\textsubscript{$\pm$.0076} & \textbf{3.2944}\textsubscript{$\pm$.0016} \\
\midrule
\multicolumn{6}{@{}l}{\emph{5-shot accuracy (\%) $\uparrow$}} \\
MMLU & 24.6\textsubscript{$\pm$.4} & \textbf{25.0}\textsubscript{$\pm$.3} & 24.6\textsubscript{$\pm$.4} & 24.6\textsubscript{$\pm$.5} & 24.0\textsubscript{$\pm$.5} \\
ARC-Easy & 47.9\textsubscript{$\pm$.2} & 40.0\textsubscript{$\pm$.6} & 45.9\textsubscript{$\pm$.7} & 46.6\textsubscript{$\pm$1.0} & \textbf{47.1}\textsubscript{$\pm$.8} \\
ARC-Challenge & 23.9\textsubscript{$\pm$.8} & 22.3\textsubscript{$\pm$.8} & 24.3\textsubscript{$\pm$.3} & \textbf{24.5}\textsubscript{$\pm$.7} & 23.6\textsubscript{$\pm$.8} \\ \addlinespace[3pt]
SVAMP & 7.7\textsubscript{$\pm$1.9} & 6.3\textsubscript{$\pm$.0} & 6.6\textsubscript{$\pm$1.5} & \textbf{6.8}\textsubscript{$\pm$3.3} & 6.3\textsubscript{$\pm$1.0} \\ \addlinespace[3pt]
COPA & 60.5\textsubscript{$\pm$2.1} & 55.0\textsubscript{$\pm$1.4} & 60.3\textsubscript{$\pm$2.3} & 60.3\textsubscript{$\pm$3.1} & \textbf{61.7}\textsubscript{$\pm$3.1} \\
PIQA & 65.1\textsubscript{$\pm$.1} & 60.8\textsubscript{$\pm$.7} & 64.3\textsubscript{$\pm$.9} & 63.9\textsubscript{$\pm$.3} & \textbf{64.6}\textsubscript{$\pm$.4} \\ \addlinespace[3pt]
HellaSwag & 35.3\textsubscript{$\pm$.2} & 28.9\textsubscript{$\pm$1.0} & 33.1\textsubscript{$\pm$.2} & \textbf{33.9}\textsubscript{$\pm$.6} & \textbf{33.9}\textsubscript{$\pm$.3} \\
WinoGrande & 52.0\textsubscript{$\pm$.6} & \textbf{50.9}\textsubscript{$\pm$.6} & 50.5\textsubscript{$\pm$.8} & 50.2\textsubscript{$\pm$.5} & 50.0\textsubscript{$\pm$2.1} \\
LAMBADA & 29.6\textsubscript{$\pm$.1} & 19.3\textsubscript{$\pm$1.0} & 27.6\textsubscript{$\pm$.4} & 28.0\textsubscript{$\pm$1.4} & \textbf{29.9}\textsubscript{$\pm$1.1} \\ \addlinespace[3pt]
BoolQ & 54.1\textsubscript{$\pm$.1} & 53.9\textsubscript{$\pm$.8} & 51.1\textsubscript{$\pm$7.6} & 50.1\textsubscript{$\pm$3.5} & \textbf{55.3}\textsubscript{$\pm$1.4} \\
LSAT-RC & 25.6\textsubscript{$\pm$.3} & \textbf{26.3}\textsubscript{$\pm$2.4} & 23.6\textsubscript{$\pm$1.3} & 24.0\textsubscript{$\pm$1.6} & 26.0\textsubscript{$\pm$1.5} \\
LSAT-LR & 24.9\textsubscript{$\pm$.6} & \textbf{27.1}\textsubscript{$\pm$1.7} & \textbf{27.1}\textsubscript{$\pm$1.2} & 25.7\textsubscript{$\pm$2.2} & 26.8\textsubscript{$\pm$1.6} \\
SAT-EN & 24.8\textsubscript{$\pm$.7} & \textbf{26.5}\textsubscript{$\pm$1.0} & \textbf{26.5}\textsubscript{$\pm$3.0} & 25.2\textsubscript{$\pm$1.5} & 26.2\textsubscript{$\pm$1.7} \\
\midrule
Average & 36.6\textsubscript{$\pm$.2} & 34.0\textsubscript{$\pm$.7} & 35.8\textsubscript{$\pm$.5} & 35.7\textsubscript{$\pm$.4} & \textbf{36.6}\textsubscript{$\pm$.4} \\
\bottomrule
\end{tabular}
\end{table}

\begin{table}[tbp]
\caption{\textbf{\boldmath Per-task results at the $119$M--$285$M scale, Tier~4 (Extra, about $236$M parameters).} Sub-network width $576$ for Standalone, sampled standalone and the fully elastic model, FFN hidden width $2816$ at $D_{\mathrm{model}}=512$ for MatFormer, and $896$ for \Ours{}. Mean over three runs, with the standard deviation as a subscript.}
\label{tab:main-scale1-tier4}
\centering
\setlength{\tabcolsep}{3pt}
\begin{tabular}{@{}lccccc@{}}
\toprule
\multicolumn{1}{@{}l@{}}{\textbf{Tier~4\,$\cdot$\,Extra}} & Standalone$^\dagger$ & Sampled standalone & MatFormer & Fully elastic & \textbf{\Ours} \\
\midrule
Params (M) & 238.7 & 238.7 & 232.1 & 238.7 & 233.1 \\
Train loss $\downarrow$ & 3.0474\textsubscript{$\pm$.0006} & 3.3561\textsubscript{$\pm$.0056} & 3.1407\textsubscript{$\pm$.0036} & 3.1233\textsubscript{$\pm$.0071} & \textbf{3.0941}\textsubscript{$\pm$.0015} \\
Val.\ loss $\downarrow$ & 3.2186\textsubscript{$\pm$.0012} & 3.5319\textsubscript{$\pm$.0056} & 3.3123\textsubscript{$\pm$.0037} & 3.2954\textsubscript{$\pm$.0072} & \textbf{3.2641}\textsubscript{$\pm$.0017} \\
\midrule
\multicolumn{6}{@{}l}{\emph{5-shot accuracy (\%) $\uparrow$}} \\
MMLU & 24.6\textsubscript{$\pm$.5} & \textbf{24.7}\textsubscript{$\pm$.1} & 24.5\textsubscript{$\pm$.4} & 24.6\textsubscript{$\pm$.5} & 24.4\textsubscript{$\pm$.7} \\
ARC-Easy & 50.7\textsubscript{$\pm$1.0} & 42.0\textsubscript{$\pm$.7} & 46.3\textsubscript{$\pm$.2} & 47.2\textsubscript{$\pm$.9} & \textbf{47.6}\textsubscript{$\pm$.6} \\
ARC-Challenge & 25.2\textsubscript{$\pm$1.1} & 23.2\textsubscript{$\pm$.4} & 24.0\textsubscript{$\pm$.5} & \textbf{24.8}\textsubscript{$\pm$1.1} & 24.0\textsubscript{$\pm$.5} \\ \addlinespace[3pt]
SVAMP & 8.0\textsubscript{$\pm$.9} & 6.5\textsubscript{$\pm$.7} & \textbf{7.0}\textsubscript{$\pm$.6} & 6.2\textsubscript{$\pm$2.2} & 6.7\textsubscript{$\pm$1.5} \\ \addlinespace[3pt]
COPA & 61.0\textsubscript{$\pm$2.8} & 57.5\textsubscript{$\pm$.7} & 60.7\textsubscript{$\pm$.6} & 61.0\textsubscript{$\pm$.0} & \textbf{63.0}\textsubscript{$\pm$2.6} \\
PIQA & 65.8\textsubscript{$\pm$.2} & 62.3\textsubscript{$\pm$.1} & \textbf{64.9}\textsubscript{$\pm$1.3} & 64.3\textsubscript{$\pm$.2} & 64.8\textsubscript{$\pm$.3} \\ \addlinespace[3pt]
HellaSwag & 36.5\textsubscript{$\pm$.2} & 29.7\textsubscript{$\pm$.2} & 33.4\textsubscript{$\pm$.3} & 34.2\textsubscript{$\pm$.6} & \textbf{34.6}\textsubscript{$\pm$.4} \\
WinoGrande & 50.7\textsubscript{$\pm$.1} & 48.6\textsubscript{$\pm$.4} & 50.1\textsubscript{$\pm$.3} & \textbf{50.2}\textsubscript{$\pm$.6} & 49.8\textsubscript{$\pm$1.4} \\
LAMBADA & 31.7\textsubscript{$\pm$.9} & 20.4\textsubscript{$\pm$.4} & 28.0\textsubscript{$\pm$.4} & 29.2\textsubscript{$\pm$.9} & \textbf{30.2}\textsubscript{$\pm$.9} \\ \addlinespace[3pt]
BoolQ & 50.7\textsubscript{$\pm$1.6} & 51.3\textsubscript{$\pm$5.0} & 50.4\textsubscript{$\pm$7.7} & 49.2\textsubscript{$\pm$2.7} & \textbf{54.5}\textsubscript{$\pm$2.4} \\
LSAT-RC & 25.7\textsubscript{$\pm$1.6} & 25.2\textsubscript{$\pm$.3} & 23.4\textsubscript{$\pm$.9} & 24.1\textsubscript{$\pm$1.5} & \textbf{26.4}\textsubscript{$\pm$1.1} \\
LSAT-LR & 26.5\textsubscript{$\pm$3.6} & 27.8\textsubscript{$\pm$1.1} & \textbf{28.0}\textsubscript{$\pm$2.1} & 25.2\textsubscript{$\pm$1.4} & 27.6\textsubscript{$\pm$2.2} \\
SAT-EN & 24.8\textsubscript{$\pm$3.4} & 25.0\textsubscript{$\pm$.3} & 26.2\textsubscript{$\pm$2.6} & \textbf{26.5}\textsubscript{$\pm$.7} & \textbf{26.5}\textsubscript{$\pm$1.2} \\
\midrule
Average & 37.1\textsubscript{$\pm$.3} & 34.2\textsubscript{$\pm$.1} & 35.9\textsubscript{$\pm$.6} & 35.9\textsubscript{$\pm$.3} & \textbf{36.9}\textsubscript{$\pm$.2} \\
\bottomrule
\end{tabular}
\end{table}

\begin{table}[tbp]
\caption{\textbf{\boldmath Per-task results at the $119$M--$285$M scale, Tier~5 (Max, about $285$M parameters).} Sub-network width $640$ for Standalone, sampled standalone and the fully elastic model, FFN hidden width $3840$ at $D_{\mathrm{model}}=512$ for MatFormer, and $1024$ for \Ours{}. Mean over three runs, with the standard deviation as a subscript.}
\label{tab:main-scale1-tier5}
\centering
\setlength{\tabcolsep}{3pt}
\begin{tabular}{@{}lccccc@{}}
\toprule
\multicolumn{1}{@{}l@{}}{\textbf{Tier~5\,$\cdot$\,Max}} & Standalone$^\dagger$ & Sampled standalone & MatFormer & Fully elastic & \textbf{\Ours} \\
\midrule
Params (M) & 287.5 & 287.5 & 282.4 & 287.5 & 279.6 \\
Train loss $\downarrow$ & 3.0203\textsubscript{$\pm$.0017} & 3.3505\textsubscript{$\pm$.0321} & 3.1436\textsubscript{$\pm$.0035} & 3.1196\textsubscript{$\pm$.0070} & \textbf{3.0713}\textsubscript{$\pm$.0014} \\
Val.\ loss $\downarrow$ & 3.1908\textsubscript{$\pm$.0014} & 3.5239\textsubscript{$\pm$.0340} & 3.3139\textsubscript{$\pm$.0036} & 3.2899\textsubscript{$\pm$.0071} & \textbf{3.2403}\textsubscript{$\pm$.0016} \\
\midrule
\multicolumn{6}{@{}l}{\emph{5-shot accuracy (\%) $\uparrow$}} \\
MMLU & 24.8\textsubscript{$\pm$.9} & 24.5\textsubscript{$\pm$.2} & \textbf{24.8}\textsubscript{$\pm$.4} & 24.3\textsubscript{$\pm$.2} & 24.2\textsubscript{$\pm$.4} \\
ARC-Easy & 50.5\textsubscript{$\pm$.5} & 41.1\textsubscript{$\pm$.1} & 46.2\textsubscript{$\pm$.5} & 47.1\textsubscript{$\pm$.9} & \textbf{48.0}\textsubscript{$\pm$.7} \\
ARC-Challenge & 25.4\textsubscript{$\pm$.4} & 23.2\textsubscript{$\pm$1.1} & 24.2\textsubscript{$\pm$.4} & \textbf{25.1}\textsubscript{$\pm$1.2} & 24.3\textsubscript{$\pm$.8} \\ \addlinespace[3pt]
SVAMP & 5.8\textsubscript{$\pm$.7} & 6.3\textsubscript{$\pm$1.4} & \textbf{7.2}\textsubscript{$\pm$1.0} & 5.7\textsubscript{$\pm$2.4} & 6.4\textsubscript{$\pm$1.2} \\ \addlinespace[3pt]
COPA & 61.0\textsubscript{$\pm$.0} & 58.0\textsubscript{$\pm$2.8} & 60.0\textsubscript{$\pm$1.0} & 60.7\textsubscript{$\pm$1.2} & \textbf{63.3}\textsubscript{$\pm$2.9} \\
PIQA & 66.4\textsubscript{$\pm$.6} & 61.5\textsubscript{$\pm$.2} & 64.8\textsubscript{$\pm$1.3} & 64.0\textsubscript{$\pm$.1} & \textbf{65.2}\textsubscript{$\pm$.5} \\ \addlinespace[3pt]
HellaSwag & 37.6\textsubscript{$\pm$.1} & 29.9\textsubscript{$\pm$.8} & 33.4\textsubscript{$\pm$.1} & 34.2\textsubscript{$\pm$.7} & \textbf{35.1}\textsubscript{$\pm$.3} \\
WinoGrande & 50.7\textsubscript{$\pm$.1} & 49.8\textsubscript{$\pm$1.7} & \textbf{50.4}\textsubscript{$\pm$.9} & 49.9\textsubscript{$\pm$1.1} & 49.8\textsubscript{$\pm$1.4} \\
LAMBADA & 33.0\textsubscript{$\pm$1.9} & 20.0\textsubscript{$\pm$.4} & 27.8\textsubscript{$\pm$.4} & 29.3\textsubscript{$\pm$1.3} & \textbf{30.4}\textsubscript{$\pm$.4} \\ \addlinespace[3pt]
BoolQ & 49.7\textsubscript{$\pm$4.2} & \textbf{55.8}\textsubscript{$\pm$5.4} & 50.7\textsubscript{$\pm$7.3} & 50.6\textsubscript{$\pm$3.4} & 52.3\textsubscript{$\pm$3.1} \\
LSAT-RC & 25.6\textsubscript{$\pm$.3} & 23.9\textsubscript{$\pm$.5} & 25.1\textsubscript{$\pm$.6} & 25.0\textsubscript{$\pm$.6} & \textbf{25.6}\textsubscript{$\pm$1.1} \\
LSAT-LR & 25.7\textsubscript{$\pm$2.2} & 27.2\textsubscript{$\pm$.1} & 27.0\textsubscript{$\pm$3.0} & 25.4\textsubscript{$\pm$2.3} & \textbf{27.6}\textsubscript{$\pm$.6} \\
SAT-EN & 26.2\textsubscript{$\pm$1.4} & 26.5\textsubscript{$\pm$2.4} & 25.2\textsubscript{$\pm$1.0} & \textbf{26.9}\textsubscript{$\pm$.7} & \textbf{26.9}\textsubscript{$\pm$2.0} \\
\midrule
Average & 37.1\textsubscript{$\pm$.7} & 34.4\textsubscript{$\pm$.5} & 35.9\textsubscript{$\pm$.7} & 36.0\textsubscript{$\pm$.5} & \textbf{36.9}\textsubscript{$\pm$.5} \\
\bottomrule
\end{tabular}
\end{table}

\begin{figure}[p]
\centering
\includegraphics[width=\linewidth]{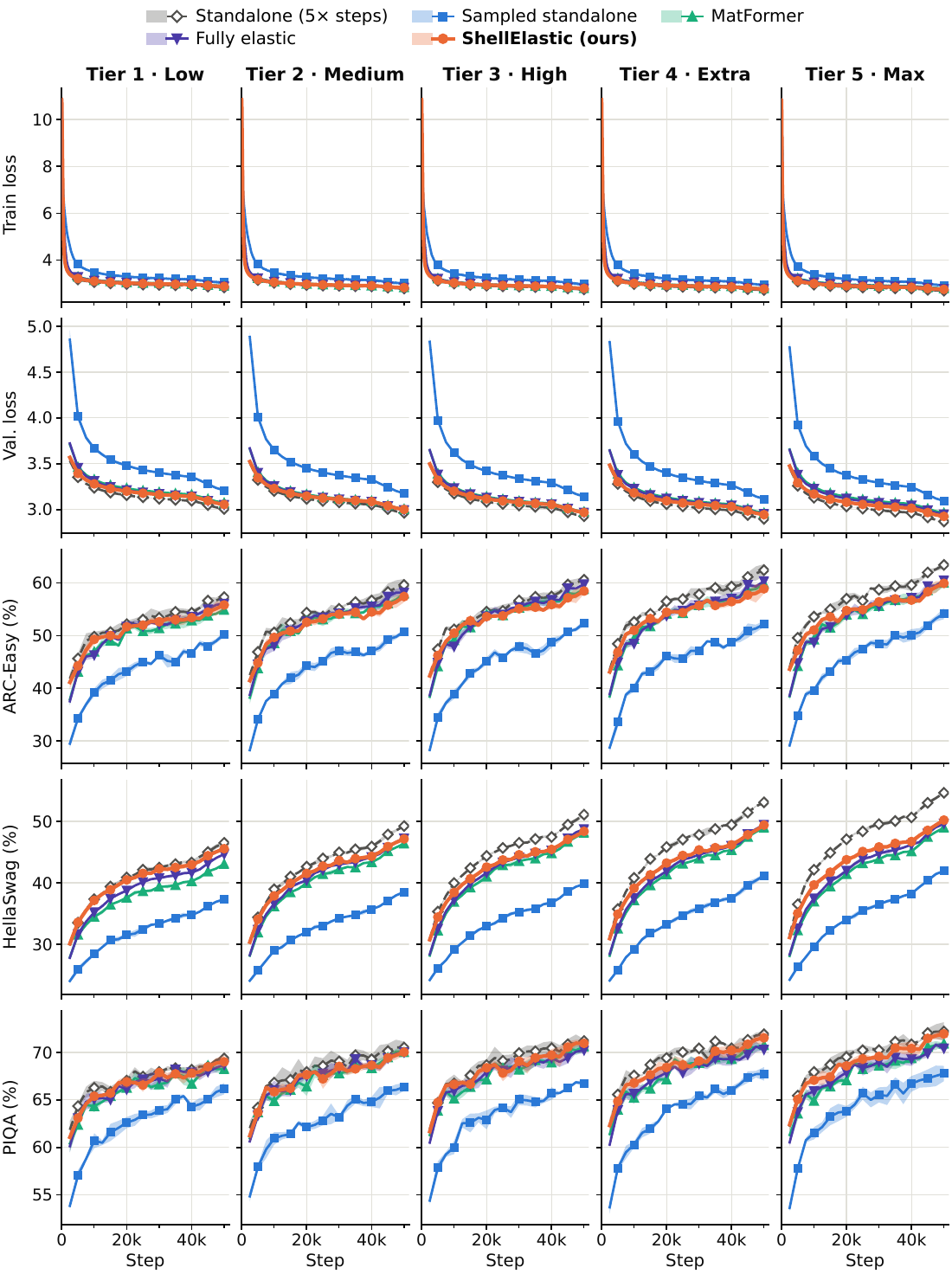}
\caption{\textbf{\boldmath Training curves at the $375$M--$906$M scale} (Table~\ref{tab:main-scales}). Columns are the five tiers, from Tier~1 (Low, about $375$M parameters) to Tier~5 (Max, about $906$M); rows are the train loss, the validation loss and 5-shot accuracy on ARC-Easy, HellaSwag and PIQA. Lines are means over three runs, with shaded bands of $\pm1$ standard deviation; conventions in Appendix~\ref{app:main-full}.}
\label{fig:main-scale2}
\end{figure}

\begin{table}[tbp]
\caption{\textbf{\boldmath Per-task results at the $375$M--$906$M scale, Tier~1 (Low, about $375$M parameters).} Sub-network width $832$ for Standalone, sampled standalone and the fully elastic model, FFN hidden width $1024$ at $D_{\mathrm{model}}=1024$ for MatFormer, and $1024$ for \Ours{}. Mean over three runs, with the standard deviation as a subscript.}
\label{tab:main-scale2-tier1}
\centering
\setlength{\tabcolsep}{3pt}
\begin{tabular}{@{}lccccc@{}}
\toprule
\multicolumn{1}{@{}l@{}}{\textbf{Tier~1\,$\cdot$\,Low}} & Standalone$^\dagger$ & Sampled standalone & MatFormer & Fully elastic & \textbf{\Ours} \\
\midrule
Params (M) & 377.5 & 377.5 & 371.9 & 377.5 & 371.9 \\
Train loss $\downarrow$ & 2.8302\textsubscript{$\pm$.0016} & 3.0381\textsubscript{$\pm$.0070} & 2.8993\textsubscript{$\pm$.0014} & 2.8791\textsubscript{$\pm$.0019} & \textbf{2.8746}\textsubscript{$\pm$.0035} \\
Val.\ loss $\downarrow$ & 3.0069\textsubscript{$\pm$.0016} & 3.2085\textsubscript{$\pm$.0079} & 3.0724\textsubscript{$\pm$.0015} & \textbf{3.0519}\textsubscript{$\pm$.0023} & 3.0546\textsubscript{$\pm$.0039} \\
\midrule
\multicolumn{6}{@{}l}{\small\emph{5-shot accuracy (\%) $\uparrow$}} \\
MMLU & 24.9\textsubscript{$\pm$.5} & \textbf{25.1}\textsubscript{$\pm$.4} & \textbf{25.1}\textsubscript{$\pm$.2} & 23.9\textsubscript{$\pm$.7} & 24.8\textsubscript{$\pm$.0} \\
ARC-Easy & 57.3\textsubscript{$\pm$.8} & 50.3\textsubscript{$\pm$.5} & 54.9\textsubscript{$\pm$.6} & \textbf{56.1}\textsubscript{$\pm$.4} & 55.8\textsubscript{$\pm$.8} \\
ARC-Challenge & 29.2\textsubscript{$\pm$.3} & 26.1\textsubscript{$\pm$.2} & \textbf{28.6}\textsubscript{$\pm$.8} & 27.5\textsubscript{$\pm$.2} & 28.4\textsubscript{$\pm$1.4} \\ \addlinespace[3pt]
SVAMP & 6.0\textsubscript{$\pm$1.5} & \textbf{7.3}\textsubscript{$\pm$.3} & 5.6\textsubscript{$\pm$.4} & 6.4\textsubscript{$\pm$1.3} & 6.6\textsubscript{$\pm$.2} \\ \addlinespace[3pt]
COPA & 67.0\textsubscript{$\pm$.0} & 62.3\textsubscript{$\pm$4.0} & \textbf{67.7}\textsubscript{$\pm$1.2} & 67.0\textsubscript{$\pm$1.7} & 67.3\textsubscript{$\pm$2.5} \\
PIQA & 69.4\textsubscript{$\pm$.5} & 66.2\textsubscript{$\pm$.6} & 68.2\textsubscript{$\pm$.5} & 68.7\textsubscript{$\pm$.3} & \textbf{69.0}\textsubscript{$\pm$.7} \\ \addlinespace[3pt]
HellaSwag & 46.5\textsubscript{$\pm$.3} & 37.3\textsubscript{$\pm$.2} & 43.1\textsubscript{$\pm$.2} & 44.7\textsubscript{$\pm$.1} & \textbf{45.5}\textsubscript{$\pm$.2} \\
WinoGrande & 54.0\textsubscript{$\pm$1.8} & 51.8\textsubscript{$\pm$.2} & 52.7\textsubscript{$\pm$1.1} & \textbf{53.9}\textsubscript{$\pm$.7} & \textbf{53.9}\textsubscript{$\pm$1.5} \\
LAMBADA & 43.5\textsubscript{$\pm$.6} & 33.7\textsubscript{$\pm$.8} & 41.4\textsubscript{$\pm$.7} & 42.4\textsubscript{$\pm$.9} & \textbf{43.4}\textsubscript{$\pm$1.5} \\ \addlinespace[3pt]
BoolQ & 50.8\textsubscript{$\pm$1.0} & 54.1\textsubscript{$\pm$.6} & 52.5\textsubscript{$\pm$4.8} & \textbf{58.7}\textsubscript{$\pm$.3} & 57.3\textsubscript{$\pm$2.4} \\
LSAT-RC & 26.1\textsubscript{$\pm$2.7} & 23.1\textsubscript{$\pm$1.1} & 24.0\textsubscript{$\pm$2.4} & 21.5\textsubscript{$\pm$3.2} & \textbf{25.0}\textsubscript{$\pm$3.7} \\
LSAT-LR & 23.8\textsubscript{$\pm$3.4} & \textbf{28.0}\textsubscript{$\pm$.9} & 26.3\textsubscript{$\pm$1.5} & 25.3\textsubscript{$\pm$1.7} & 27.3\textsubscript{$\pm$2.1} \\
SAT-EN & 26.9\textsubscript{$\pm$2.0} & 24.3\textsubscript{$\pm$2.1} & 26.7\textsubscript{$\pm$3.4} & \textbf{28.0}\textsubscript{$\pm$1.7} & 26.1\textsubscript{$\pm$3.5} \\
\midrule
Average & 40.4\textsubscript{$\pm$.4} & 37.7\textsubscript{$\pm$.1} & 39.7\textsubscript{$\pm$.9} & 40.3\textsubscript{$\pm$.3} & \textbf{40.8}\textsubscript{$\pm$.7} \\
\bottomrule
\end{tabular}
\end{table}

\begin{table}[tbp]
\caption{\textbf{\boldmath Per-task results at the $375$M--$906$M scale, Tier~2 (Medium, about $486$M parameters).} Sub-network width $960$ for Standalone, sampled standalone and the fully elastic model, FFN hidden width $2176$ at $D_{\mathrm{model}}=1024$ for MatFormer, and $1280$ for \Ours{}. Mean over three runs, with the standard deviation as a subscript.}
\label{tab:main-scale2-tier2}
\centering
\setlength{\tabcolsep}{3pt}
\begin{tabular}{@{}lccccc@{}}
\toprule
\multicolumn{1}{@{}l@{}}{\textbf{Tier~2\,$\cdot$\,Medium}} & Standalone$^\dagger$ & Sampled standalone & MatFormer & Fully elastic & \textbf{\Ours} \\
\midrule
Params (M) & 487.7 & 487.7 & 485.1 & 487.7 & 481.6 \\
Train loss $\downarrow$ & 2.7862\textsubscript{$\pm$.0010} & 3.0000\textsubscript{$\pm$.0032} & 2.8283\textsubscript{$\pm$.0015} & \textbf{2.8172}\textsubscript{$\pm$.0013} & 2.8222\textsubscript{$\pm$.0011} \\
Val.\ loss $\downarrow$ & 2.9632\textsubscript{$\pm$.0010} & 3.1750\textsubscript{$\pm$.0039} & 3.0050\textsubscript{$\pm$.0014} & \textbf{2.9948}\textsubscript{$\pm$.0015} & 2.9998\textsubscript{$\pm$.0012} \\
\midrule
\multicolumn{6}{@{}l}{\small\emph{5-shot accuracy (\%) $\uparrow$}} \\
MMLU & 25.2\textsubscript{$\pm$.5} & \textbf{25.2}\textsubscript{$\pm$.7} & 24.7\textsubscript{$\pm$.3} & \textbf{25.2}\textsubscript{$\pm$.1} & 24.8\textsubscript{$\pm$.5} \\
ARC-Easy & 59.6\textsubscript{$\pm$1.1} & 50.8\textsubscript{$\pm$.4} & 58.0\textsubscript{$\pm$.2} & \textbf{58.2}\textsubscript{$\pm$.9} & 57.4\textsubscript{$\pm$1.1} \\
ARC-Challenge & 30.5\textsubscript{$\pm$.7} & 26.5\textsubscript{$\pm$1.0} & \textbf{29.4}\textsubscript{$\pm$.1} & \textbf{29.4}\textsubscript{$\pm$.3} & 28.8\textsubscript{$\pm$.7} \\ \addlinespace[3pt]
SVAMP & 6.7\textsubscript{$\pm$.6} & \textbf{6.4}\textsubscript{$\pm$1.5} & 6.2\textsubscript{$\pm$1.3} & 6.0\textsubscript{$\pm$1.0} & 5.4\textsubscript{$\pm$1.1} \\ \addlinespace[3pt]
COPA & 66.3\textsubscript{$\pm$2.1} & 62.0\textsubscript{$\pm$1.0} & 65.0\textsubscript{$\pm$.0} & \textbf{68.3}\textsubscript{$\pm$.6} & 68.0\textsubscript{$\pm$3.6} \\
PIQA & 70.5\textsubscript{$\pm$.6} & 66.4\textsubscript{$\pm$.4} & \textbf{70.0}\textsubscript{$\pm$.3} & \textbf{70.0}\textsubscript{$\pm$.3} & \textbf{70.0}\textsubscript{$\pm$.3} \\ \addlinespace[3pt]
HellaSwag & 49.2\textsubscript{$\pm$.1} & 38.5\textsubscript{$\pm$.3} & 46.4\textsubscript{$\pm$.2} & \textbf{47.2}\textsubscript{$\pm$.2} & 47.1\textsubscript{$\pm$.3} \\
WinoGrande & 55.7\textsubscript{$\pm$.5} & 51.0\textsubscript{$\pm$1.8} & 53.2\textsubscript{$\pm$.8} & 54.4\textsubscript{$\pm$1.0} & \textbf{54.6}\textsubscript{$\pm$1.5} \\
LAMBADA & 46.1\textsubscript{$\pm$1.1} & 34.7\textsubscript{$\pm$.1} & 44.1\textsubscript{$\pm$.1} & \textbf{45.1}\textsubscript{$\pm$1.1} & 44.9\textsubscript{$\pm$1.4} \\ \addlinespace[3pt]
BoolQ & 54.4\textsubscript{$\pm$.1} & 52.9\textsubscript{$\pm$1.0} & 53.4\textsubscript{$\pm$4.4} & 55.8\textsubscript{$\pm$1.1} & \textbf{57.0}\textsubscript{$\pm$.7} \\
LSAT-RC & 27.7\textsubscript{$\pm$1.8} & 25.2\textsubscript{$\pm$3.2} & 24.3\textsubscript{$\pm$2.6} & 21.5\textsubscript{$\pm$2.4} & \textbf{26.4}\textsubscript{$\pm$2.8} \\
LSAT-LR & 26.9\textsubscript{$\pm$2.0} & 25.2\textsubscript{$\pm$2.0} & 24.4\textsubscript{$\pm$2.6} & 24.3\textsubscript{$\pm$.9} & \textbf{26.7}\textsubscript{$\pm$2.8} \\
SAT-EN & 24.9\textsubscript{$\pm$1.0} & 25.9\textsubscript{$\pm$1.0} & 24.6\textsubscript{$\pm$.3} & \textbf{27.2}\textsubscript{$\pm$2.6} & 26.5\textsubscript{$\pm$2.2} \\
\midrule
Average & 41.8\textsubscript{$\pm$.4} & 37.8\textsubscript{$\pm$.3} & 40.3\textsubscript{$\pm$.3} & 41.0\textsubscript{$\pm$.3} & \textbf{41.4}\textsubscript{$\pm$.7} \\
\bottomrule
\end{tabular}
\end{table}

\begin{table}[tbp]
\caption{\textbf{\boldmath Per-task results at the $375$M--$906$M scale, Tier~3 (High, about $611$M parameters).} Sub-network width $1088$ for Standalone, sampled standalone and the fully elastic model, FFN hidden width $3456$ at $D_{\mathrm{model}}=1024$ for MatFormer, and $1536$ for \Ours{}. Mean over three runs, with the standard deviation as a subscript.}
\label{tab:main-scale2-tier3}
\centering
\setlength{\tabcolsep}{3pt}
\begin{tabular}{@{}lccccc@{}}
\toprule
\multicolumn{1}{@{}l@{}}{\textbf{Tier~3\,$\cdot$\,High}} & Standalone$^\dagger$ & Sampled standalone & MatFormer & Fully elastic & \textbf{\Ours} \\
\midrule
Params (M) & 611.8 & 611.8 & 610.9 & 611.8 & 608.1 \\
Train loss $\downarrow$ & 2.7496\textsubscript{$\pm$.0018} & 2.9719\textsubscript{$\pm$.0037} & 2.8035\textsubscript{$\pm$.0016} & 2.7933\textsubscript{$\pm$.0011} & \textbf{2.7901}\textsubscript{$\pm$.0013} \\
Val.\ loss $\downarrow$ & 2.9262\textsubscript{$\pm$.0017} & 3.1414\textsubscript{$\pm$.0040} & 2.9750\textsubscript{$\pm$.0015} & \textbf{2.9662}\textsubscript{$\pm$.0016} & 2.9670\textsubscript{$\pm$.0016} \\
\midrule
\multicolumn{6}{@{}l}{\small\emph{5-shot accuracy (\%) $\uparrow$}} \\
MMLU & 24.1\textsubscript{$\pm$.5} & 24.6\textsubscript{$\pm$.9} & 24.9\textsubscript{$\pm$.9} & \textbf{25.0}\textsubscript{$\pm$.8} & 24.8\textsubscript{$\pm$.3} \\
ARC-Easy & 60.6\textsubscript{$\pm$.5} & 52.4\textsubscript{$\pm$.6} & 58.8\textsubscript{$\pm$.3} & \textbf{59.8}\textsubscript{$\pm$.4} & 58.5\textsubscript{$\pm$.6} \\
ARC-Challenge & 30.6\textsubscript{$\pm$.6} & 26.7\textsubscript{$\pm$.5} & 30.0\textsubscript{$\pm$.2} & \textbf{30.3}\textsubscript{$\pm$.9} & 29.1\textsubscript{$\pm$.3} \\ \addlinespace[3pt]
SVAMP & 5.0\textsubscript{$\pm$.9} & 6.0\textsubscript{$\pm$.6} & 5.9\textsubscript{$\pm$.7} & 5.9\textsubscript{$\pm$.8} & \textbf{6.6}\textsubscript{$\pm$1.3} \\ \addlinespace[3pt]
COPA & 68.3\textsubscript{$\pm$1.5} & 61.3\textsubscript{$\pm$4.0} & 66.3\textsubscript{$\pm$.6} & 68.0\textsubscript{$\pm$1.7} & \textbf{68.3}\textsubscript{$\pm$2.5} \\
PIQA & 71.2\textsubscript{$\pm$.2} & 66.7\textsubscript{$\pm$.3} & 70.6\textsubscript{$\pm$.6} & 70.2\textsubscript{$\pm$.4} & \textbf{70.9}\textsubscript{$\pm$.4} \\ \addlinespace[3pt]
HellaSwag & 51.1\textsubscript{$\pm$.2} & 39.8\textsubscript{$\pm$.1} & 48.1\textsubscript{$\pm$.2} & \textbf{48.7}\textsubscript{$\pm$.2} & 48.4\textsubscript{$\pm$.1} \\
WinoGrande & 55.8\textsubscript{$\pm$1.3} & 52.6\textsubscript{$\pm$.7} & 53.1\textsubscript{$\pm$.4} & 54.0\textsubscript{$\pm$1.6} & \textbf{54.4}\textsubscript{$\pm$1.2} \\
LAMBADA & 48.3\textsubscript{$\pm$.4} & 36.0\textsubscript{$\pm$1.2} & 45.2\textsubscript{$\pm$.4} & \textbf{46.8}\textsubscript{$\pm$1.2} & 45.5\textsubscript{$\pm$.9} \\ \addlinespace[3pt]
BoolQ & 56.5\textsubscript{$\pm$4.4} & 53.0\textsubscript{$\pm$2.2} & 52.1\textsubscript{$\pm$4.9} & 55.6\textsubscript{$\pm$3.3} & \textbf{58.2}\textsubscript{$\pm$1.9} \\
LSAT-RC & 28.0\textsubscript{$\pm$2.3} & 24.8\textsubscript{$\pm$3.8} & 23.8\textsubscript{$\pm$1.5} & 24.4\textsubscript{$\pm$2.9} & \textbf{26.6}\textsubscript{$\pm$.8} \\
LSAT-LR & 25.1\textsubscript{$\pm$1.2} & 24.6\textsubscript{$\pm$2.3} & 24.2\textsubscript{$\pm$2.2} & 25.0\textsubscript{$\pm$1.5} & \textbf{26.3}\textsubscript{$\pm$2.2} \\
SAT-EN & 27.8\textsubscript{$\pm$2.9} & 25.2\textsubscript{$\pm$1.3} & 22.3\textsubscript{$\pm$1.0} & 27.7\textsubscript{$\pm$.5} & \textbf{28.2}\textsubscript{$\pm$2.2} \\
\midrule
Average & 42.5\textsubscript{$\pm$.8} & 38.0\textsubscript{$\pm$.4} & 40.4\textsubscript{$\pm$.1} & 41.6\textsubscript{$\pm$.4} & \textbf{42.0}\textsubscript{$\pm$.4} \\
\bottomrule
\end{tabular}
\end{table}

\begin{table}[tbp]
\caption{\textbf{\boldmath Per-task results at the $375$M--$906$M scale, Tier~4 (Extra, about $750$M parameters).} Sub-network width $1216$ for Standalone, sampled standalone and the fully elastic model, FFN hidden width $4864$ at $D_{\mathrm{model}}=1024$ for MatFormer, and $1792$ for \Ours{}. Mean over three runs, with the standard deviation as a subscript.}
\label{tab:main-scale2-tier4}
\centering
\setlength{\tabcolsep}{3pt}
\begin{tabular}{@{}lccccc@{}}
\toprule
\multicolumn{1}{@{}l@{}}{\textbf{Tier~4\,$\cdot$\,Extra}} & Standalone$^\dagger$ & Sampled standalone & MatFormer & Fully elastic & \textbf{\Ours} \\
\midrule
Params (M) & 749.8 & 749.8 & 749.3 & 749.8 & 751.4 \\
Train loss $\downarrow$ & 2.7201\textsubscript{$\pm$.0016} & 2.9403\textsubscript{$\pm$.0057} & 2.7857\textsubscript{$\pm$.0015} & 2.7735\textsubscript{$\pm$.0011} & \textbf{2.7671}\textsubscript{$\pm$.0010} \\
Val.\ loss $\downarrow$ & 2.8972\textsubscript{$\pm$.0016} & 3.1146\textsubscript{$\pm$.0057} & 2.9628\textsubscript{$\pm$.0016} & 2.9520\textsubscript{$\pm$.0014} & \textbf{2.9435}\textsubscript{$\pm$.0016} \\
\midrule
\multicolumn{6}{@{}l}{\small\emph{5-shot accuracy (\%) $\uparrow$}} \\
MMLU & 25.1\textsubscript{$\pm$.5} & 23.9\textsubscript{$\pm$.1} & \textbf{24.6}\textsubscript{$\pm$.5} & 24.4\textsubscript{$\pm$.5} & 24.5\textsubscript{$\pm$.3} \\
ARC-Easy & 62.4\textsubscript{$\pm$1.1} & 52.2\textsubscript{$\pm$.6} & 59.9\textsubscript{$\pm$.4} & \textbf{60.4}\textsubscript{$\pm$.1} & 58.9\textsubscript{$\pm$.9} \\
ARC-Challenge & 32.1\textsubscript{$\pm$.4} & 26.2\textsubscript{$\pm$.3} & \textbf{30.5}\textsubscript{$\pm$.5} & \textbf{30.5}\textsubscript{$\pm$1.1} & 29.6\textsubscript{$\pm$.3} \\ \addlinespace[3pt]
SVAMP & 5.0\textsubscript{$\pm$.7} & \textbf{6.3}\textsubscript{$\pm$1.5} & 6.0\textsubscript{$\pm$1.5} & 6.2\textsubscript{$\pm$.8} & 5.8\textsubscript{$\pm$.8} \\ \addlinespace[3pt]
COPA & 68.7\textsubscript{$\pm$1.2} & 64.7\textsubscript{$\pm$3.8} & 66.0\textsubscript{$\pm$.0} & 68.3\textsubscript{$\pm$1.2} & \textbf{70.0}\textsubscript{$\pm$4.0} \\
PIQA & 71.9\textsubscript{$\pm$.3} & 67.7\textsubscript{$\pm$.7} & 70.7\textsubscript{$\pm$.1} & 70.3\textsubscript{$\pm$.1} & \textbf{71.5}\textsubscript{$\pm$.7} \\ \addlinespace[3pt]
HellaSwag & 53.1\textsubscript{$\pm$.3} & 41.2\textsubscript{$\pm$.3} & 48.9\textsubscript{$\pm$.1} & \textbf{49.4}\textsubscript{$\pm$.4} & \textbf{49.4}\textsubscript{$\pm$.1} \\
WinoGrande & 58.0\textsubscript{$\pm$2.0} & 52.8\textsubscript{$\pm$.7} & 54.1\textsubscript{$\pm$.8} & 54.2\textsubscript{$\pm$1.6} & \textbf{55.0}\textsubscript{$\pm$.8} \\
LAMBADA & 49.5\textsubscript{$\pm$.5} & 39.5\textsubscript{$\pm$.3} & 45.7\textsubscript{$\pm$.6} & \textbf{47.0}\textsubscript{$\pm$1.2} & 46.2\textsubscript{$\pm$1.1} \\ \addlinespace[3pt]
BoolQ & 59.0\textsubscript{$\pm$2.4} & 53.7\textsubscript{$\pm$2.1} & 54.1\textsubscript{$\pm$4.4} & 56.8\textsubscript{$\pm$5.5} & \textbf{57.9}\textsubscript{$\pm$.4} \\
LSAT-RC & 25.2\textsubscript{$\pm$1.8} & 22.9\textsubscript{$\pm$2.2} & 24.0\textsubscript{$\pm$1.7} & 23.5\textsubscript{$\pm$2.4} & \textbf{27.6}\textsubscript{$\pm$1.6} \\
LSAT-LR & 24.5\textsubscript{$\pm$1.1} & 25.8\textsubscript{$\pm$1.2} & 24.7\textsubscript{$\pm$1.3} & 25.6\textsubscript{$\pm$1.7} & \textbf{28.0}\textsubscript{$\pm$1.8} \\
SAT-EN & 27.7\textsubscript{$\pm$.0} & 26.2\textsubscript{$\pm$3.9} & 26.2\textsubscript{$\pm$1.9} & 29.1\textsubscript{$\pm$1.0} & \textbf{29.4}\textsubscript{$\pm$1.4} \\
\midrule
Average & 43.2\textsubscript{$\pm$.5} & 38.7\textsubscript{$\pm$.5} & 41.2\textsubscript{$\pm$.3} & 42.0\textsubscript{$\pm$.5} & \textbf{42.6}\textsubscript{$\pm$.3} \\
\bottomrule
\end{tabular}
\end{table}

\begin{table}[tbp]
\caption{\textbf{\boldmath Per-task results at the $375$M--$906$M scale, Tier~5 (Max, about $906$M parameters).} Sub-network width $1344$ for Standalone, sampled standalone and the fully elastic model, FFN hidden width $6528$ at $D_{\mathrm{model}}=1024$ for MatFormer, and $2048$ for \Ours{}. Mean over three runs, with the standard deviation as a subscript.}
\label{tab:main-scale2-tier5}
\centering
\setlength{\tabcolsep}{3pt}
\begin{tabular}{@{}lccccc@{}}
\toprule
\multicolumn{1}{@{}l@{}}{\textbf{Tier~5\,$\cdot$\,Max}} & Standalone$^\dagger$ & Sampled standalone & MatFormer & Fully elastic & \textbf{\Ours} \\
\midrule
Params (M) & 901.6 & 901.6 & 912.9 & 901.6 & 911.5 \\
Train loss $\downarrow$ & 2.6922\textsubscript{$\pm$.0012} & 2.9196\textsubscript{$\pm$.0026} & 2.7871\textsubscript{$\pm$.0015} & 2.7710\textsubscript{$\pm$.0009} & \textbf{2.7488}\textsubscript{$\pm$.0008} \\
Val.\ loss $\downarrow$ & 2.8696\textsubscript{$\pm$.0010} & 3.0931\textsubscript{$\pm$.0032} & 2.9631\textsubscript{$\pm$.0016} & 2.9481\textsubscript{$\pm$.0015} & \textbf{2.9250}\textsubscript{$\pm$.0012} \\
\midrule
\multicolumn{6}{@{}l}{\small\emph{5-shot accuracy (\%) $\uparrow$}} \\
MMLU & 24.9\textsubscript{$\pm$.5} & 24.8\textsubscript{$\pm$.2} & 24.8\textsubscript{$\pm$.0} & \textbf{24.9}\textsubscript{$\pm$.5} & 24.8\textsubscript{$\pm$.1} \\
ARC-Easy & 63.4\textsubscript{$\pm$.8} & 54.2\textsubscript{$\pm$1.1} & 60.0\textsubscript{$\pm$.7} & \textbf{60.5}\textsubscript{$\pm$.1} & 59.9\textsubscript{$\pm$.4} \\
ARC-Challenge & 33.0\textsubscript{$\pm$.6} & 26.8\textsubscript{$\pm$.8} & 30.6\textsubscript{$\pm$.6} & 30.4\textsubscript{$\pm$.7} & \textbf{31.0}\textsubscript{$\pm$.8} \\ \addlinespace[3pt]
SVAMP & 5.6\textsubscript{$\pm$.7} & 4.9\textsubscript{$\pm$.8} & 5.8\textsubscript{$\pm$.8} & \textbf{6.4}\textsubscript{$\pm$.5} & 6.0\textsubscript{$\pm$1.3} \\ \addlinespace[3pt]
COPA & 70.0\textsubscript{$\pm$2.0} & 61.7\textsubscript{$\pm$2.1} & 66.3\textsubscript{$\pm$.6} & 67.7\textsubscript{$\pm$1.2} & \textbf{72.0}\textsubscript{$\pm$2.6} \\
PIQA & 72.2\textsubscript{$\pm$1.0} & 67.8\textsubscript{$\pm$.8} & 70.9\textsubscript{$\pm$.1} & 70.4\textsubscript{$\pm$.3} & \textbf{71.9}\textsubscript{$\pm$.6} \\ \addlinespace[3pt]
HellaSwag & 54.6\textsubscript{$\pm$.2} & 41.9\textsubscript{$\pm$.1} & 49.0\textsubscript{$\pm$.1} & 49.6\textsubscript{$\pm$.5} & \textbf{50.2}\textsubscript{$\pm$.2} \\
WinoGrande & 57.6\textsubscript{$\pm$1.4} & 53.2\textsubscript{$\pm$.7} & 54.2\textsubscript{$\pm$1.0} & 54.3\textsubscript{$\pm$.9} & \textbf{54.9}\textsubscript{$\pm$.9} \\
LAMBADA & 51.4\textsubscript{$\pm$.8} & 39.0\textsubscript{$\pm$.4} & 45.4\textsubscript{$\pm$.8} & \textbf{47.0}\textsubscript{$\pm$.8} & 46.3\textsubscript{$\pm$1.2} \\ \addlinespace[3pt]
BoolQ & 54.8\textsubscript{$\pm$5.4} & 57.2\textsubscript{$\pm$1.7} & 55.7\textsubscript{$\pm$4.1} & 57.9\textsubscript{$\pm$4.6} & \textbf{59.8}\textsubscript{$\pm$1.7} \\
LSAT-RC & 24.5\textsubscript{$\pm$2.6} & 24.6\textsubscript{$\pm$.4} & 24.6\textsubscript{$\pm$1.0} & 24.1\textsubscript{$\pm$.2} & \textbf{27.4}\textsubscript{$\pm$2.1} \\
LSAT-LR & 24.1\textsubscript{$\pm$.7} & 26.3\textsubscript{$\pm$3.2} & 25.0\textsubscript{$\pm$1.6} & 25.2\textsubscript{$\pm$.3} & \textbf{27.1}\textsubscript{$\pm$.5} \\
SAT-EN & 24.8\textsubscript{$\pm$3.8} & 24.8\textsubscript{$\pm$1.8} & 24.4\textsubscript{$\pm$1.2} & \textbf{28.0}\textsubscript{$\pm$.6} & 27.5\textsubscript{$\pm$1.7} \\
\midrule
Average & 43.1\textsubscript{$\pm$.8} & 39.0\textsubscript{$\pm$.5} & 41.3\textsubscript{$\pm$.6} & 42.0\textsubscript{$\pm$.4} & \textbf{43.0}\textsubscript{$\pm$.3} \\
\bottomrule
\end{tabular}
\end{table}

\begin{figure}[p]
\centering
\includegraphics[width=\linewidth]{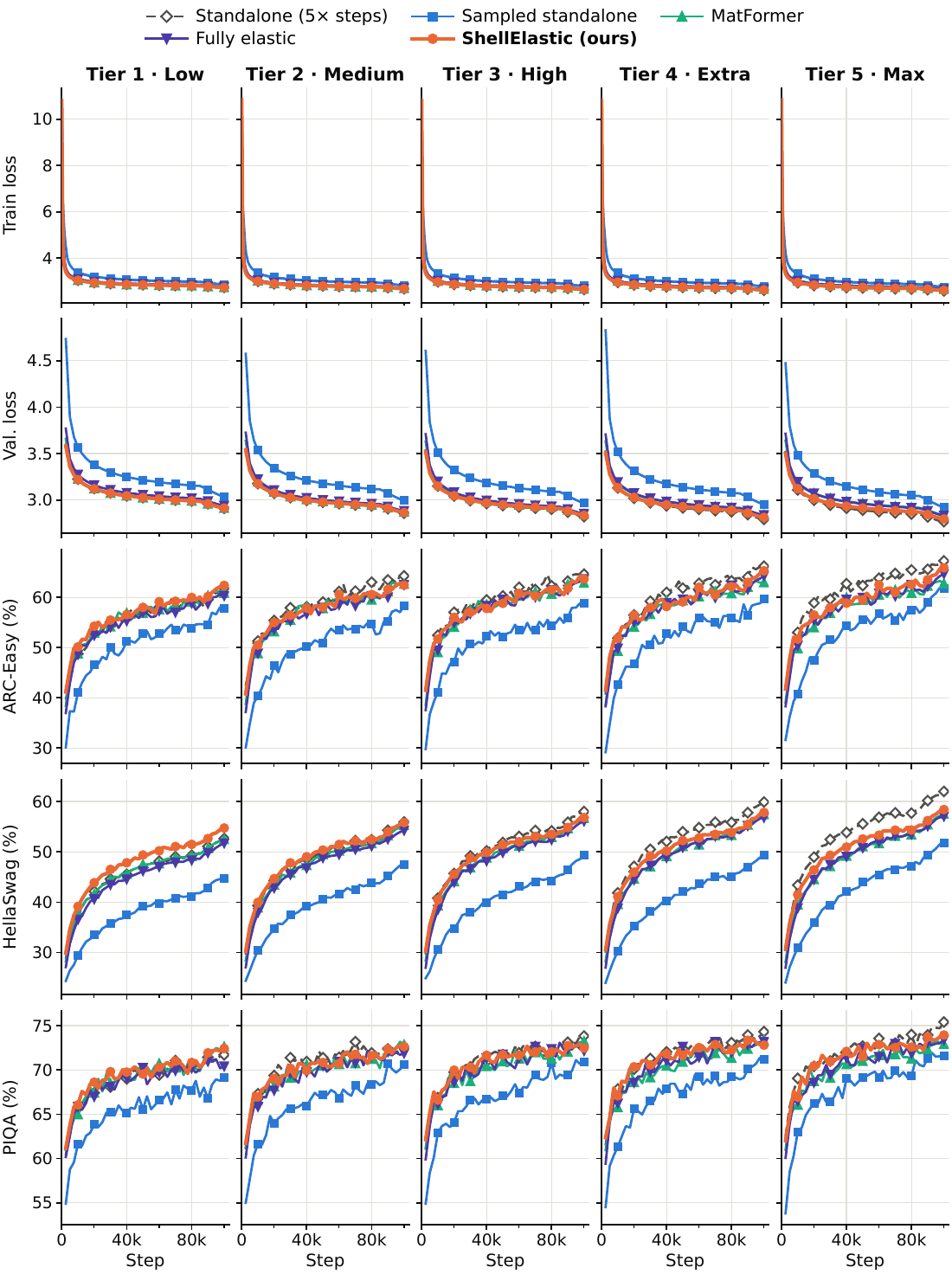}
\caption{\textbf{\boldmath Training curves at the $675$M--$1.69$B scale} (Table~\ref{tab:main-scales}). Columns are the five tiers, from Tier~1 (Low, about $675$M parameters) to Tier~5 (Max, about $1.69$B); rows are the train loss, the validation loss and 5-shot accuracy on ARC-Easy, HellaSwag and PIQA. Each line is a single run; conventions in Appendix~\ref{app:main-full}.}
\label{fig:main-scale3}
\end{figure}

\begin{table}[tbp]
\caption{\textbf{\boldmath Per-task results at the $675$M--$1.69$B scale, Tier~1 (Low, about $675$M parameters).} Sub-network width $832$ for Standalone, sampled standalone and the fully elastic model, FFN hidden width $4096$ at $D_{\mathrm{model}}=1024$ for MatFormer, and $1024$ for \Ours{}. Single runs.}
\label{tab:main-scale3-tier1}
\centering
\setlength{\tabcolsep}{3pt}
\begin{tabular}{@{}lccccc@{}}
\toprule
\multicolumn{1}{@{}l@{}}{\textbf{Tier~1\,$\cdot$\,Low}} & Standalone$^\dagger$ & Sampled standalone & MatFormer & Fully elastic & \textbf{\Ours} \\
\midrule
Params (M) & 676.6 & 676.6 & 673.8 & 676.6 & 673.8 \\
Train loss $\downarrow$ & 2.7335 & 2.8643 & 2.7365 & 2.7638 & \textbf{2.7347} \\
Val.\ loss $\downarrow$ & 2.9118 & 3.0370 & 2.9140 & 2.9403 & \textbf{2.9135} \\
\midrule
\multicolumn{6}{@{}l}{\small\emph{5-shot accuracy (\%) $\uparrow$}} \\
MMLU & 25.2 & 23.6 & \textbf{24.1} & 23.3 & \textbf{24.1} \\
ARC-Easy & 61.8 & 57.7 & 60.9 & 60.3 & \textbf{62.4} \\
ARC-Challenge & 31.4 & 29.6 & \textbf{32.3} & 30.2 & 32.2 \\ \addlinespace[3pt]
SVAMP & 6.3 & 6.3 & \textbf{6.7} & 4.7 & 4.7 \\ \addlinespace[3pt]
COPA & 70.0 & 64.0 & 67.0 & \textbf{71.0} & 70.0 \\
PIQA & 71.7 & 69.2 & \textbf{72.6} & 70.4 & 72.4 \\ \addlinespace[3pt]
HellaSwag & 52.6 & 44.8 & 52.5 & 51.7 & \textbf{54.7} \\
WinoGrande & 57.1 & 50.8 & 56.7 & 56.9 & \textbf{58.2} \\
LAMBADA & 47.1 & 43.0 & 49.7 & 44.0 & \textbf{51.6} \\ \addlinespace[3pt]
BoolQ & 57.3 & 52.4 & 60.6 & 55.7 & \textbf{61.9} \\
LSAT-RC & 29.1 & \textbf{29.5} & 23.9 & 26.9 & 26.1 \\
LSAT-LR & 24.9 & 28.4 & 27.5 & \textbf{29.0} & 24.9 \\
SAT-EN & 29.1 & 23.8 & 22.3 & 22.8 & \textbf{25.2} \\
\midrule
Average & 43.4 & 40.2 & 42.8 & 42.1 & \textbf{43.7} \\
\bottomrule
\end{tabular}
\end{table}

\begin{table}[tbp]
\caption{\textbf{\boldmath Per-task results at the $675$M--$1.69$B scale, Tier~2 (Medium, about $882$M parameters).} Sub-network width $960$ for Standalone, sampled standalone and the fully elastic model, FFN hidden width $6144$ at $D_{\mathrm{model}}=1024$ for MatFormer, and $1280$ for \Ours{}. Single runs.}
\label{tab:main-scale3-tier2}
\centering
\setlength{\tabcolsep}{3pt}
\begin{tabular}{@{}lccccc@{}}
\toprule
\multicolumn{1}{@{}l@{}}{\textbf{Tier~2\,$\cdot$\,Medium}} & Standalone$^\dagger$ & Sampled standalone & MatFormer & Fully elastic & \textbf{\Ours} \\
\midrule
Params (M) & 885.8 & 885.8 & 875.2 & 885.8 & 878.0 \\
Train loss $\downarrow$ & 2.6791 & 2.8253 & 2.6938 & 2.7067 & \textbf{2.6865} \\
Val.\ loss $\downarrow$ & 2.8587 & 2.9980 & 2.8720 & 2.8840 & \textbf{2.8657} \\
\midrule
\multicolumn{6}{@{}l}{\small\emph{5-shot accuracy (\%) $\uparrow$}} \\
MMLU & 23.8 & 24.0 & \textbf{24.6} & 24.1 & 23.6 \\
ARC-Easy & 64.3 & 58.3 & \textbf{62.7} & 62.5 & 62.4 \\
ARC-Challenge & 33.8 & 28.2 & \textbf{33.4} & 31.6 & 32.9 \\ \addlinespace[3pt]
SVAMP & 6.0 & \textbf{7.0} & \textbf{7.0} & 5.0 & 6.0 \\ \addlinespace[3pt]
COPA & 73.0 & 63.0 & 69.0 & \textbf{73.0} & 65.0 \\
PIQA & 72.6 & 70.6 & \textbf{72.9} & 71.9 & 72.6 \\ \addlinespace[3pt]
HellaSwag & 55.9 & 47.4 & 55.3 & 54.2 & \textbf{55.8} \\
WinoGrande & 59.0 & 52.7 & 57.5 & 58.2 & \textbf{58.8} \\
LAMBADA & 52.0 & 43.8 & 51.2 & 48.2 & \textbf{52.4} \\ \addlinespace[3pt]
BoolQ & 55.0 & 59.0 & 62.6 & 55.8 & \textbf{62.8} \\
LSAT-RC & 26.1 & 23.1 & 20.9 & \textbf{27.6} & 23.5 \\
LSAT-LR & 25.3 & \textbf{27.8} & 25.5 & 27.5 & 27.3 \\
SAT-EN & 23.8 & \textbf{24.8} & 20.9 & \textbf{24.8} & \textbf{24.8} \\
\midrule
Average & 43.9 & 40.8 & 43.3 & 43.4 & \textbf{43.7} \\
\bottomrule
\end{tabular}
\end{table}

\begin{table}[tbp]
\caption{\textbf{\boldmath Per-task results at the $675$M--$1.69$B scale, Tier~3 (High, about $1.12$B parameters).} Sub-network width $1088$ for Standalone, sampled standalone and the fully elastic model, FFN hidden width $8704$ at $D_{\mathrm{model}}=1024$ for MatFormer, and $1536$ for \Ours{}. Single runs.}
\label{tab:main-scale3-tier3}
\centering
\setlength{\tabcolsep}{3pt}
\begin{tabular}{@{}lccccc@{}}
\toprule
\multicolumn{1}{@{}l@{}}{\textbf{Tier~3\,$\cdot$\,High}} & Standalone$^\dagger$ & Sampled standalone & MatFormer & Fully elastic & \textbf{\Ours} \\
\midrule
Params (M) & 1123.2 & 1123.2 & 1126.8 & 1123.2 & 1117.7 \\
Train loss $\downarrow$ & 2.6424 & 2.7994 & 2.6730 & 2.6788 & \textbf{2.6550} \\
Val.\ loss $\downarrow$ & 2.8212 & 2.9696 & 2.8487 & 2.8539 & \textbf{2.8351} \\
\midrule
\multicolumn{6}{@{}l}{\small\emph{5-shot accuracy (\%) $\uparrow$}} \\
MMLU & 25.3 & 22.9 & \textbf{24.3} & \textbf{24.3} & 24.2 \\
ARC-Easy & 64.6 & 58.8 & 62.9 & \textbf{63.6} & \textbf{63.6} \\
ARC-Challenge & 34.3 & 30.1 & 33.4 & 32.8 & \textbf{33.8} \\ \addlinespace[3pt]
SVAMP & 6.3 & 5.3 & \textbf{7.0} & 5.0 & \textbf{7.0} \\ \addlinespace[3pt]
COPA & 72.0 & 68.0 & 69.0 & \textbf{71.0} & 67.0 \\
PIQA & 73.8 & 70.9 & \textbf{73.2} & 72.1 & 72.5 \\ \addlinespace[3pt]
HellaSwag & 58.0 & 49.3 & 56.7 & 56.1 & \textbf{56.8} \\
WinoGrande & 59.7 & 54.9 & \textbf{59.4} & 58.7 & 58.6 \\
LAMBADA & 52.7 & 45.8 & \textbf{52.7} & 50.2 & 52.4 \\ \addlinespace[3pt]
BoolQ & 61.8 & 54.2 & 61.3 & 57.9 & \textbf{62.1} \\
LSAT-RC & 22.4 & \textbf{26.1} & 23.5 & 25.4 & 22.8 \\
LSAT-LR & 24.5 & 28.0 & 23.3 & \textbf{28.8} & 24.7 \\
SAT-EN & 25.7 & \textbf{25.7} & 18.0 & 25.2 & \textbf{25.7} \\
\midrule
Average & 44.7 & 41.6 & 43.4 & \textbf{43.9} & \textbf{43.9} \\
\bottomrule
\end{tabular}
\end{table}

\begin{table}[tbp]
\caption{\textbf{\boldmath Per-task results at the $675$M--$1.69$B scale, Tier~4 (Extra, about $1.39$B parameters).} Sub-network width $1216$ for Standalone, sampled standalone and the fully elastic model, FFN hidden width $11520$ at $D_{\mathrm{model}}=1024$ for MatFormer, and $1792$ for \Ours{}. Single runs.}
\label{tab:main-scale3-tier4}
\centering
\setlength{\tabcolsep}{3pt}
\begin{tabular}{@{}lccccc@{}}
\toprule
\multicolumn{1}{@{}l@{}}{\textbf{Tier~4\,$\cdot$\,Extra}} & Standalone$^\dagger$ & Sampled standalone & MatFormer & Fully elastic & \textbf{\Ours} \\
\midrule
Params (M) & 1388.5 & 1388.5 & 1403.7 & 1388.5 & 1393.1 \\
Train loss $\downarrow$ & 2.6122 & 2.7838 & 2.6617 & 2.6621 & \textbf{2.6323} \\
Val.\ loss $\downarrow$ & 2.7933 & 2.9557 & 2.8385 & 2.8384 & \textbf{2.8131} \\
\midrule
\multicolumn{6}{@{}l}{\small\emph{5-shot accuracy (\%) $\uparrow$}} \\
MMLU & 23.7 & 24.6 & 24.2 & \textbf{24.8} & 24.2 \\
ARC-Easy & 66.2 & 59.7 & 63.0 & 64.3 & \textbf{65.3} \\
ARC-Challenge & 35.1 & 30.0 & 32.8 & \textbf{33.6} & 32.9 \\ \addlinespace[3pt]
SVAMP & 6.7 & 2.7 & \textbf{8.0} & 3.7 & 6.0 \\ \addlinespace[3pt]
COPA & 73.0 & 64.0 & 72.0 & \textbf{73.0} & 68.0 \\
PIQA & 74.3 & 71.2 & \textbf{73.3} & 73.2 & 72.8 \\ \addlinespace[3pt]
HellaSwag & 59.9 & 49.3 & 57.4 & 56.8 & \textbf{57.8} \\
WinoGrande & 59.7 & 54.2 & \textbf{59.3} & 58.8 & 59.2 \\
LAMBADA & 55.2 & 46.6 & \textbf{53.3} & 51.6 & 52.5 \\ \addlinespace[3pt]
BoolQ & 64.6 & 56.3 & 61.3 & 59.2 & \textbf{62.4} \\
LSAT-RC & 24.6 & 25.0 & 23.1 & \textbf{26.1} & 24.3 \\
LSAT-LR & 28.4 & 22.7 & 25.9 & \textbf{29.2} & 25.1 \\
SAT-EN & 29.1 & \textbf{27.2} & 18.0 & 20.9 & 25.7 \\
\midrule
Average & 46.2 & 41.0 & 44.0 & 44.2 & \textbf{44.3} \\
\bottomrule
\end{tabular}
\end{table}

\begin{table}[tbp]
\caption{\textbf{\boldmath Per-task results at the $675$M--$1.69$B scale, Tier~5 (Max, about $1.69$B parameters).} Sub-network width $1344$ for Standalone, sampled standalone and the fully elastic model, FFN hidden width $14592$ at $D_{\mathrm{model}}=1024$ for MatFormer, and $2048$ for \Ours{}. Single runs.}
\label{tab:main-scale3-tier5}
\centering
\setlength{\tabcolsep}{3pt}
\begin{tabular}{@{}lccccc@{}}
\toprule
\multicolumn{1}{@{}l@{}}{\textbf{Tier~5\,$\cdot$\,Max}} & Standalone$^\dagger$ & Sampled standalone & MatFormer & Fully elastic & \textbf{\Ours} \\
\midrule
Params (M) & 1682.0 & 1682.0 & 1705.6 & 1682.0 & 1704.2 \\
Train loss $\downarrow$ & 2.5859 & 2.7509 & 2.6585 & 2.6549 & \textbf{2.6152} \\
Val.\ loss $\downarrow$ & 2.7674 & 2.9242 & 2.8378 & 2.8334 & \textbf{2.7959} \\
\midrule
\multicolumn{6}{@{}l}{\small\emph{5-shot accuracy (\%) $\uparrow$}} \\
MMLU & 25.2 & \textbf{24.7} & 24.2 & 24.3 & 24.6 \\
ARC-Easy & 67.3 & 61.7 & 62.8 & 64.6 & \textbf{65.9} \\
ARC-Challenge & 36.6 & 32.4 & 32.6 & \textbf{34.1} & 33.9 \\ \addlinespace[3pt]
SVAMP & 5.0 & 4.7 & \textbf{6.3} & 5.3 & 6.0 \\ \addlinespace[3pt]
COPA & 73.0 & 66.0 & \textbf{72.0} & \textbf{72.0} & 66.0 \\
PIQA & 75.4 & 71.5 & 72.9 & 73.4 & \textbf{73.9} \\ \addlinespace[3pt]
HellaSwag & 62.0 & 51.8 & 57.5 & 57.0 & \textbf{58.4} \\
WinoGrande & 60.5 & 55.0 & 59.2 & \textbf{60.2} & 59.2 \\
LAMBADA & 56.9 & 47.6 & \textbf{53.5} & 52.2 & 52.4 \\ \addlinespace[3pt]
BoolQ & 63.0 & 61.4 & 60.6 & 59.3 & \textbf{62.5} \\
LSAT-RC & 25.0 & \textbf{27.2} & 20.1 & \textbf{27.2} & 22.4 \\
LSAT-LR & 22.7 & \textbf{28.6} & 24.1 & 28.0 & 27.8 \\
SAT-EN & 26.7 & \textbf{29.1} & 20.4 & 21.4 & 25.2 \\
\midrule
Average & 46.1 & 43.2 & 43.6 & \textbf{44.6} & 44.5 \\
\bottomrule
\end{tabular}
\end{table}

\clearpage
\section{Additional transfer results}
\label{app:transfer-extra}
\label{app:mutransfer-tiers}

Figures~\ref{fig:mutransfer-matformer}--\ref{fig:mutransfer-shell} repeat the sweeps of Figure~\ref{fig:mutransfer} for each of the five sub-networks of MatFormer, the fully elastic model and \Ours{}.
Lines, stars and $\Delta$ val.\ loss are as in Figure~\ref{fig:mutransfer}, with the loss of each sub-network in place of the mean over the five; the proxy line is repeated in every column, and the five panels of a row share one y-range.
Two observations emerge:
\begin{itemize}
\item \textbf{\muE{} transfers to every sub-network.} Under \muE{}, the optimal $\eta_0$ of every sub-network lies between $10^{-3}$ and $3.3\times10^{-3}$ and the optimal $\sigma_0$ at $3.3\times10^{-3}$ or $10^{-2}$, and the base values of Section~\ref{sec:setup} are within $0.005$ nats of the optimum of every sub-network at the $119$M--$285$M and $375$M--$906$M scales.
\item \textbf{The gap between SP and \muE{} is largest at the narrow widths.} The gap is larger at the smaller sub-networks at both scales, for example $0.022$ nats at Low and $0.007$ at Max for \Ours{} in the $\eta_0$ sweep at the $119$M--$285$M scale, as in Table~\ref{tab:muinternal}, and $0.015$ at Low and $0.003$ at Max at the $375$M--$906$M scale.
\end{itemize}

\begin{figure}[htp]
\centering
\includegraphics[width=\linewidth]{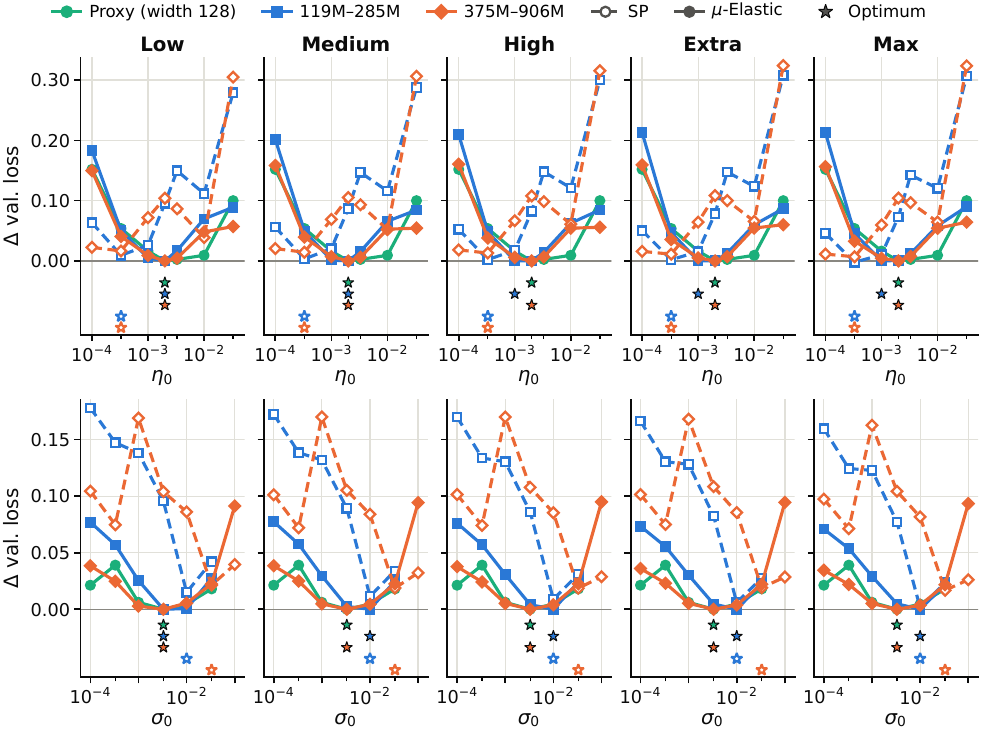}
\caption{\textbf{Per-sub-network sweeps for MatFormer.} The sweeps of Figure~\ref{fig:mutransfer} for each sub-network, from the smallest (Low) to the largest (Max; widths in Table~\ref{tab:models}), with the loss of each sub-network in place of the mean over the five.}
\label{fig:mutransfer-matformer}
\end{figure}

\begin{figure}[htp]
\centering
\includegraphics[width=\linewidth]{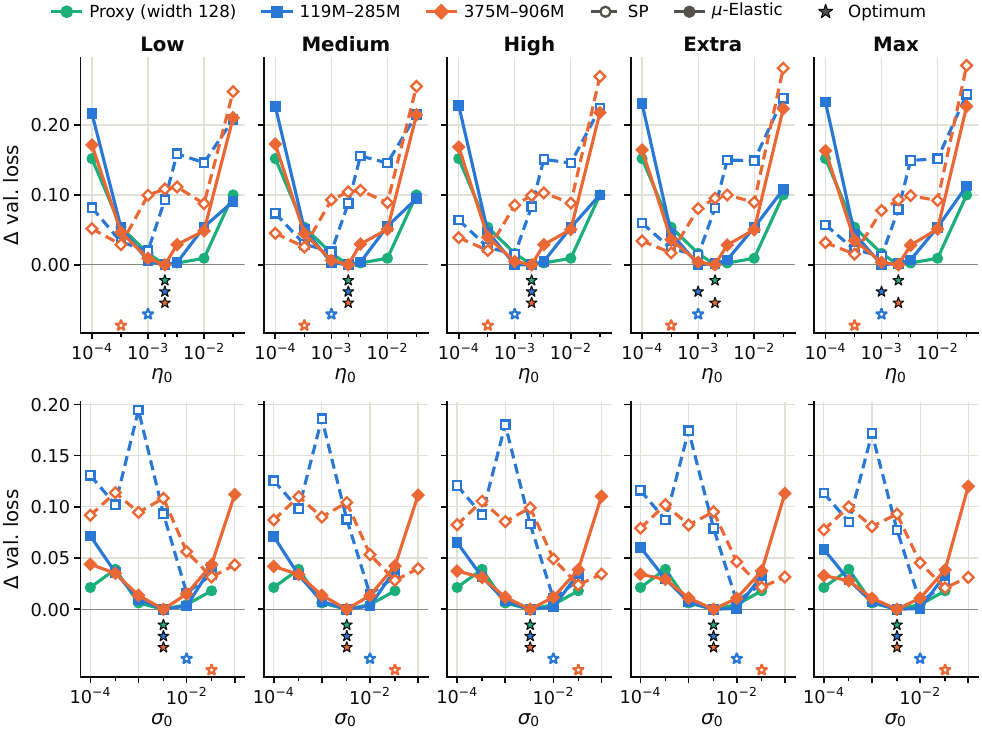}
\caption{\textbf{Per-sub-network sweeps for the fully elastic model}, as in Figure~\ref{fig:mutransfer-matformer}.}
\label{fig:mutransfer-fully-elastic}
\end{figure}

\begin{figure}[htp]
\centering
\includegraphics[width=\linewidth]{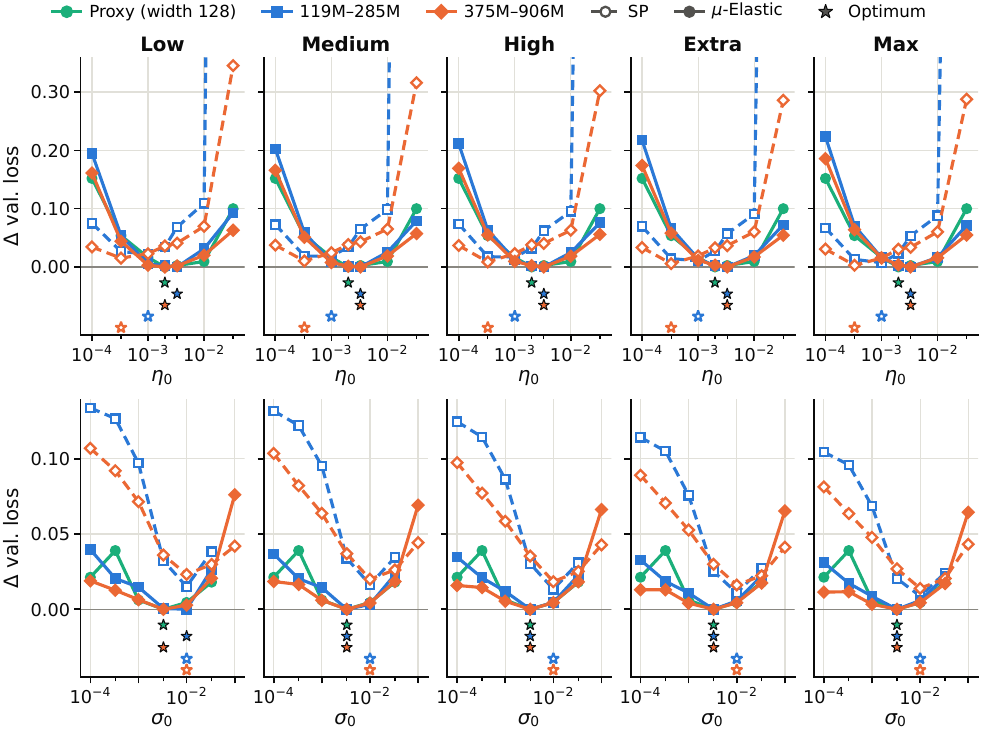}
\caption{\textbf{Per-sub-network sweeps for \Ours{}}, as in Figure~\ref{fig:mutransfer-matformer}.}
\label{fig:mutransfer-shell}
\end{figure}

\paragraph{Weight decay.}
Figure~\ref{fig:mutransfer-wd} sweeps the weight decay $\lambda_0$ over seven values for the proxy and for \Ours{} at the $119$M--$285$M scale, with $\eta_0$ and $\sigma_0$ at the base values of Section~\ref{sec:setup}, and Figure~\ref{fig:mutransfer-wd-tiers} shows each sub-network.
Two observations emerge:
\begin{itemize}
\item \textbf{Weight decay transfers under \muE{}.} From $\lambda_0=10^{-3}$ to $10^{-1}$, the mean loss of \Ours{} under \muE{} varies by at most $0.003$ nats, the same flat range as the proxy's, and both rise from $3.3\times10^{-1}$ on; the base value is within $0.001$ nats of the best \muE{} value.
\item \textbf{\muE{} is lower than SP at every swept value.} From $3.3\times10^{-3}$ to $10^{-1}$, where SP is also flat (within $0.009$ nats), the gap is $0.018$--$0.028$ nats in the mean and again largest at the narrow widths, $0.023$--$0.032$ nats at Low against $0.010$--$0.020$ at Max. SP also degrades at $10^{-3}$, where the gap grows to $0.039$ nats, and the gap shrinks to $0.010$ and $0.005$ nats at $3.3\times10^{-1}$ and $1$, where both degrade.
\end{itemize}

\begin{figure}[ht]
\centering
\includegraphics[width=0.6\linewidth]{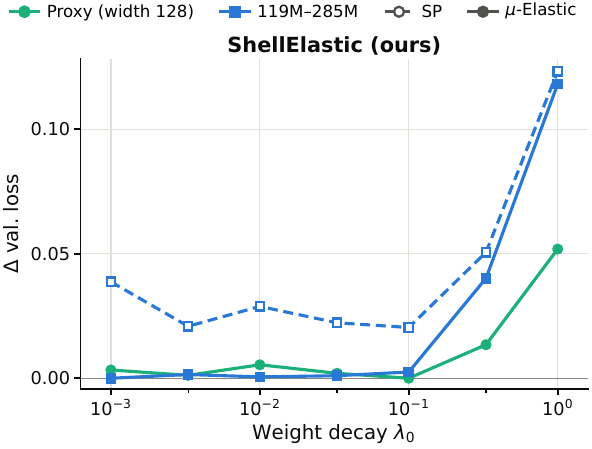}
\caption{\textbf{Weight-decay sweeps.} Sweeps of $\lambda_0$ for the proxy under \mup{} and for \Ours{} at the $119$M--$285$M scale under \muE{} (solid) and SP (dashed, open markers); $\Delta$ val.\ loss is as in Figure~\ref{fig:mutransfer}.}
\label{fig:mutransfer-wd}
\end{figure}

\begin{figure}[ht]
\centering
\includegraphics[width=\linewidth]{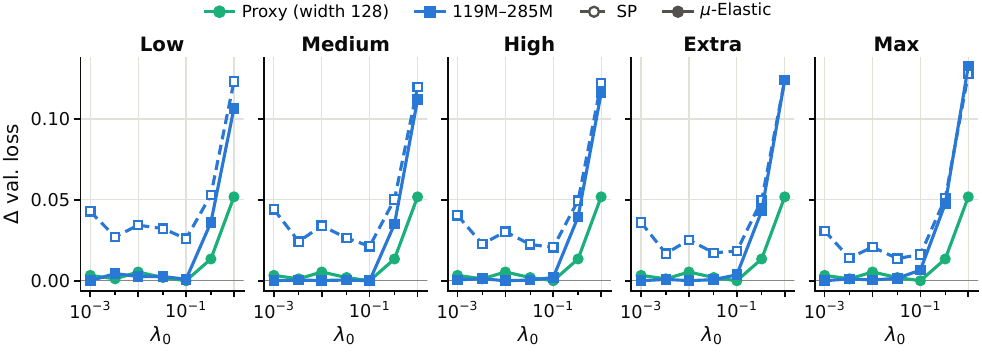}
\caption{\textbf{Per-sub-network weight-decay sweeps.} The sweeps of Figure~\ref{fig:mutransfer-wd} for each sub-network of \Ours{}, from the smallest (Low) to the largest (Max); lines and $\Delta$ val.\ loss as there, with the loss of each sub-network in place of the mean over the five, and one y-range for all panels.}
\label{fig:mutransfer-wd-tiers}
\end{figure}

\begin{table}[htp]
\caption{\textbf{\boldmath Effect of \muE{} on each elastic model at the $119$M--$285$M and $375$M--$906$M scales}, the values of Table~\ref{tab:muinternal} with four decimals.}
\label{tab:muinternal-full}
\centering
\small
\begin{tabular}{@{}lcccccc@{}}
\toprule
Method & Low & Medium & High & Extra & Max & Avg. \\
\midrule
\multicolumn{7}{@{}l}{\emph{$119$M--$285$M, train loss $\downarrow$}} \\
MatFormer & 3.2614 & 3.1841 & 3.1566 & \textbf{3.1405} & \textbf{3.1393} & 3.1764 \\
\quad w/ \muE{} & \textbf{3.2536} & \textbf{3.1821} & \textbf{3.1561} & 3.1407 & 3.1436 & \textbf{3.1752} \\ \addlinespace[2pt]
Fully elastic & 3.2717 & 3.1917 & 3.1544 & 3.1304 & 3.1263 & 3.1749 \\
\quad w/ \muE{} & \textbf{3.2572} & \textbf{3.1783} & \textbf{3.1447} & \textbf{3.1233} & \textbf{3.1196} & \textbf{3.1646} \\ \addlinespace[2pt]
\textbf{\Ours} & 3.2472 & 3.1801 & 3.1356 & 3.1040 & 3.0794 & 3.1493 \\
\quad \textbf{w/ \muE{}} & \textbf{3.2305} & \textbf{3.1641} & \textbf{3.1229} & \textbf{3.0941} & \textbf{3.0713} & \textbf{3.1366} \\
\midrule
\multicolumn{7}{@{}l}{\emph{$119$M--$285$M, val.\ loss $\downarrow$}} \\
MatFormer & 3.4341 & 3.3598 & 3.3271 & 3.3138 & \textbf{3.3110} & 3.3492 \\
\quad w/ \muE{} & \textbf{3.4248} & \textbf{3.3566} & \textbf{3.3250} & \textbf{3.3123} & 3.3139 & \textbf{3.3465} \\ \addlinespace[2pt]
Fully elastic & 3.4451 & 3.3681 & 3.3251 & 3.3030 & 3.2969 & 3.3477 \\
\quad w/ \muE{} & \textbf{3.4291} & \textbf{3.3536} & \textbf{3.3145} & \textbf{3.2954} & \textbf{3.2899} & \textbf{3.3365} \\ \addlinespace[2pt]
\textbf{\Ours} & 3.4232 & 3.3538 & 3.3077 & 3.2749 & 3.2491 & 3.3217 \\
\quad \textbf{w/ \muE{}} & \textbf{3.4086} & \textbf{3.3375} & \textbf{3.2944} & \textbf{3.2641} & \textbf{3.2403} & \textbf{3.3090} \\
\midrule
\multicolumn{7}{@{}l}{\emph{$375$M--$906$M, train loss $\downarrow$}} \\
MatFormer & 2.9140 & 2.8413 & 2.8152 & 2.7958 & 2.7929 & 2.8318 \\
\quad w/ \muE{} & \textbf{2.8993} & \textbf{2.8283} & \textbf{2.8035} & \textbf{2.7857} & \textbf{2.7871} & \textbf{2.8208} \\ \addlinespace[2pt]
Fully elastic & 2.9061 & 2.8418 & 2.8133 & 2.7900 & 2.7856 & 2.8274 \\
\quad w/ \muE{} & \textbf{2.8791} & \textbf{2.8172} & \textbf{2.7933} & \textbf{2.7735} & \textbf{2.7710} & \textbf{2.8068} \\ \addlinespace[2pt]
\textbf{\Ours} & 2.8877 & 2.8303 & 2.7963 & 2.7696 & \textbf{2.7481} & 2.8064 \\
\quad \textbf{w/ \muE{}} & \textbf{2.8746} & \textbf{2.8222} & \textbf{2.7901} & \textbf{2.7671} & 2.7488 & \textbf{2.8006} \\
\midrule
\multicolumn{7}{@{}l}{\emph{$375$M--$906$M, val.\ loss $\downarrow$}} \\
MatFormer & 3.0889 & 3.0194 & 2.9881 & 2.9741 & 2.9700 & 3.0081 \\
\quad w/ \muE{} & \textbf{3.0724} & \textbf{3.0050} & \textbf{2.9750} & \textbf{2.9628} & \textbf{2.9631} & \textbf{2.9957} \\ \addlinespace[2pt]
Fully elastic & 3.0808 & 3.0202 & 2.9867 & 2.9690 & 2.9633 & 3.0040 \\
\quad w/ \muE{} & \textbf{3.0519} & \textbf{2.9948} & \textbf{2.9662} & \textbf{2.9520} & \textbf{2.9481} & \textbf{2.9826} \\ \addlinespace[2pt]
\textbf{\Ours} & 3.0697 & 3.0087 & 2.9732 & 2.9461 & \textbf{2.9243} & 2.9844 \\
\quad \textbf{w/ \muE{}} & \textbf{3.0546} & \textbf{2.9998} & \textbf{2.9670} & \textbf{2.9435} & 2.9250 & \textbf{2.9780} \\
\bottomrule
\end{tabular}
\end{table}

\FloatBarrier
\section{Architecture comparison and search}
\label{app:arch-compare}

This appendix complements Section~\ref{sec:arch-compare}: Appendix~\ref{app:stairformer} gives the per-task results and training curves of StairFormer, and Appendix~\ref{app:shape-all} details the architecture search at equal parameters and width count.

\subsection{Comparison with StairFormer}
\label{app:stairformer}

StairFormer~\citep{trost2026fullynested} nests sub-networks as prefixes of one forward pass with equal-width groups, $\mathcal S=\{1,\dots,M\}$ in our notation; it normalizes over cumulative prefixes (Equation~\ref{eq:normvariants}) and shares one terminal-norm gain and one readout across sub-networks.
Our StairFormer baseline has five $208$-wide groups ($g=208$, $\mathcal S=\{1,2,3,4,5\}$) with $d_{\mathrm{head}}=104$, so that each group holds two complete heads, and an FFN hidden width of $236$ per group ($e\approx1.13$, against $e=1$ for \Ours{}), so that its Max matches that of \Ours{} in parameters.
Like \Ours{}, it trains all sub-networks at every step; we give it the data, steps, HPs and loss weights of Section~\ref{sec:setup}, and \muE{} with $m=208/128=1.625$ (Appendix~\ref{app:other-param}).
Table~\ref{tab:stairformer-tasks} gives its per-task results at every tier, and Figure~\ref{fig:stairformer-curves} its training curves next to the methods of Table~\ref{tab:main-scales} other than sampled standalone.
Two observations emerge:
\begin{itemize}
\item \textbf{Only Max is parameter-matched.} StairFormer's Max and that of \Ours{} both have $279.6$M parameters, but at Tiers~1--4, StairFormer's sub-networks are $14$--$73\%$ smaller ($32.6$--$200.4$M against $118.8$--$233.1$M), so these pairs compare sub-networks of different sizes; Figure~\ref{fig:stairformer} compares the validation loss and average 5-shot accuracy against parameter count.
\item \textbf{At Max, \Ours{} is better on most tasks.} At the matched size, \Ours{} is better on $9$ of the $13$ tasks and worse on $4$ (MMLU, SVAMP, WinoGrande and LSAT-RC).
\end{itemize}

\begin{table}[tp]
\caption{\textbf{Per-task results of StairFormer and \Ours{} at every tier} of the $119$M--$285$M scale, with the rows of Appendix~\ref{app:main-full}. Stair.: StairFormer. Ours: \Ours{}. Bold marks the better entry of each pair (ties both bold). Only Tier~5 is parameter-matched.}
\label{tab:stairformer-tasks}
\centering
\begin{tabular}{@{}l@{\hspace{6pt}}c@{\hspace{4pt}}c@{\hspace{6pt}}c@{\hspace{4pt}}c@{\hspace{6pt}}c@{\hspace{4pt}}c@{\hspace{6pt}}c@{\hspace{4pt}}c@{\hspace{6pt}}c@{\hspace{4pt}}c@{}}
\toprule
 & \multicolumn{2}{@{}c@{\hspace{6pt}}}{Tier~1\,$\cdot$\,Low} & \multicolumn{2}{@{}c@{\hspace{6pt}}}{Tier~2\,$\cdot$\,Medium} & \multicolumn{2}{@{}c@{\hspace{6pt}}}{Tier~3\,$\cdot$\,High} & \multicolumn{2}{@{}c@{\hspace{6pt}}}{Tier~4\,$\cdot$\,Extra} & \multicolumn{2}{@{}c@{}}{Tier~5\,$\cdot$\,Max} \\
\cmidrule(lr){2-3} \cmidrule(lr){4-5} \cmidrule(lr){6-7} \cmidrule(lr){8-9} \cmidrule(lr){10-11}
 & Stair. & \textcolor{ourcolor}{\textbf{Ours}} & Stair. & \textcolor{ourcolor}{\textbf{Ours}} & Stair. & \textcolor{ourcolor}{\textbf{Ours}} & Stair. & \textcolor{ourcolor}{\textbf{Ours}} & Stair. & \textcolor{ourcolor}{\textbf{Ours}} \\
\midrule
Params (M) & 32.6 & 118.8 & 76.9 & 152.7 & 132.8 & 190.8 & 200.4 & 233.1 & 279.6 & 279.6 \\
Train loss $\downarrow$ & 3.6874 & \textbf{3.2305} & 3.4224 & \textbf{3.1641} & 3.2764 & \textbf{3.1229} & 3.1871 & \textbf{3.0941} & 3.1258 & \textbf{3.0713} \\
Val.\ loss $\downarrow$ & 3.8836 & \textbf{3.4086} & 3.6004 & \textbf{3.3375} & 3.4519 & \textbf{3.2944} & 3.3588 & \textbf{3.2641} & 3.2949 & \textbf{3.2403} \\
\midrule
\multicolumn{11}{@{}l}{\emph{5-shot accuracy (\%) $\uparrow$}} \\
MMLU & \textbf{25.1} & 24.7 & \textbf{25.1} & 23.8 & \textbf{25.1} & 24.0 & \textbf{25.5} & 24.4 & \textbf{25.3} & 24.2 \\
ARC-Easy & 35.7 & \textbf{45.0} & 41.3 & \textbf{45.8} & 41.8 & \textbf{47.1} & 44.0 & \textbf{47.6} & 45.7 & \textbf{48.0} \\
ARC-Challenge & 21.5 & \textbf{23.3} & 22.4 & \textbf{23.5} & 22.3 & \textbf{23.6} & 22.5 & \textbf{24.0} & 23.2 & \textbf{24.3} \\ \addlinespace[3pt]
SVAMP & \textbf{6.0} & 5.7 & 5.3 & \textbf{6.3} & \textbf{7.3} & 6.3 & \textbf{7.0} & 6.7 & \textbf{8.3} & 6.4 \\ \addlinespace[3pt]
COPA & 55.0 & \textbf{59.7} & 57.0 & \textbf{61.7} & 60.0 & \textbf{61.7} & 56.0 & \textbf{63.0} & 59.0 & \textbf{63.3} \\
PIQA & 57.8 & \textbf{62.1} & 60.2 & \textbf{63.7} & 61.8 & \textbf{64.6} & 63.1 & \textbf{64.8} & 64.0 & \textbf{65.2} \\ \addlinespace[3pt]
HellaSwag & 27.0 & \textbf{32.4} & 29.7 & \textbf{33.2} & 31.1 & \textbf{33.9} & 32.2 & \textbf{34.6} & 33.6 & \textbf{35.1} \\
WinoGrande & \textbf{52.6} & 50.7 & \textbf{51.9} & 49.3 & \textbf{51.8} & 50.0 & \textbf{52.2} & 49.8 & \textbf{50.4} & 49.8 \\
LAMBADA & 15.0 & \textbf{25.9} & 23.3 & \textbf{29.1} & 25.3 & \textbf{29.9} & 26.8 & \textbf{30.2} & 27.4 & \textbf{30.4} \\ \addlinespace[3pt]
BoolQ & 47.0 & \textbf{55.1} & 48.4 & \textbf{57.2} & 53.2 & \textbf{55.3} & 51.1 & \textbf{54.5} & 51.7 & \textbf{52.3} \\
LSAT-RC & \textbf{25.4} & 24.8 & 23.1 & \textbf{25.6} & 22.4 & \textbf{26.0} & 25.4 & \textbf{26.4} & \textbf{28.0} & 25.6 \\
LSAT-LR & 25.7 & \textbf{26.2} & 26.5 & \textbf{29.2} & 24.1 & \textbf{26.8} & 24.5 & \textbf{27.6} & 24.1 & \textbf{27.6} \\
SAT-EN & \textbf{25.7} & 25.6 & \textbf{27.2} & \textbf{27.2} & 24.3 & \textbf{26.2} & 24.3 & \textbf{26.5} & 23.8 & \textbf{26.9} \\
\midrule
Average & 32.3 & \textbf{35.5} & 34.0 & \textbf{36.6} & 34.7 & \textbf{36.6} & 35.0 & \textbf{36.9} & 35.7 & \textbf{36.9} \\
\bottomrule
\end{tabular}
\end{table}

\begin{figure}[tp]
\centering
\includegraphics[width=\linewidth]{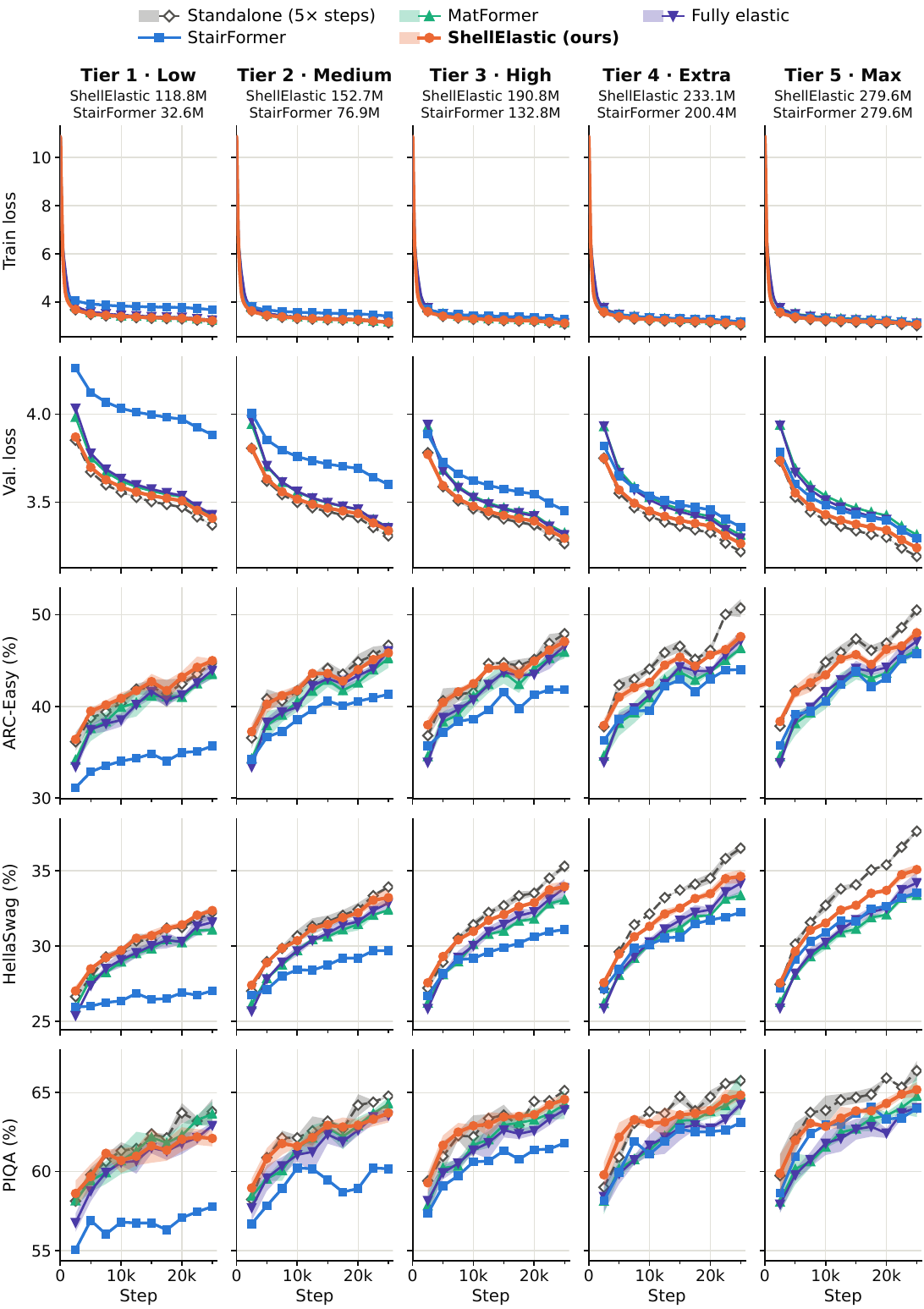}
\caption{\textbf{Training curves of StairFormer and the methods of Table~\ref{tab:main-scales}} at the $119$M--$285$M scale, except sampled standalone, in the format of Figure~\ref{fig:main-scale1}; column titles give the parameters of \Ours{} and StairFormer.}
\label{fig:stairformer-curves}
\end{figure}

\subsection{Architecture search at equal parameters and width count}
\label{app:shape-all}

The runs of Figures~\ref{fig:arch-placement}c and~\ref{fig:arch-placement-b39} train every supported width at every step with the default recipe and change only $\mathcal S$.
A supported set fixes how many of the $M^2$ blocks of size $g\times g$ a hidden matrix keeps: group $\mathcal G_k$ has $s_k-s_{k-1}$ block rows, each reading $s_k$ input blocks, so
\begin{equation}
 B(\mathcal S)=\sum_{k=1}^{K}(s_k-s_{k-1})\,s_k .
\label{eq:blockcount}
\end{equation}
We compare two families of supported sets (Figure~\ref{fig:arch-masks}): five supported widths at $B=42$, the block count of the default $\mathcal S=\{4,5,6,7,8\}$, and six supported widths at $B=39$.
Every set contains exactly one group wider than $g$, and the sets of a family differ only in where this group sits, from the last shell to the core.
A width at or above the wide group has the same number of blocks wherever the group sits, so its parameter count is the same in every set of the family; a width below it keeps only the lower-triangle blocks and has fewer parameters ($63.9$M at $384$, $93.7$M at $512$, $127.6$M at $640$; marked $^\dagger$ in Table~\ref{tab:arch-placement}).
The average in Table~\ref{tab:arch-placement} is taken over the supported widths of each set, so it covers different widths in different rows, and the core placements have no width below the core.

\begin{figure}[htp]
\centering
\includegraphics[width=\linewidth]{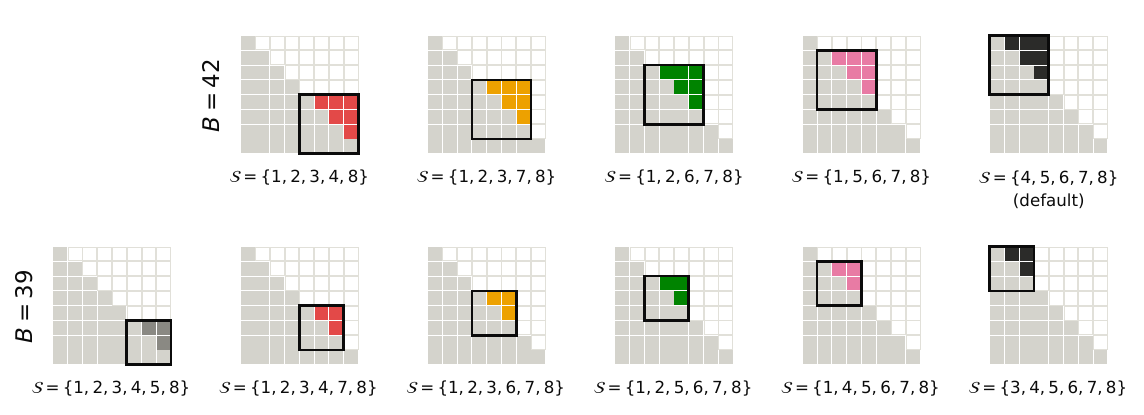}
\caption{\textbf{Block masks of the supported sets in Figures~\ref{fig:arch-placement}c and~\ref{fig:arch-placement-b39}} ($M=8$, $g=128$). Cell $(i,j)$ is the $g\times g$ block of every hidden matrix from input block $j$ to output block $i$. Gray: lower-triangle blocks kept by every set; colored: blocks above the diagonal, all inside the one group wider than $g$ (outlined). Each row has the same $B$ and moves the wide group from the last shell (left) to the core (right).}
\label{fig:arch-masks}
\end{figure}

\begin{figure}[htp]
\centering
\includegraphics[width=0.6\linewidth]{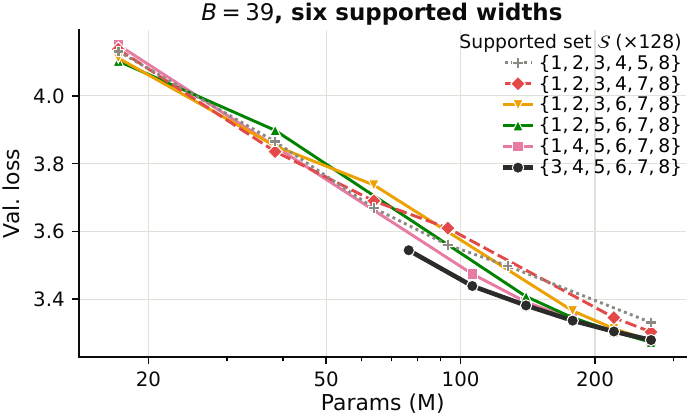}
\caption{As Figure~\ref{fig:arch-placement}c, at $B=39$ with six supported widths (core placement $\mathcal S=\{3,4,5,6,7,8\}$, black).}
\label{fig:arch-placement-b39}
\end{figure}

\begin{table}[htp]
\caption{\textbf{Validation loss of every supported width in the architecture search}, at step $25{,}000$, as the one group wider than $g$ moves from the last shell (first row of each block) to the core (last row, bold). Params: non-zero parameters (M) of each width; $^\dagger$a width below the wide group, which has fewer parameters. Avg.: mean over the supported widths of each set. Bold marks the lowest value in each column.}
\label{tab:arch-placement}
\centering
\setlength{\tabcolsep}{3pt}
\begin{tabular}{@{}lccccccccc@{}}
\toprule
$\mathcal S$ & 128 & 256 & 384 & 512 & 640 & 768 & 896 & 1024 & Avg. \\
\midrule
\multicolumn{10}{@{}l}{\textit{$B=42$, 5 supported widths}} \\
Params & 17.1 & 38.4 & 63.9 & 118.8 & 152.7 & 190.8 & 233.1 & 279.6 & \\
$\{1,2,3,4,8\}$ & 4.1178 & \textbf{3.8183} & \textbf{3.6738} & 3.6023$^\dagger$ & -- & -- & -- & 3.2698 & 3.6964 \\
$\{1,2,3,7,8\}$ & 4.1037 & 3.8310 & 3.7224 & -- & -- & -- & 3.2910 & 3.2503 & 3.6397 \\
$\{1,2,6,7,8\}$ & \textbf{4.0895} & 3.8832 & -- & -- & -- & 3.3193 & 3.2721 & \textbf{3.2400} & 3.5608 \\
$\{1,5,6,7,8\}$ & 4.1154 & -- & -- & -- & 3.3657 & 3.3068 & 3.2687 & 3.2409 & 3.4595 \\
$\boldsymbol{\{4,5,6,7,8\}}$ & -- & -- & -- & \textbf{3.4086} & \textbf{3.3375} & \textbf{3.2944} & \textbf{3.2641} & 3.2403 & \textbf{3.3090} \\
\midrule
\multicolumn{10}{@{}l}{\textit{$B=39$, 6 supported widths}} \\
Params & 17.1 & 38.4 & 76.5 & 106.2 & 140.1 & 178.2 & 220.5 & 267.0 & \\
$\{1,2,3,4,5,8\}$ & 4.1310 & 3.8646 & 3.6689$^\dagger$ & 3.5601$^\dagger$ & 3.4983$^\dagger$ & -- & -- & 3.3314 & 3.6757 \\
$\{1,2,3,4,7,8\}$ & 4.1371 & \textbf{3.8356} & 3.6898$^\dagger$ & 3.6098$^\dagger$ & -- & -- & 3.3460 & 3.3035 & 3.6537 \\
$\{1,2,3,6,7,8\}$ & 4.1112 & 3.8520 & 3.7369$^\dagger$ & -- & -- & 3.3659 & 3.3134 & 3.2775 & 3.6095 \\
$\{1,2,5,6,7,8\}$ & \textbf{4.1022} & 3.8981 & -- & -- & 3.4082 & 3.3456 & \textbf{3.3047} & \textbf{3.2734} & 3.5554 \\
$\{1,4,5,6,7,8\}$ & 4.1499 & -- & -- & 3.4746 & 3.3919 & 3.3419 & 3.3057 & 3.2778 & 3.4903 \\
$\boldsymbol{\{3,4,5,6,7,8\}}$ & -- & -- & \textbf{3.5445} & \textbf{3.4391} & \textbf{3.3818} & \textbf{3.3371} & 3.3050 & 3.2798 & \textbf{3.3812} \\
\bottomrule
\end{tabular}
\end{table}

\FloatBarrier
\section{Ablation studies}
\label{app:ablation-studies}
\label{app:ablations}

All ablations use the supernet of Section~\ref{sec:setup} ($g=128$, $\mathcal S=\{4,5,6,7,8\}$) at the $119$M--$285$M scale and change one setting of its recipe at a time.
Tables~\ref{tab:ablations-1} and~\ref{tab:ablations-2} list all results, numbered as the paragraphs below; each block of rows ends with the default.
\begin{itemize}
\item \textbf{Default and run-to-run variation.} The default rows average the three \Ours{} runs of Table~\ref{tab:main-scales}, whose standard deviation is at most $0.003$ nats in the validation loss of any tier, $0.002$ nats in the mean validation loss and $0.32$ points in the mean 5-shot accuracy; every other row is a single run. A further run of the default on H100 GPUs is within $0.003$ nats of this mean at every tier, so we treat differences below $0.003$ nats as run-to-run variation.
\item \textbf{Hardware and cost.} The default runs used A100 GPUs with activation checkpointing on every other block, as did ablations~(11) and~(12), whereas ablations~(1) to~(10) used H100 GPUs without it. Step times and peak memory are compared between runs on the same GPUs with the same activation checkpointing; on H100 GPUs, the reference is the loss-weight runs of~(2) ($37.1$\,GB peak memory).
\end{itemize}

\begin{table}[htp]
\caption{\textbf{\boldmath Ablations (1)--(6) of \Ours{} at the $119$M--$285$M scale.} Each block changes one setting of the default recipe and ends with the default. Val.\ loss is taken at step $25{,}000$ for each tier (Low to Max: widths $512$ to $1024$), and 5-shot acc.\ averages the $13$ downstream tasks over the five tiers. Bold marks the best value of each column within a block.}
\label{tab:ablations-1}
\centering
\small
\setlength{\tabcolsep}{3pt}
\begin{tabular}{@{}lcccccc@{\hspace{8pt}}c@{}}
\toprule
 & \multicolumn{6}{c}{Val.\ loss $\downarrow$} & 5-shot acc.\ $\uparrow$ \\
\cmidrule(r{4pt}){2-7} \cmidrule(l{0pt}){8-8}
Setting & Low & Medium & High & Extra & Max & Avg. & Avg. \\
\midrule
\multicolumn{8}{@{}l}{\emph{(1) Selected sub-networks per step $\mathcal A_\tau$}} \\
Sandwich: Low, Max and one other & \textbf{3.4019} & 3.3617 & 3.3183 & 3.2816 & \textbf{3.2395} & 3.3206 & 36.1 \\
Random: one of the five sub-networks & 3.4311 & 3.3786 & 3.3542 & 3.3410 & 3.3356 & 3.3681 & 34.8 \\
Max + random: Max and one other & 3.4944 & 3.4079 & 3.3447 & 3.2956 & 3.2397 & 3.3564 & 35.5 \\
All: all five sub-networks (default) & 3.4086 & \textbf{3.3375} & \textbf{3.2944} & \textbf{3.2641} & 3.2403 & \textbf{3.3090} & \textbf{36.5} \\
\midrule
\multicolumn{8}{@{}l}{\emph{(2) Loss weights $\omega_s$ (Low to Max)}} \\
$(1,1,1,1,1)$ & 3.3888 & \textbf{3.3279} & \textbf{3.2923} & 3.2677 & 3.2518 & 3.3057 & 36.5 \\
$(1,0.5,0.3,0.2,2)$ & \textbf{3.3801} & 3.3312 & 3.2978 & 3.2690 & 3.2412 & \textbf{3.3039} & \textbf{36.6} \\
$(1,1,1,1,2)$ & 3.3959 & 3.3317 & 3.2934 & 3.2673 & 3.2486 & 3.3074 & 36.3 \\
$(1,1,1,1,6)$ & 3.4236 & 3.3437 & 3.2977 & \textbf{3.2638} & \textbf{3.2365} & 3.3130 & 36.2 \\
$(1,1,1,1,4)$ (default) & 3.4086 & 3.3375 & 3.2944 & 3.2641 & 3.2403 & 3.3090 & 36.5 \\
\midrule
\multicolumn{8}{@{}l}{\emph{(3) Trunk normalization}} \\
Cumulative prefix & \textbf{3.4060} & 3.3376 & 3.2952 & 3.2645 & 3.2410 & \textbf{3.3088} & 36.1 \\
Block-local over every $g$ coordinates & 3.4093 & \textbf{3.3368} & 3.2955 & 3.2652 & 3.2411 & 3.3096 & \textbf{36.7} \\
Group-local over $\mathcal G_k$ (default) & 3.4086 & 3.3375 & \textbf{3.2944} & \textbf{3.2641} & \textbf{3.2403} & 3.3090 & 36.5 \\
\midrule
\multicolumn{8}{@{}l}{\emph{(4) Terminal norm $N^{\rm out}_s$}} \\
Group-local, as the trunk & 3.4122 & 3.3414 & 3.2975 & 3.2656 & 3.2417 & 3.3117 & 35.4 \\
Over all $sg$ coordinates (default) & \textbf{3.4086} & \textbf{3.3375} & \textbf{3.2944} & \textbf{3.2641} & \textbf{3.2403} & \textbf{3.3090} & \textbf{36.5} \\
\midrule
\multicolumn{8}{@{}l}{\emph{(5) Readout}} \\
Shared $U$ & 3.4152 & 3.3446 & 3.3000 & 3.2696 & 3.2450 & 3.3149 & 35.6 \\
Separate $U_s$ (default) & \textbf{3.4086} & \textbf{3.3375} & \textbf{3.2944} & \textbf{3.2641} & \textbf{3.2403} & \textbf{3.3090} & \textbf{36.5} \\
\midrule
\multicolumn{8}{@{}l}{\emph{(6) In-place distillation from Max}} \\
Forward KL, weight $0.1$ & 3.4047 & 3.3355 & 3.2937 & 3.2641 & 3.2410 & 3.3078 & 36.4 \\
Forward KL, weight $0.5$ & 3.3920 & 3.3287 & 3.2931 & 3.2675 & 3.2482 & 3.3059 & \textbf{36.6} \\
Forward KL, weight $1$ & \textbf{3.3748} & \textbf{3.3184} & \textbf{3.2873} & 3.2653 & 3.2483 & \textbf{3.2988} & 36.1 \\
Reverse KL, weight $0.5$ & 3.3868 & 3.3261 & 3.2887 & 3.2632 & 3.2437 & 3.3017 & 35.9 \\
Symmetric KL, weight $0.5$ & 3.3875 & 3.3256 & 3.2901 & \textbf{3.2629} & 3.2435 & 3.3019 & 36.0 \\
None (default) & 3.4086 & 3.3375 & 3.2944 & 3.2641 & \textbf{3.2403} & 3.3090 & 36.5 \\
\bottomrule
\end{tabular}
\end{table}

\begin{table}[t]
\caption{\textbf{\boldmath Ablations (7)--(12) of \Ours{} at the $119$M--$285$M scale}, in the format of Table~\ref{tab:ablations-1}.}
\label{tab:ablations-2}
\centering
\small
\setlength{\tabcolsep}{3pt}
\begin{tabular}{@{}lcccccc@{\hspace{8pt}}c@{}}
\toprule
 & \multicolumn{6}{c}{Val.\ loss $\downarrow$} & 5-shot acc.\ $\uparrow$ \\
\cmidrule(r{4pt}){2-7} \cmidrule(l{0pt}){8-8}
Setting & Low & Medium & High & Extra & Max & Avg. & Avg. \\
\midrule
\multicolumn{8}{@{}l}{\emph{(7) Normalization type}} \\
LayerNorm & \textbf{3.4037} & 3.3379 & 3.2955 & 3.2655 & 3.2424 & \textbf{3.3090} & 35.8 \\
RMSNorm (default) & 3.4086 & \textbf{3.3375} & \textbf{3.2944} & \textbf{3.2641} & \textbf{3.2403} & \textbf{3.3090} & \textbf{36.5} \\
\midrule
\multicolumn{8}{@{}l}{\emph{(8) Ensembles of the five sub-networks}} \\
Learned per-token width router & 3.7735 & 3.5611 & 3.4137 & 3.3055 & \textbf{3.2269} & 3.4562 & 35.1 \\
Distillation from loss-weighted ensemble & \textbf{3.3902} & \textbf{3.3275} & \textbf{3.2926} & 3.2673 & 3.2471 & \textbf{3.3049} & \textbf{37.0} \\
Distillation from soft oracle & 3.3982 & 3.3350 & 3.2980 & 3.2709 & 3.2501 & 3.3105 & 36.3 \\
None (default) & 3.4086 & 3.3375 & 3.2944 & \textbf{3.2641} & 3.2403 & 3.3090 & 36.5 \\
\midrule
\multicolumn{8}{@{}l}{\emph{(9) Stage-residual gradient scaling}} \\
$\gamma=0$ & 3.4053 & 3.3740 & 3.3530 & 3.3348 & 3.3115 & 3.3557 & 36.1 \\
$\gamma=0.3$ & \textbf{3.3806} & \textbf{3.3317} & 3.2991 & 3.2762 & 3.2485 & \textbf{3.3072} & 36.6 \\
$\gamma=0.5$ & 3.3939 & 3.3360 & 3.2967 & 3.2700 & 3.2447 & 3.3083 & 36.0 \\
$\gamma=0.8$ & 3.4030 & 3.3365 & 3.2946 & 3.2658 & 3.2419 & 3.3084 & \textbf{37.0} \\
$\gamma=1$ (default) & 3.4086 & 3.3375 & \textbf{3.2944} & \textbf{3.2641} & \textbf{3.2403} & 3.3090 & 36.5 \\
\midrule
\multicolumn{8}{@{}l}{\emph{(10) Trunk gradient of Low to Extra}} \\
Detached & 3.9289 & 3.6481 & 3.4489 & 3.3134 & \textbf{3.2259} & 3.5130 & 35.2 \\
Kept (default) & \textbf{3.4086} & \textbf{3.3375} & \textbf{3.2944} & \textbf{3.2641} & 3.2403 & \textbf{3.3090} & \textbf{36.5} \\
\midrule
\multicolumn{8}{@{}l}{\emph{(11) Base width $g$ and head dimension $d_{\mathrm{head}}$}} \\
$g=64$, $d_{\mathrm{head}}=64$ & \textbf{3.3999} & 3.3336 & \textbf{3.2919} & \textbf{3.2619} & \textbf{3.2383} & \textbf{3.3051} & 36.2 \\
$g=128$, $d_{\mathrm{head}}=64$ & 3.4021 & \textbf{3.3334} & 3.2929 & 3.2636 & 3.2400 & 3.3064 & 36.3 \\
$g=128$, $d_{\mathrm{head}}=128$ (default) & 3.4086 & 3.3375 & 3.2944 & 3.2641 & 3.2403 & 3.3090 & \textbf{36.5} \\
\midrule
\multicolumn{8}{@{}l}{\emph{(12a) Gradient clipping (global norm)}} \\
$0.5$ & 3.4088 & 3.3392 & 3.2962 & 3.2655 & 3.2420 & 3.3104 & 36.3 \\
$1$ & 3.4150 & 3.3431 & 3.2988 & 3.2679 & 3.2437 & 3.3137 & 35.9 \\
$2$ & \textbf{3.4037} & \textbf{3.3342} & \textbf{3.2927} & \textbf{3.2633} & 3.2406 & \textbf{3.3069} & \textbf{36.6} \\
$4$ (default) & 3.4086 & 3.3375 & 3.2944 & 3.2641 & \textbf{3.2403} & 3.3090 & 36.5 \\
\midrule
\multicolumn{8}{@{}l}{\emph{(12b) AdamW $\epsilon$}} \\
$10^{-6}$ & 3.4097 & 3.3411 & 3.3011 & 3.2737 & 3.2552 & 3.3162 & 35.7 \\
$10^{-8}$ & \textbf{3.4055} & \textbf{3.3348} & \textbf{3.2920} & \textbf{3.2628} & \textbf{3.2398} & \textbf{3.3070} & 36.0 \\
$10^{-10}$ & 3.4057 & 3.3373 & 3.2956 & 3.2660 & 3.2428 & 3.3095 & 36.4 \\
$10^{-14}$ (default) & 3.4086 & 3.3375 & 3.2944 & 3.2641 & 3.2403 & 3.3090 & \textbf{36.5} \\
\midrule
\multicolumn{8}{@{}l}{\emph{(12c) $z$-loss coefficient $\lambda_z$}} \\
$10^{-3}$ & 3.4096 & 3.3393 & 3.2970 & 3.2664 & 3.2414 & 3.3107 & 35.9 \\
$10^{-2}$ & \textbf{3.4074} & \textbf{3.3363} & \textbf{3.2938} & \textbf{3.2640} & \textbf{3.2403} & \textbf{3.3084} & \textbf{36.5} \\
$10^{-4}$ (default) & 3.4086 & 3.3375 & 3.2944 & 3.2641 & \textbf{3.2403} & 3.3090 & \textbf{36.5} \\
\bottomrule
\end{tabular}
\end{table}

\paragraph{(1) Selected sub-networks per step.}
Besides the default (All), which selects all five sub-networks at every step, we test three rules for the selected set $\mathcal A_\tau$ of Equation~\ref{eq:loss}:
\begin{itemize}
\item \emph{Sandwich}: Low and Max at every step and one of the other three uniformly at random;
\item \emph{Max + random}: Max and one of the other four;
\item \emph{Random}: one of the five sub-networks uniformly at random, as MatFormer and the fully elastic model do.
\end{itemize}
\textbf{Observation:} The loss of each sub-network follows how often it is trained, and All gives the lowest mean loss.
Sandwich and Max + random keep the loss of Max but raise those of the sub-networks they select less often, by up to $0.024$ and $0.086$ nats, respectively, and Random raises every tier, Max by $0.095$ nats.

\paragraph{(2) Loss weights.}
With all widths selected at every step, we vary the loss weights $\omega_s$ of Equation~\ref{eq:loss}: the weight of Max rises from $1$ to $2$, $4$ (the default) and $6$, and the decaying weights $(1,0.5,0.3,0.2,2)$ keep the relative weight of Max at $1/2$ and double that of Low.

\textbf{Observation:} The loss weight of a width moves its loss in the same way as how often it is trained.
Raising the weight of Max from $1$ to $6$ lowers its loss monotonically, by $0.015$ nats, and raises the mean loss from $3.306$ to $3.313$, mostly through Low and Medium.
The decaying weights match the default at Max and lower the mean loss by $0.005$ nats, mostly through Low.
Training all widths at every step with weights $(1,1,1,1,4)$ is thus within $0.005$ nats of the best trade-off we found.

\paragraph{(3) Trunk normalization.}
We compare three trunk normalizations that all preserve the supported prefixes (Appendix~\ref{app:prefix-main}): group-local normalization over each $\mathcal G_k$ (the default), block-local normalization over every $g$ coordinates, and a cumulative prefix norm, which normalizes each block of $g$ coordinates with the statistics of all coordinates up to the end of that block.
With $\mathcal B_b=\{(b-1)g+1,\dots,bg\}$ the $b$-th block of $g$ coordinates, all three apply Equation~\ref{eq:norm} with the statistics of a coordinate $j\in\mathcal G_k\cap\mathcal B_b$ taken over
\begin{equation}
 \mathcal G_k\ \ \text{(group-local)},\qquad \mathcal B_b\ \ \text{(block-local)},\qquad \{1,\dots,bg\}\ \ \text{(cumulative prefix)}.
\label{eq:normvariants}
\end{equation}

\textbf{Observation:} The three are within $0.001$ nats of each other in mean loss and within $0.003$ nats at every tier.
They differ only in cost: the prefix norm, which needs cumulative statistics for every prefix of blocks, takes $6\%$ more time per step and $23\%$ more memory in our implementation, while block-local normalization costs the same as group-local normalization.

\paragraph{(4) Terminal norm.}
We compare normalizing all $sg$ coordinates of sub-network $s$ before its readout (the default, Equation~\ref{eq:head}) with normalizing each group separately, as in the trunk.

\textbf{Observation:} The default is better: $0.003$ nats lower in mean loss, mostly at Low to High, and about $1$ point higher in mean 5-shot accuracy.

\paragraph{(5) Readout.}
We compare separate readouts $U_s$ for every width (the default) with one readout $U$ shared by all widths, of which width $s$ uses the first $sg$ columns (with the split of Equation~\ref{eq:split}).

\textbf{Observation:} Separate readouts are better at every tier, by $0.006$ nats in mean loss and $0.9$ points in mean 5-shot accuracy.

\paragraph{(6) In-place distillation.}
We distill Low to Extra from Max by adding to the loss of each a weighted KL divergence between its output distribution and that of Max over the full vocabulary: forward KL with weights $0.1$, $0.5$ and $1$, and reverse and symmetric KL with weight $0.5$.

\textbf{Observation:} Distillation lowers the loss of Low to High, by up to $0.034$ nats at Low, but raises that of Max, by up to $0.008$ nats.
The mean loss falls by up to $0.010$ nats, but the mean 5-shot accuracy stays within $0.6$ points of the default and a step takes $2.2$ to $2.6$ times as long, so we do not use it.

\paragraph{(7) Normalization type.}
We compare RMSNorm (the default) with LayerNorm with a gain and no bias.

\textbf{Observation:} Both give the same mean loss ($3.309$) and differ by at most $0.005$ nats at any tier, so we keep RMSNorm.

\paragraph{(8) Ensembles of the sub-networks.}
One trunk pass yields the output distributions $p_s=\operatorname{softmax}(z_s)$ of all five sub-networks, so ensembling them costs little beyond the readouts, and nested sub-networks may still err on different tokens.
We use per-token ensembles $\sum_s\pi_sp_s$, whose weights $\pi_s$ are non-negative and sum to one, in three ways:
\begin{itemize}
\item \emph{Four ensembles of a trained supernet}: uniform weights, $\pi_s=1/5$; confidence weights, $\pi_s\propto\max_vp_s(v)$; the most confident readout, which puts all weight on the sub-network with the largest $\max_vp_s(v)$; and the best fixed ensemble, whose weights minimize the loss of the tokens it is evaluated on.
As an upper bound, an oracle puts all weight on the readout that assigns the target the highest probability.
\item \emph{Learned router}: a two-layer MLP ($1024\to256\to5$, SiLU) maps the last-layer hidden state of Max, after its terminal norm, to one logit per sub-network, and a softmax turns the five logits into the per-token weights $\pi_s$; its output layer is initialized to zero, so the weights start uniform.
It is trained jointly with the supernet, whose loss becomes the $\pi$-weighted sum of the five losses at every token plus $0.01$ times their unweighted mean, which keeps every sub-network trained.
\item \emph{Distillation from ensembles}: every sub-network, including Max, adds a forward KL divergence with weight $0.5$ toward one of two teacher ensembles, detached from the gradient: the \emph{loss-weighted ensemble}, whose weights are the loss weights $\omega_s$ of Equation~\ref{eq:loss} normalized to sum to one ($1/2$ on Max, $1/8$ on each other sub-network), or the \emph{soft oracle}, $\pi_s\propto p_s(y)^\beta$ for target $y$, a smooth version of the oracle that favors the sub-networks assigning $y$ a high probability, with $\beta$ set at every step so that the weights have a mean normalized entropy of $0.7$.
\end{itemize}
\textbf{Observation:} No ensemble beats Max alone (Table~\ref{tab:ensemble-mixtures}).
The best fixed ensemble matches it, uniform and confidence weights are $0.03$ nats worse, and the most confident readout is $0.068$ nats worse; only the oracle is better, but it needs the target.
The router collapses onto Max, which improves by $0.013$ nats, while the other sub-networks, trained only through the $0.01$ term, lose $0.041$ to $0.365$ nats.
As predictors, the loss-weighted ensemble is $0.012$ nats worse than Max and the soft oracle $0.143$ nats better, as it uses the target; yet distillation toward them raises the loss of Max by $0.007$ and $0.010$ nats and lowers that of Low by $0.018$ and $0.010$ nats, as in-place distillation does in~(6), and a step takes about $2.6$ and $3.0$ times as long.

\textbf{Nested sub-networks versus separately trained models} (Table~\ref{tab:ensemble-oracle}, bottom):
\begin{itemize}
\item \emph{Separately trained models ensemble well.} Averaging the Max of two supernets trained with the same settings lowers the loss by $0.056$ nats, and that of three by $0.077$ nats; uniform weights are already the best fixed weights, as the members are equally strong. The oracle over two such Max sub-networks also gains more than that over the five sub-networks, $0.23$ against $0.19$ nats. For the five sub-networks of one supernet, no fixed weights improve on Max: they share the trunk and are much more alike, and the narrower ones are weaker, so averaging only dilutes Max.
\item \emph{Low gives the largest share of the gain.} With Low and Max alone, the oracle is already $0.12$ nats below Max, about two thirds of its gain with all five tiers.
\item \emph{Cost.} Each separate model needs its own training run and trunk pass, whereas one trunk pass yields all five sub-networks.
\end{itemize}

\textbf{Outlook:} Ensembling the sub-networks remains promising for elastic models: the oracle is $0.19$ nats below Max alone and picks a narrower sub-network for $71\%$ of held-out tokens (Table~\ref{tab:ensemble-oracle}), and a label-free selection better than confidence could turn part of it into a gain at almost no extra cost.

\begin{table}[htp]
\caption{\textbf{Ensembles of the five sub-networks as predictors and as teachers.} Train and validation loss minus those of Max. The row labeled Router is the ensemble weighted per token by the learned router, which is trained together with its own supernet, so this row subtracts the losses of that supernet's Max. The loss-weighted ensemble and the soft oracle are the two ensembles used as distillation teachers in ablation~(8).}
\label{tab:ensemble-mixtures}
\centering
\small
\begin{tabular}{@{}lcc@{}}
\toprule
 & \multicolumn{2}{c}{Loss $-$ Max loss} \\
\cmidrule(l{0pt}){2-3}
Ensemble & Train loss & Val.\ loss \\
\midrule
\multicolumn{3}{@{}l}{\emph{Predictors}} \\
Uniform weights & $+0.0313$ & $+0.0316$ \\
Confidence weights & $+0.0295$ & $+0.0296$ \\
Most confident readout & $+0.0648$ & $+0.0678$ \\
Best fixed ensemble & $0.0000$ & $-0.0003$ \\
Router & $0.0000$ & $0.0000$ \\
\midrule
\multicolumn{3}{@{}l}{\emph{Teachers}} \\
Loss-weighted ensemble & -- & $+0.0116$ \\
Soft oracle (uses the target) & -- & $-0.1433$ \\
\midrule
Oracle (uses the target) & $-0.1731$ & $-0.1851$ \\
\bottomrule
\end{tabular}
\end{table}

\begin{table}[htp]
\caption{\textbf{The oracle on held-out C4 text} (final checkpoints). Top: share of tokens on which each tier assigns the target the highest probability (the oracle's pick), its contribution to the oracle's mean gain over Max (the four sum to $0.185$), and the share of tokens on which it beats Max. Bottom: validation loss minus that of Max when the models in each row are ensembled with uniform weights, the best fixed weights or the oracle.}
\label{tab:ensemble-oracle}
\centering
\small
\begin{tabular}{@{}lccccc@{}}
\toprule
Tier & Low & Medium & High & Extra & Max \\
\midrule
Picked by the oracle (\%) & 22.7 & 16.3 & 15.1 & 16.4 & 29.4 \\
Contribution to the gain (nats) & 0.0878 & 0.0454 & 0.0309 & 0.0210 & -- \\
Beats Max (\%) & 39.0 & 42.7 & 44.9 & 46.7 & -- \\
\bottomrule
\end{tabular}

\vspace{6pt}
\begin{tabular}{@{}lccc@{}}
\toprule
Members & Uniform & Best fixed & Oracle \\
\midrule
Five tiers of one supernet & $+0.0316$ & $-0.0003$ & $-0.1851$ \\
Low and Max of one supernet & $+0.0355$ & $-0.0002$ & $-0.1187$ \\
Max of two supernets & $-0.0564$ & $-0.0564$ & $-0.2306$ \\
Max of three supernets & $-0.0767$ & $-0.0767$ & $-0.3401$ \\
\bottomrule
\end{tabular}
\end{table}

\paragraph{(9) Stage-residual gradient scaling.}
The logits of width $s_k$ are a sum of one term per group,
\begin{equation}
 z_{s_k}=\alpha_{s_k}\sum_{j\le k}h^{(k)}_{\mathcal G_j}\big(U_{s_k}\big)_{:,\mathcal G_j}^{\top},\qquad h^{(k)}=N^{\rm out}_{s_k}\big(x^{(L)}_{:s_kg}\big),
\label{eq:stagelogits}
\end{equation}
so the supernet can be viewed as a sequence of stages: the core gives a base prediction, and each shell adds a correction to the terms of the groups before it, as in residual learning~\citep{he2016deep}.
To train the shells toward this role, we scale by $\gamma$ the gradient that the loss of each width sends into the terms of the narrower groups, through the coordinates of the next narrower sub-network:
\begin{equation}
 \frac{\partial\ell_{s_k}}{\partial h^{(k)}_{:s_{k-1}g}}\;\leftarrow\;\gamma\,\frac{\partial\ell_{s_k}}{\partial h^{(k)}_{:s_{k-1}g}},\qquad k=2,\dots,K.
\label{eq:stagegrad}
\end{equation}
Every width has its own readout and terminal norm, which only its own loss trains, so for $\gamma>0$ Adam, being invariant to the gradient scale, still trains them in full; the scaling only reduces the gradient that these coordinates pass on to the trunk.
The forward pass is unchanged, and $\gamma=1$ is the joint objective of Equation~\ref{eq:loss}.
With $\gamma=0$, the loss of each width reaches the last-layer features of its own shell only, which then fit what the narrower groups leave unexplained, while the wider losses still train the narrower groups inside the trunk.

\textbf{Observation:} A $\gamma$ between $0.3$ and $0.8$ lowers the loss of Low by up to $0.028$ nats but raises those of Extra and Max by up to $0.012$ nats and leaves the mean loss within $0.002$ nats.
With $\gamma=0$, Low is $0.003$ nats lower, but the wider widths lose more, from $0.036$ nats at Medium to $0.071$ at Extra and Max, and the mean loss rises by $0.047$ nats, as the last-layer features of the narrower groups no longer adapt to the losses of the wider widths.
We therefore keep $\gamma=1$.

\textbf{Outlook:} Although $\gamma<1$ does not lower the mean loss here, the stage view matches a property that \Ours{} has but MatFormer and fully elastic models lack: it can grow after training.
Since no group reads a later one (Equation~\ref{eq:grouplinear}) and every group is normalized on its own (Equation~\ref{eq:norm}), a new shell $\mathcal G_{K+1}$ can be appended to a trained supernet and trained as a new stage, for example on downstream data, while all trained weights stay frozen, so every trained sub-network keeps its exact outputs.
MatFormer and fully elastic models cannot grow this way: new FFN units of MatFormer write into the hidden states that its trained layers read, and new channels of a fully elastic model change the inputs of every trained weight in the wider sub-network, so their trained weights would operate on inputs they were not trained on.

\paragraph{(10) Trunk gradient of the narrower widths.}
We detach the trunk from the losses of Low to Extra, which then train only their terminal norms and readouts on a trunk trained by Max alone.

\textbf{Observation:} Detaching lowers the loss of Max by $0.014$ nats but raises those of Low to Extra by $0.049$ to $0.520$ nats, so keeping their trunk gradients is much better overall.

\paragraph{(11) Base width and head dimension.}
At the same supported widths $512$ to $1024$ and the same parameter counts, halving the base width to $g=64$ ($M=16$, $\mathcal S=\{8,10,12,14,16\}$) keeps the $512$-wide core and every shell $128$ wide, so the groups, their normalization and the \muE{} factors $ms$ are unchanged.
We also halve the head dimension to $d_{\mathrm{head}}=64$, which $g=64$ requires.

\textbf{Observation:} With $d_{\mathrm{head}}=64$ in both, $g=64$ and $g=128$ define the same model, and their mean losses differ by $0.001$ nats.
Halving the head dimension lowers the mean loss by $0.003$ to $0.004$ nats, mostly at Low.
The head dimension thus has a small effect near run-to-run variation, and the base width none; we keep $g=d_{\mathrm{head}}=128$.

\paragraph{(12) Optimization HPs.}
We vary the gradient clipping (global norm $0.5$, $1$ and $2$), the AdamW $\epsilon$ ($10^{-6}$, $10^{-8}$ and $10^{-10}$) and the $z$-loss coefficient $\lambda_z$ ($10^{-3}$ and $10^{-2}$) around the defaults of Table~\ref{tab:hparams}: global norm $4$, $\epsilon=10^{-14}$ and $\lambda_z=10^{-4}$.

\textbf{Observation:} The defaults are within $0.002$ nats in mean loss of clipping at $0.5$ or $2$, $\epsilon=10^{-8}$ or $10^{-10}$, and $\lambda_z=10^{-3}$ or $10^{-2}$.
Clipping at $1$ is $0.005$ nats worse and $\epsilon=10^{-6}$ is $0.007$ nats worse, so the recipe is robust to these settings except for a large $\epsilon$.

\end{document}